\documentclass{article}
\usepackage{times}
\usepackage{url}
\usepackage{microtype}
\usepackage{graphicx}
\usepackage{booktabs} 
\usepackage{hyperref} 

\usepackage[export]{adjustbox}

\usepackage{amsmath,amsthm,amssymb} 
\usepackage[utf8]{inputenc} 
\usepackage[english]{babel}
\usepackage{listings} 
\usepackage[title]{appendix}

\usepackage{mathtools}
\usepackage{subcaption}
\usepackage{varwidth}
\usepackage{pgf}
\usepackage{tikz}
\usepackage{multirow}
\usetikzlibrary{arrows,automata,positioning}

\makeatletter
\newcommand*\bigcdot{\mathpalette\bigcdot@{.5}}
\newcommand*\bigcdot@[2]{\mathbin{\vcenter{\hbox{\scalebox{#2}{$\m@th#1\bullet$}}}}}
\makeatother
\def\G{{\mathcal G}}

\def\N{{\mathcal N}}

\def\M{{\mathcal M}}

\def\aChosen{a_{\mathrm{chosen}}}
\def\aEnv{a_{\mathrm{env}}}
\def\aEnvt{a_{\mathrm{env},t}}

\newcommand*\norm[1]{\left\|#1\right\|}
\newcommand*\val[2]{#1\tiny\raisebox{.2ex}{$\scriptstyle\pm$#2}}
\newcommand*\valbf[2]{\val{\textbf{#1}}{#2}}
\usepackage{array}
\newcolumntype{P}[1]{>{\centering\arraybackslash}p{#1}}
\newcolumntype{L}[1]{>{\raggedright\arraybackslash}p{#1}}

\newif\ifcomments
\commentstrue 

\ifcomments

\fi

\newtheorem{proposition}{Proposition}

\newtheorem{lemma}[proposition]{Lemma}
\newtheorem{theorem}[proposition]{Theorem}

\newtheorem{defn}{Definition}

\DeclareMathOperator*{\E}{\mathbb{E}}
\DeclareMathOperator*{\R}{\mathbb{R}}

\def\G{{\mathcal G}}
\def\S{{\mathcal S}}
\def\A{{\mathcal A}}

\newcommand{\condE}[1]{\E_{\mid #1}}

\usepackage[accepted]{icml2019}

\icmltitlerunning{Exploration Conscious Reinforcement Learning Revisited} 

\begin{document}

\twocolumn[
\icmltitle{Exploration Conscious Reinforcement Learning Revisited}

\icmlsetsymbol{equal}{*}
\begin{icmlauthorlist}
\icmlauthor{Lior Shani}{equal,to}
\icmlauthor{Yonathan Efroni}{equal,to}
\icmlauthor{Shie Mannor}{to}

\end{icmlauthorlist}

\icmlaffiliation{to}{Department of Electrical Engineering, Technion, Haifa, Israel}

\icmlcorrespondingauthor{Lior Shani}{shanlior@gmail.com}
\icmlcorrespondingauthor{Yonathan Efroni}{jonathan.efroni@gmail.com}

\icmlkeywords{Machine Learning, Reinforcement Learning, Exploration Conscious, Exploration, ICML}

\vskip 0.3in
]

\printAffiliationsAndNotice{\icmlEqualContribution} 

\begin{abstract}
  


The Exploration-Exploitation tradeoff arises in Reinforcement Learning when one cannot tell if a policy is optimal. Then, there is a constant need to explore new actions instead of exploiting past experience. In practice, it is common to resolve the tradeoff by using a fixed exploration mechanism, such as  $\epsilon$-greedy exploration or by adding Gaussian noise, while still trying to learn an optimal policy. In this work, we take a different approach and study exploration-conscious criteria, that result in optimal policies with respect to the exploration mechanism. Solving these criteria, as we establish, amounts to solving a surrogate Markov Decision Process. We continue and analyze properties of exploration-conscious optimal policies and characterize two general approaches to solve such criteria. Building on the approaches, we apply simple changes in existing tabular and deep Reinforcement Learning algorithms and empirically demonstrate superior performance relatively to their non-exploration-conscious counterparts, both for discrete and continuous action spaces.
\end{abstract}

\section{Introduction}


The main goal of Reinforcement Learning (RL) \cite{sutton1998reinforcement} is to find an optimal policy for a given decision problem. 
A major difficulty arises due to the Exploration-Exploitation tradeoff, which characterizes the omnipresent tension between exploring new actions and exploiting the so-far acquired knowledge. Considerable line of work has been devoted for dealing with this tradeoff. Algorithms that explicitly balance between  exploration and exploitation were developed for tabular RL \cite{kearns2002near,brafman2002r,jaksch2010near,osband2013more}. However, generalizing these results to approximate RL, i.e, when using function approximation, remains an open problem. On the practical side, recent works combined more advanced exploration schemes in approximate RL (e.g, \citet{bellemare2016unifying,fortunato2017noisy}), inspired by the theory of tabular RL. Nonetheless, even in the presence of more advanced mechanisms, $\epsilon$-greedy exploration is still applied \cite{bellemare2017distributional,dabney2018distributional,osband2016deep}.  More generally,  the traditional and simpler $\epsilon$-greedy scheme \cite{sutton1998reinforcement,asadi2016alternative} in discrete RL, and Gaussian action noise in continuous RL, are still very useful and popular in practice \cite{mnih2015human,mnih2016asynchronous,silver2014deterministic,schulman2017proximal,horgan2018distributed}, especially due to their simplicity. 



These types of exploration schemes share common properties. First, they all fix some exploration parameter beforehand, e.g, $\epsilon$, the `inverse temperature' $\beta$, or the action variance $\sigma$ for the $\epsilon$-greedy, soft-max and Gaussian exploration schemes, respectively.  By doing so, the balance between exploration and exploitation is set. Second, they all explore using a random policy, and exploit using current estimate of the \emph{optimal policy}. In this work, we follow a different approach, when using these fixed exploration schemes: exploiting by using an estimate of the optimal policy w.r.t. the \emph{exploration mechanism}.

\emph{Exploration-Consciousness} is the main reason for the improved performance of on-policy methods like Sarsa and Expected-Sarsa \cite{van2009theoretical} over Q-learning during training \cite{sutton1998reinforcement}[Example 6.6: Cliff Walking]. Imagine a simple Cliff-Walking problem: The goal of the agent is to reach the end without falling of the cliff, where the optimal policy is to go alongside the cliff. While using a fixed-exploration scheme, playing a near optimal policy which goes alongside the cliff will lead to a significant sub-optimal performance. This, in turn, will hurt the acquisition of new experience needed to learn the optimal policy. However, learning to act optimally w.r.t. the exploration scheme can mitigate this difficultly; the agent learns to reach the goal while keeping a safe enough distance from the cliff.

In the past, tabular q-learning-like exploration-conscious algorithms were suggested \cite{john1994best,littman1997generalized,van2009theoretical}. Here we take a different approach, and focus on exploration conscious \emph{policies}. The main contributions of this work are as follows:

\begin{itemize}
\item We define exploration-consciousness optimization criteria, for discrete and continuous actions spaces. The criteria are interpreted as finding an optimal policy within a restricted set of policies. Both, we show, can be reduced to solving a surrogate MDP. The surrogate MDP approach, to the best of our knowledge, is a new one, and serves us repeatedly in this work.
\item We formalize a bias-error sensitivity tradeoff. The solutions are biased w.r.t. the optimal policy, yet, are less sensitive to approximation errors.
\item We establish two fundamental approaches to practically solve Exploration-Conscious optimization problems. Based on these, we formulate algorithms in discrete and continuous action spaces, and empirically test the algorithms on the Atari and MuJoCo domains. 
\end{itemize}

\section{Preliminaries}\label{sec: preliminaries}
Our framework is the infinite-horizon discounted Markov Decision Process (MDP). An MDP is defined as the 5-tuple $(\mathcal{S}, \mathcal{A},P,R,\gamma)$ \cite{puterman1994markov}, where ${\mathcal S}$ is a finite state space, ${\mathcal A}$ is a compact space, $P \equiv P(s'|s,a)$ is a transition kernel, $R \equiv r(s,a)\in[0,R_{\max}]$ is a bounded reward function, and $\gamma\in[0,1)$. Let ${\pi: \mathcal{S}\rightarrow \mathcal{P}(\mathcal{A})}$ be a stationary policy, where $\mathcal{P}(\mathcal{A})$ is a probability distribution on $\mathcal{A}$, and denote $\Pi$ as the set of deterministic policies, $\pi\in\Pi : \mathcal{S}\rightarrow \mathcal{A}$. Let $v^\pi \in \mathbb{R}^{|\mathcal{S}|}$ be the value of a policy $\pi,$ defined in state $s$ as $v^\pi(s) \equiv \condE{s} ^\pi[\sum_{t=0}^\infty\gamma^tr(s_t,a_t)]$, where $a_t\sim \pi(s_t)$, and $\condE{s}^\pi$ denotes expectation w.r.t. the distribution induced by $\pi$ and conditioned on the event $\{s_0=s\}.$  
It is known that ${v^\pi=\sum_{t=0}^\infty \gamma^t (P^\pi)^t r^\pi=(I-\gamma P^\pi)^{-1}r^\pi}$, with the component-wise values $[P^\pi]_{s,s'}  \triangleq \mathbb{E}_{a\sim \pi}[P(s'\mid s, a)]$ and $[r^\pi]_s \triangleq  \mathbb{E}_{a\sim\pi}[r(s,a)]$. Furthermore, the $q$-function of $\pi$ is given by ${q^{\pi}(s,a)= r(s,a)+\gamma \sum_{s'} P(s'\mid s,a)v^{\pi}(s')}$, and represents the value of taking an action $a$ from state $s$ and then using the policy $\pi$.

Usually, the goal is to find $\pi^*$ yielding the optimal value, $\pi^* \in \arg \max_{\pi\in \Pi} \mathbb{E} ^\pi[\sum_{t=0}^\infty\gamma^tr(s_t,a_t)],$ and the optimal value is ${v^* = v^{\pi^*}}$. It is known that optimal deterministic policy always exists \cite{puterman1994markov}. To achieve this goal the following classical operators are defined (with equalities holding component-wise). $\forall v,\pi$ :  
\begin{align}
T^\pi v =&  r^\pi +\gamma P^\pi v,~ T v = \max_\pi T^\pi v, \label{eq: T opt}\\
&\G(v)= \{\pi : T^\pi v = T v\}, \label{eq: G greeedy}
\end{align}
where $T^\pi$ is a linear operator, $T$ is the optimal Bellman operator and both $T^\pi$ and $T$ are $\gamma$-contraction mappings w.r.t. the max norm. It is known that the unique fixed points of $T^\pi$ and $T$ are $v^\pi$ and $v^*$, respectively. $\G(v)$ is the standard set of 1-step greedy policies w.r.t. $v$. 
Furthermore, given $v^*$, the set $\G(v^*)$ coincides with that of stationary optimal policies. It is also useful to define the $q$-optimal Bellman operator, which is a $\gamma$-contraction, with fixed point~$q^*$.
\begin{align}
T^{q}q (s,a)\! =\! r(s,a)\!+\!\gamma \sum_{s'} P(s'\mid s,a) \max_{a'} q(s',a'), \label{eq: def optimal bellman q mdp}
\end{align}

In this work, the use of \emph{mixture policies} is abundant. We denote the $\alpha\in[0,1]$-convex mixture of policies $\pi_1,\ \pi_2$ by $
\pi^\alpha(\pi_1,\pi_2)\triangleq(1-\alpha)\pi_1+\alpha\pi_2$. Importantly, $\pi^\alpha(\pi_1,\pi_2)$ can be interpreted as a stochastic policy s.t with w.p $(1-\alpha)$ the agent acts with $\pi_1$ and w.p $\alpha$ acts with $\pi_2$.


\section{The $\alpha$-optimal criterion}\label{sec: alpha objective}


In this section, we define the notion of $\alpha$-optimal policy w.r.t. a policy, $\pi_0$. We then claim that finding an $\alpha$-optimal policy can be done by solving a \emph{surrogate} MDP. We continue by defining the surrogate MDP, and analyze some basic properties of the $\alpha$-optimal policy.  

Let $\alpha\in [0,1]$. We define  $\pi^*_{\alpha,\pi_0}$ to be the $\alpha$-optimal policy w.r.t. $\pi_0$, and is contained in the following set,
\begin{align}\label{eq:eps_optimization}
&\pi^*_{\alpha,\pi_0} \in \arg\max_{\pi'\in \Pi} {\E}^{\pi^\alpha(\pi',\pi_0)} \left[\sum_{t=0} \gamma^t r(s_t,a_t)) \right], 
\end{align}
or, $\pi^*_{\alpha,\pi_0} \in \arg\max_{\pi'} v^{\pi^\alpha(\pi',\pi_0)}$, where ${a_t\sim \pi^\alpha(\pi',\pi_0)}$  and $\pi^\alpha(\pi',\pi_0)$ is the $\alpha$-convex mixture of $\pi'$ and~$\pi_0$, and thus a probability distribution. For brevity, we omit the subscript $\pi_0$, and denote the $\alpha$-optimal policy by~$\pi^*_{\alpha}$  throughout the rest of the paper. The $\alpha$-optimal value (w.r.t. $\pi_0$) is $v^{\pi^\alpha(\pi^*_\alpha,\pi_0)}$, the value of the policy $\pi^\alpha(\pi^*_\alpha,\pi_0)$. In the following, we will see the problem is equivalent to solving a surrogate MDP, for which an optimal deterministic policy is known to exist. Thus, there is no loss optimizing over the set of deterministic policies $\Pi$.

Optimization problem \eqref{eq:eps_optimization} can be viewed as optimizing over a restricted set of policies: all policies that are a convex combination of $\pi_0$ with a fixed $\alpha$. Naturally, we can consider in \eqref{eq:eps_optimization} a state-dependent $\alpha(s)$ as well, and some of the results in this work will consider this scenario. In other words, $\pi^*_{\alpha}$ is the best policy an agent can act with, if it plays w.p $(1-\alpha)$ according to $\pi^*_{\alpha}$, and w.p $\alpha$ according to $\pi_0$, where $\pi_0$ can be any policy. The relation to the $\epsilon$-greedy exploration setup becomes clear when $\pi_0$ is a uniform distribution on the actions, and set $\alpha=\epsilon$ instead of $\alpha$. Then, $\pi^*_{\alpha}$ is optimal w.r.t. the $\epsilon$-greedy exploration scheme; the policy would have the largest accumulated reward, relatively to all other policies, when acting in an $\epsilon$-greedy fashion w.r.t. it. 


We choose to name the policy as the $\alpha$- and not $\epsilon$-optimal to prevent confusion with other frameworks. The $\epsilon$-optimal policy is a notation used in the context of PAC-MDP type of analysis \cite{strehl2009reinforcement}, and has a different meaning than the objective in this work \eqref{eq:eps_optimization}.

\subsection{The $\alpha$-optimal Bellman operator, $\alpha$-optimal policy and policy improvement}

In the previous section, we defined the $\alpha$-optimal policy and the $\alpha$-optimal value, $\pi^*_\alpha$ and $v^{\pi^\alpha(\pi^*_\alpha,\pi_0)}$, respectively.  We start this section by observing that problem \eqref{eq:eps_optimization} can be viewed as solving a \emph{surrogate MDP}, denoted by $\mathcal{M}_\alpha$. We define the Bellman operators of the surrogate MDP, and use them to prove an important improvement property.


Define the surrogate MDP as ${\mathcal{M}_\alpha \!=\! (\mathcal{S},\mathcal{A},P_\alpha,R_\alpha,\gamma )}$.
\begin{align}
&\forall a\in \mathcal{A},\ r_\alpha(s,a) \!=\! (1-\alpha) r(s,a)+\alpha r^{\pi_0}(s),\nonumber\\
&\ P^\pi_\alpha(s'\mid s,a) \! = \! (1-\alpha)P(s'\mid s,a)+\alpha P^{\pi_0}(s'\mid s),\label{eq: surrogate MDP reward and dynamics}
\end{align}
are its reward and dynamics, and rest of its ingredients are similar to $\mathcal{M}$. We denote the value of a policy $\pi$ on $\mathcal{M}_\alpha$ by $v^\pi_\alpha$, and the optimal value on $\mathcal{M}_\alpha$ by $v^*_\alpha$. The following simple lemma relates the value of a policy $\pi$, measured on $\mathcal{M}$ and $\mathcal{M}_\alpha$ (see proof in Appendix \ref{supp: lemma equivalence}).
\begin{lemma}\label{lemma: equivalence}
For any policy $\pi$, $v^\pi_\alpha = v^{\pi^\alpha(\pi,\pi_0)}$. Thus, an optimal policy on $\mathcal{M}_\alpha$ is the $\alpha$-optimal policy $\pi^*_\alpha~\eqref{eq:eps_optimization}$.
\end{lemma}


The fixed-policy and optimal Bellman operators of $\mathcal{M}_\alpha$ are denoted by $T^{\pi}_\alpha$ and $T_\alpha$, respectively. Again, for brevity we omit $\pi_0$ from the definitions. Notice that $T^{\pi}_\alpha$ and $T_\alpha$ are $\gamma$-contractions as being Bellman operators of a $\gamma$-discounted MDP.  The following Lemma relates $T^{\pi}_\alpha$ and $T_\alpha$ to the Bellman operators of the original MDP, $\mathcal{M}$. Furthermore, it stresses a non-trivial relation between the $\alpha$-optimal policy $\pi_\alpha^*$ and the $\alpha$-optimal value, $v^{\pi^\alpha(\pi_\alpha^*,\pi_0)}$.

\begin{proposition}\label{prop: alpha surrogate mdp bellman}
The following claims hold for any policy $\pi$:
\begin{enumerate}
\item $T^{\pi}_\alpha \!\!= \!\!(1\!-\!\alpha)T^{\pi}\!+\!\alpha T^{\pi_0}$, with fixed point $v^{\pi}_\alpha\!\!=\!\!v^{\pi^\alpha(\pi,\pi_0)}$.
\item $T_\alpha \!=\!(1\!-\!\alpha)T\!+\!\alpha T^{\pi_0}$, with fixed point $v^*_\alpha\!=\!v^{\pi^\alpha(\pi_\alpha^*,\pi_0)}$.
\item An $\alpha$-optimal policy is an optimal policy of $\mathcal{M}_\alpha$ and is greedy w.r.t. $v^*_\alpha$, $\pi^*_\alpha\in \G(v_\alpha^*)=\{\pi': T^{\pi'} v_\alpha^* = T v_\alpha^* \}.$
\end{enumerate}
\end{proposition}

In previous works, e.g. \cite{asadi2016alternative}, the operator $(1-\epsilon)T+\epsilon T^{\pi_0}$ was referred to as the $\epsilon$-greedy operator. Lemma \ref{prop: alpha surrogate mdp bellman} shows this operator is $T_\alpha$ (with $\alpha=\epsilon$), the optimal Bellman operator of the defined surrogate MDP~$\mathcal{M}_\alpha$. This lemma leads to the following important property. 

\begin{proposition} 
\label{proposition: alpha optimal improvement}
Let $\alpha\in[0,1)$, $\beta\in[0,\alpha]$, $\pi_0$ be a policy, and $\pi_\alpha^*$ be the $\alpha$-optimal policy w.r.t $\pi_0$. Then,
$
{v^{\pi_0}\leq  v^{\pi^{\alpha}(\pi_\alpha^*,\pi_0)}\leq  v^{\pi^{\beta}(\pi_\alpha^*,\pi_0)},}
$
with equality iff ${v^{\pi_0}=v^*}$.
\end{proposition}


The first relation  $v^{\pi_0}\leq  v^{\pi^{\alpha}(\pi_\alpha^*,\pi_0)}$, $\pi^\alpha(\pi^*_\alpha,\pi_0)$ is better than~$\pi_0$, is trivial and holds by definition \eqref{eq:eps_optimization}. The non-trivial statement is the second one. It asserts that given $\pi^*_\alpha$, it is worthwhile to use the mixture policy $\pi^{\beta}(\pi_\alpha^*,\pi_0)$ with $\beta<\alpha$; use $\pi_0$ with smaller probability. Specifically, better performance, compared to $\pi^{\alpha}(\pi_\alpha^*,\pi_0)$,  is assured  when using the deterministic policy $\pi^*_\alpha$, by setting $\beta=0$.

In section \ref{sec: experiments}, we demonstrate the empirical consequences of the improvement lemma, which, to our knowledge, has not yet been stated. Furthermore, the improvement lemma is unique to the defined optimization criterion \eqref{eq:eps_optimization}. We will show that alternative definitions of exploration conscious criteria does not necessarily have this property. Moreover, one can use Proposition \ref{proposition: alpha optimal improvement} to generalize the notion of the 1-step greedy policy \eqref{eq: G greeedy}, as was done in \citet{beyond2018efroni} with multiple-step greedy improvement.  We leave studying this generalization and its Policy Iteration scheme for future work, and focus on solving \eqref{eq:eps_optimization} a single time.

\subsection{Performance bounds in the presence of approximations}

We now consider an approximate setting and quantify a bias - error sensitivity tradeoff in $\pi^\alpha(\hat{\pi}_\alpha^*,\pi_0)$, where $\hat{\pi}_\alpha^*$ is an approximated $\alpha$-optimal policy.  We formalize an intuitive argument; as $\alpha$ increases the bias relatively to the optimal policy increases. Yet, the sensitivity to errors decreases, since the agent uses $\pi_0$ w.p. $\alpha$ regardless of errors. 

\begin{defn}\label{defn: lipschitz constant}
Let $v^*$ be the optimal value of an MDP, $\mathcal{M}$. We define $L(s)\triangleq v^*(s)-T^{\pi_0}v^*(s)\geq 0$, to be the Lipschitz constant w.r.t. $\pi_0$ of the MDP at state $s$. We further define the upper bound on the Lipschitz constant $L\triangleq \max_s L(s)$.
\end{defn}

Definition~\ref{defn: lipschitz constant} defines the `Lipschitz' property of the optimal value, $v^*$.
Intuitively, $L(s)$ quantifies a degree of `smoothness' of the \emph{optimal value}. A small value of $L(s)$ indicates that if one acts according to $\pi_0$ once and then continue playing the optimal policy from state $s$, a great loss will not occur. Large values of $L(s)$ indicate that using $\pi_0$ from state $s$ leads to an irreparable outcome (e.g, falling off a cliff). The following theorem  formalizes a bias-error sensitivity tradeoff. As $\alpha$ increases, the bias increases, while the sensitivity to errors decreases (see proof in Appendix~\ref{supp: theorem performance model free}).
\begin{theorem}\label{theorem:performance model free}
Let $\alpha\in [0,1]$. Assume $\hat{v}^*_\alpha$ is an approximate $\alpha$-optimal value s.t $\norm{v^*_\alpha-\hat{v}^*_\alpha}=\delta$ for some $\delta\geq 0$. Let $\hat{\pi}_\alpha^*$ be the greedy policy w.r.t. $\hat{v}_\alpha^*$,  $\hat{\pi}_\alpha^*\in \G(\hat{v}^*_\alpha
)$. Then, the performance relatively to the optimal policy is bounded by,
\begin{align*}
\norm{v^*-v^{\pi^{\alpha}(\hat{\pi}_\alpha^*,\pi_0)}} \leq \underset{\mathrm{Bias}}{\underbrace{\frac{\alpha L }{1-\gamma}}} + \underset{\mathrm{Sensitivity}}{\underbrace{\frac{2(1-\alpha)\gamma\delta}{1-\gamma}}}.
\end{align*}
\end{theorem}

When the bias of the $\alpha$-optimal value relatively to the optimal one is small, solving \eqref{eq:eps_optimization} does not lead to a great loss relatively to the optimal performance. The bias can be bounded by the `Lipschitz' property $L$ of the MDP. For a state dependent $\alpha(s)$, the bias bound changes to be dependent on $\max_s\alpha(s)L(s)$. This highlights the importance of prior knowledge when using \eqref{eq:eps_optimization}. Choosing $\pi_0$ (possibly state-wise) s.t. $\max_s\alpha(s)L(s)$ is small, allows to use a  bigger $\alpha$, while the bias is small. The sensitivity term upper bounds the performance of $\pi^\alpha(\hat{\pi}_\alpha^*,\pi_0)$ relatively to the $\alpha$-optimal value, and is less sensitive to errors as $\alpha$ increase.


The bias term is derived by using the structure of $\mathcal{M}_\alpha$, and is not a direct application of the Simulation Lemma \cite{kearns2002near,strehl2009reinforcement}; applying it would lead to a bias of $\frac{\alpha R_{\mathrm{max}}}{(1-\gamma)^2}$. 
 For the sensitivity term, we generalize  \cite{bertsekas1995neuro}[Proposition 6.1] (see Appendix~\ref{supp: generalized sensitivity bound}). There, a $(1-\alpha)$ factor does not exists.

\section{Exploration-Conscious Continuous Control}
\label{sec: ContinuousControl}
The $\alpha$-greedy approach from Section~\ref{sec: alpha objective} relies on an exploration mechanism which is fixed beforehand: $\pi_0$ and $\alpha$ are fixed, and an optimal policy w.r.t. them is being calculated~\eqref{eq:eps_optimization}. However, in continuous control RL algorithms, such as DDPG and PPO \cite{lillicrap2015continuous,schulman2017proximal}, different approach is used. Usually, a policy is being learned, and the exploration noise is injected by perturbing the policy, e.g., by adding to it a Gaussian noise. 

We start this section by defining an exploration-conscious optimality criterion that captures such perturbation for the simple case of Gaussian noise. Then, results from Section~\ref{sec: alpha objective} are adapted to the newly defined criterion, while highlighting commonalities and differences relatively to~\eqref{eq:eps_optimization}. As in Section~\ref{sec: alpha objective}, we define an appropriate surrogate MDP and we show it can be solved by the usual machinery of Bellman operators. Unlike Section~\ref{sec: alpha objective}, we show that improvement when decreasing the stochasticity does not generally hold. Finally, we prove a similar bias-error sensitivity result: As $\sigma$ grows, the bias increases, but the sensitivity term decreases. 

Instead of restricting the set of policies to the one defined in~\eqref{eq:eps_optimization}, we restrict our set of policies to be the set of Gaussian policies with a fixed $\sigma^2$ variance. Formally, we wish to find the optimal deterministic policy $\mu^*_\sigma:\mathcal{S}\rightarrow\mathcal{A}$ in this set,
\begin{align}\label{eq:continuous eps_optimization}
&\mu_{\sigma}^* \in \arg\max_{\mu\in \Pi} {\E}^{\pi_{\mu,\sigma}} \left[\sum_{t=0}^{\infty} \gamma^t r(s_t,a_t) \right],
\end{align}
where $\pi_{\mu,\sigma}(\cdot\mid s) = \mathcal{N}(\mu(s),\sigma^2)$, is a Gaussian policy with mean $\mu(s)$ and a fixed variance $\sigma^2$. We name $\mu_{\sigma}^*$ and $\pi_{\sigma}^*$ as the mean and $\sigma$-optimal policy, respectively. As in \eqref{eq:eps_optimization}, we show in the following that solving \eqref{eq:continuous eps_optimization} is equivalent for solving a surrogate MDP. Thus, optimal policy can always be found in the deterministic class of policies $\Pi$; mixture of Gaussians  would not lead to a better performance in~\eqref{eq:continuous eps_optimization}.

Similarly to \eqref{eq: surrogate MDP reward and dynamics}, we define a surrogate MDP $\M_\sigma$ w.r.t. to the Gaussian noise and relate it to values of Gaussian policies on the original MDP $\mathcal{M}$. Then, we characterize its Bellman operators and thus establish it can be solved using Dynamic Programming. Define the surrogate MDP as ${\mathcal{M}_\sigma \!=\! (\mathcal{S},\mathcal{A},P_\sigma,R_\sigma,\gamma )}$. For every $a\in \mathcal{A}$,
\begin{align}
 &r_\sigma(s, a) \! = \int_{\mathcal{A}} \mathcal{N}(a';a,\sigma)r(s,a') da', \nonumber\\
  &P_\sigma(s'\mid s,a) \! = \int_{\mathcal{A}} \mathcal{N}(a';a,\sigma)P(s'\mid s,a') da' \label{eq: surrogate gaussian MDP reward and dynamics},
\end{align}
%
are its reward and dynamics, and denote a value of a policy on $\mathcal{M}_\sigma$ by $v_\sigma^\mu$. The following results correspond to Lemma \ref{lemma: equivalence} and Proposition \ref{prop: alpha surrogate mdp bellman} for the class of Gaussian policies.

\begin{lemma}\label{lemma: equivalence continuous} For any policy $\pi$, $v^\pi_{\mu,\sigma} = v^\mu_\sigma$. Thus, an optimal policy on $\mathcal{M}_\sigma$ is the mean optimal policy $\mu^*_\sigma~\eqref{eq:continuous eps_optimization}$.
\end{lemma}

\begin{proposition}\label{prop:ContinuousContraction}
Let $\pi$ be a mixture of Gaussian policies. Then, the following holds:
\begin{enumerate}
\item $T_\sigma^{\mu} =  \E^{\pi\sim\pi_{\mu,\sigma}} T^\pi $, with fixed point $v^\mu_\sigma\!=\!v^{\pi_{\mu,\sigma}}$.
\item $T_\sigma \!=\! \underset{\mu\in\tilde{\A}}{\max} \E^{\pi\sim\pi_{\mu,\sigma}} T^\pi$, with fixed point $v^*_\sigma\!=\!v^{\pi_{\mu_\sigma^*,\sigma}}$.
\item The mean $\sigma$-optimal policy $\mu^*_\sigma$ is an optimal policy of $\mathcal{M}_\sigma$ and, $\mu^*_\sigma \in \{\mu: T^{\pi_{\mu,\sigma}} v_\sigma^* = \max_{\mu} T^{\pi_{\mu,\sigma}} v_\sigma^* \}.$
\end{enumerate}
\end{proposition}

 Surprisingly, given a $\sigma$-optimal policy mean $\mu^*_\sigma$, an improvement is not assured when lowering the stochasticity by decreasing $\sigma$ in $\pi_{\mu^*_\sigma,\sigma}$. This comes in contrast to Proposition \ref{proposition: alpha optimal improvement} and highlights its uniqueness (proof in Appendix~\ref{supp: prop NoImprovementContinuous}).
 \begin{proposition}\label{prop: NoImprovementContinuous}
Let $0\leq\sigma'<\sigma$ and let $\mu^*_\sigma$ be the mean $\sigma$-optimal policy. There exists an MDP s.t ${v^{{\pi}_{\mu^*,\sigma}} \nleq v^{{\pi}_{\mu^*,\sigma'}}}$.
\end{proposition}

\begin{defn}\label{defn: lipschitz continuous} 
Let $\mathcal{M}$ be a continuous action space MDP. Assume that exists $L_r,\ L_p \geq 0$, s.t. $\forall s\in \mathcal{S},\ \forall a_1,a_2 \in \mathcal{A}$, $\left|r(s,a_1)-r(s,a_2)\right| \leq L_r\norm{a_1-a_2}_1$ and $\norm{p(\cdot|s,a_1)-p(\cdot|s,a_2)}_{TV} \leq L_p\norm{a_1-a_2}_1$. The Lipschitz constant of $\mathcal{M}$ is $\mathcal{L}  \triangleq (1-\gamma)L_r + \gamma L_p R_{max}$.
\end{defn}

The following theorem quantifies a bias-error sensitivity tradeoff in $\sigma$, similarly to Theorem \ref{theorem:performance model free} (see Appendix \ref{supp: theorem gaussian tradeoff}).
\begin{theorem}\label{theorem: gaussian tradeoff}
Let $\M$ be an MDP with Lipschitz constant $\mathcal{L}$ and let $\sigma \in {\R}^{|\A|}_{+}$. Let $v_\sigma^*$ be the $\sigma$-optimal value of $\mathcal{M}_\sigma$. Let $\hat{v}_\sigma^*$ be an approximation of $v_\sigma^*$ s.t. $\norm{v_\sigma^*-\hat{v}_\sigma^*}=\delta$ for $\delta\geq 0$. Let ${\mu}_\sigma^*,\hat{\mu}_\sigma^* \in {\R}^{\A}$ be the greedy mean policy w.r.t. ${v}_\sigma^*$ and $\hat{v}_\sigma^*$ respectively. Let $\norm{\cdot}_{\sigma^{-2}}$ is the $\sigma^{-2}$-weighted euclidean norm. Then,
\begin{align*}
    \norm{v^*-v^{{\hat{\pi}}^*_\sigma}} &\! \leq \underset{\mathrm{Bias}}{\underbrace{\frac{\mathcal{L}\norm{\sigma}_1}{2\left(1-\gamma\right)^2} }} \!+\! \underset{\mathrm{Sensitivity}}{\underbrace{  \frac{\gamma\delta 
\min \{\frac{1}{2}\norm{{\mu}_\sigma^*-\hat{\mu}_\sigma^*}_{\sigma^{-2}} ,2\}}{1-\gamma}}}.
\end{align*}
\end{theorem}

\section{Algorithms} \label{sec: algorithms}

In this section, we offer two fundamental approaches to solve exploration conscious criteria using sample-based algorithms: the \emph{Expected} and \emph{Surrogate} approaches. 
For both, we formulate converging, q-learning-like, algorithms. Next, by adapting DDPG, we show the two approaches can be used in exploration-conscious continuous control as well.

Consider any fixed exploration scheme. Generally, these schemes operate in two stages: (i) Choose a greedy action, $\aChosen$. (ii) Based on $\aChosen$ and some randomness generator, choose an action to be applied on the environment, $\aEnv$. E.g., for $\epsilon$-greedy exploration, w.p. $1-\alpha$ the agent acts with $\aChosen$, otherwise, with a random uniform policy. While in RL the common update rules use $\aEnv$, the saved experience is $(s,\aEnv,r,s')$, in the following we motivate the use of  $\aChosen$, and view the data as $(s,\aChosen,\aEnv,r,s')$.

The two approaches characterized in the following are based on two, inequivalent, ways to define the $q$-function.
For the \emph{Expected} approach the $q$-function is defined as usual: $q^\pi(s,a)$ represents the value obtained when \textbf{taking an action $a=\aEnv$} and then acting with $\pi$, meaning $a$ is the action chosen in step (ii). Alternatively, for the \emph{Surrogate} approach, the $q$-function is defined on the `Surrogate' MDP, i.e., the exploration is viewed as stochasticity of the environment. Then, $q_\alpha
^\pi(s,a)$ is the value obtained when $a$ is the action of step (i), i.e., \textbf{choosing action $a=\aChosen$}.

%

\subsection{Exploration Conscious Q-Learning}\label{sec: alg alpha optimal}

We focus on solving the $\alpha$-optimal policy \eqref{eq:eps_optimization}, and formulate $q$-learning-like algorithms using the two aforementioned approaches. The \emph{Expected} $\alpha$-optimal $q$-function is,

\setlength{\abovedisplayskip}{0pt}
\setlength{\belowdisplayskip}{0pt}
\begin{align}
q^{\pi^\alpha(\pi^*_\alpha,\pi_0)}(s,a) \! \triangleq \! r(s,a) \!+\! \gamma \sum_{s'} P(s'\mid s,a)v^*_\alpha(s')\label{eq: def q pi mixture}
\end{align} 

\setlength{\abovedisplayskip}{7pt plus2pt minus5pt}
\setlength{\belowdisplayskip}{\abovedisplayskip}

Indeed, $q^{\pi^\alpha(\pi^*_\alpha,\pi_0)}$ is the usually defined $q$-function of the policy $\pi^\alpha(\pi^*_\alpha,\pi_0)$ on an MDP $\mathcal{M}$. Here, the action $a$ represents the actual performed action, $\aEnv$. 
By relating $q^{\pi^\alpha(\pi^*_\alpha,\pi_0)}$ to $v^*_\alpha$ it can be easily verified that $q^{\pi^\alpha(\pi^*_\alpha,\pi_0)}$ satisfies the fixed point equation (see Appendix \ref{supp: algorithms}),
\begin{align}
&q^{\pi^\alpha(\pi^*_\alpha,\pi_0)}(s,a) = \nonumber\\
& r(s,a)  \!+\! \gamma(1\!-\!\alpha) \sum_{s'} P(s'\!\mid\! s,a)\max_{a'} q^{\pi^\alpha(\pi^*_\alpha,\pi_0)}(s',a') \nonumber \\
&\!+\!\gamma\alpha  \sum_{s',a'} P(s'\!\mid\! s,a)\pi_0(a'\!\mid\! s') q^{\pi^\alpha(\pi^*_\alpha,\pi_0)}(s',a'). \label{eq: expected fix point}
\end{align}

Alternatively, consider the optimal $q$-function of the surrogate MDP $\mathcal{M}_\alpha$ \eqref{eq: surrogate MDP reward and dynamics}. It satisfies the fixed-point equation
\begin{align*}
q_\alpha^{*}(s,a) \! \triangleq \! r_\alpha(s,a) \!+\! \gamma \sum_{s'} P_\alpha(s'\mid s,a)\max_{a'}q_\alpha^{*}(s',a').
\end{align*}
The following lemma formalizes the relation between the two $q$-functions, and shows they are related by a function of the state, and not of the action.
\begin{lemma} \label{lemma: q alpha q pi mixture}
$q^*_\alpha(s,a) = (1-\alpha)q^{\pi^\alpha(\pi^*_\alpha,\pi_0)}(s,a)+f(s).$
\end{lemma}

The $\alpha$-optimal policy $\pi_\alpha^*$ is also an optimal policy of $\mathcal{M}_\alpha$ (Lemma \ref{lemma: equivalence}). Thus, it is greedy w.r.t. $q_\alpha^{*}$, the optimal $q$ of~$\mathcal{M}_\alpha$. By Proposition \ref{prop: alpha surrogate mdp bellman}.3 it is also greedy w.r.t. $q^{\pi^\alpha(\pi^*_\alpha,\pi_0)}$, i.e., $$\pi_\alpha^*(s) \in \arg\max_{a'} q_\alpha^{*}(s,a') =  \arg\max_{a'} q^{\pi^\alpha(\pi^*_\alpha,\pi_0)}(s,a').$$

Lemma \ref{lemma: q alpha q pi mixture} describes this fact by different means; the two $q$-functions are related by a function of the state and, thus, the greedy action w.r.t. each is equal. Furthermore, it stresses the fact that the two $q$-function are not equal.

Before describing the algorithms, we define the following notation for any $q(s,a)$,
\begin{align*}
v(s)=\max_{a'}q(s,a'), \phantom{a} v^{\pi}(s)=\sum_{a'}\pi(a'\mid s)q(s,a').
\end{align*}
We now describe the Expected $\alpha$-Q-learning algorithm (see Algorithm \ref{alg:expected alpha}), also given in \cite{john1994best,littman1997generalized}, and re-interpret it in light of the previous discussion.  

The fixed point equation \eqref{eq: expected fix point}, leads us to define the operator $T^{Eq}_\alpha$ for which ${q^{\pi^\alpha(\pi^*_\alpha,\pi_0)} = T^{Eq}_\alpha q^{\pi^\alpha(\pi^*_\alpha,\pi_0)}}$. 
Expected $\alpha$-Q-learning (Alg.~\ref{alg:expected alpha}) is a Stochastic Approximation (SA) alg. based on the operator $T^{Eq}_\alpha$. Given a sample of the form $(s,\aChosen,\aEnv,r,s')$, it updates $q(s,\aEnv)$ by
\begin{align}
(1\!-\!\eta)q(s,\aEnv)\! + \!\eta\left( r_t\! +\!\gamma ((1\!-\!\alpha)v(s_{t+1})\!+\!\alpha v^{\pi_0}(s_{t+1}))\right)\label{eq: expected q learning}
\end{align}

\begin{algorithm}
\caption{Expected $\alpha$-Q-Learning}\label{alg:expected alpha}
\begin{algorithmic}
\INITIALIZE $\alpha \in [0,1],\ \pi_0,\ q$, learning rate $\eta_t$.
\FOR{$t=0,1,...$}
	\STATE $\aChosen \gets \arg\max_a q_t(s_t,a)$
	\STATE $X_t \sim Bernoulli (1-\alpha)$
	\STATE $\aEnv = \begin{cases} \aChosen,\ \mathrm{if}\ X_t=1 \\
					a\sim \pi_0(\cdot\mid s),\ \mathrm{if}\ X_t=0 
				\end{cases}$
	\STATE $r_t,s_{t+1} \gets ACT(\aEnv)$
    \STATE $y_t \gets r_t + \gamma (1-\alpha) v_t(s_{t+1}) +\gamma \alpha v_t^{\pi_0}(s_{t+1})$ \label{eq: expected alpha eq update}
	\STATE $q(s_t,\aEnv) \gets \left(1-\eta_t\right)q(s_t,\aEnv) + \eta_t y_t$ \label{eq: expected update q es}
\ENDFOR
\STATE {\bf return:} $\pi\in \arg\max_a q(\cdot,a)$
\end{algorithmic}
\end{algorithm}

Its convergence proof is standard and follows by showing $T^{Eq}_\alpha$ is a $\gamma$-contraction and using \cite{bertsekas1995neuro}[Proposition 4.4] (see proof in Appendix \ref{supp: expected q alpha learning}). 





We now turn to describe an alternative algorithm,  which operates on the surrogate MDP, $\mathcal{M}_\alpha$, and converges to $q^*_\alpha$. Naively, given a sample $(s,\aChosen,r,s')$, regular $q$-learning on $\mathcal{M}_\alpha$ can be used by updating $q(s,\aChosen)$ as,
\begin{align}
(1\!-\!\eta_t)q(s,\aChosen)\! + \!\eta_t( r_t\! +\!\gamma v(s_{t+1})), \label{eq: naive alpha q learning 1}
\end{align} 
Yet, this approach does not utilize a meaningful knowledge; when the exploration policy $\pi_0$ is played, i.e., when $X_t=0$, the sample $(r_t,s_{t+1})$ can be used to update all the action entries from the current state. These entries are also affected by the policy $\pi_0$. In fact, we cannot prove the convergence of the naive update based on current techniques; if the greedy action is repeatedly chosen, `infinitely often' visit in all  $(s,a)$ pairs cannot be guaranteed.

\begin{algorithm}
\caption{Surrogate $\alpha$-Q-Learning}\label{alg: surrogate Q}
\begin{algorithmic}
\INITIALIZE $\alpha \in [0,1],\ \pi_0,\ q_\alpha,q$, learning rate $\eta_t$.
\FOR{$t=0,1,...$}
	\STATE $\aChosen \gets \arg\max_a q(s_t,a)$
	\STATE $X_t \sim Bernoulli (1-\alpha)$
	\STATE $\aEnv = \begin{cases} \aChosen,\ \mathrm{if}\ X_t=1 \\
					a\sim \pi_0(\cdot\mid s),\ \mathrm{if}\ X_t=0 
				\end{cases}$
	\STATE $r_t,s_{t+1} \gets ACT(\aEnv)$
	\FOR{ $\bar{a}\in \mathcal{A}$}
	\STATE  $y_t^{\bar{a}} \!\!=\!\! \begin{cases} r_t + \gamma  v_\alpha(s_{t+1}),\ \bar{a}=\aChosen \\
					X_t q(s_t,\bar{a})\!+\! (1\!-\!X_t)\left( r_t \!+\! \gamma  v_\alpha(s_{t+1}) \right)\!,\mathrm{o.w}\end{cases}$
		\STATE $q_\alpha(s_t,\bar{a}) \gets \left(1-\eta\right)q_\alpha(s_t,\bar{a}) + \eta y_t^{\bar{a}}$
	\ENDFOR
	\STATE $y_t \gets r_t + \gamma (1-\alpha) v(s_{t+1}) +\gamma \alpha v^{\pi_0}(s_{t+1})$
	\STATE $q(s_t,\aEnv) \gets (1-\eta_t)q(s_t,\aEnv) + \eta_t y_t$
	
	\ENDFOR
\STATE {\bf return} $\pi \in \arg\max_a q_\alpha(\cdot,a)$
\end{algorithmic}
\end{algorithm}

This reasoning leads us to formulate Surrogate $\alpha$-Q-learning (see Algorithm \ref{alg: surrogate Q}). The Surrogate $\alpha$-Q-learning updates two $q$-functions, $q$ and $q_\alpha$. The first, $q$, has the same update as in Expected $\alpha$-Q-learning, and thus converges (w.p~$1$) to  $q^{\pi^\alpha(\pi^*_\alpha,\pi_0)}$. The second, $q_\alpha$, updates the chosen greedy action using equation \eqref{eq: naive alpha q learning 1}, when the exploration policy is not played ($X_t=1$). By bootstrapping on $q$, the algorithm updates all other actions when the exploration policy $\pi_0$ is played ($X_t=0$). Using \cite{singh2000convergence}[Lemma 1], the convergence of Surrogate $\alpha$-Q-learning to $(q^{\pi^\alpha(\pi^*_\alpha,\pi_0)},q^*_\alpha)$ is established (see proof in Appendix \ref{supp: surrogate q alpha learning}). Interestingly, and unlike other $q$-learning algorithms (e.g, Expected $\alpha$-Q-learning, Q-learning, etc.), Surrogate $\alpha$-Q-learning updates the entire action set given a single sample. For completness, we state the convergence result for both algorithms.

\begin{theorem}\label{theorem: expected q alpha}
Consider the processes described in Alg. ~\ref{alg:expected alpha}, \ref{alg: surrogate Q}. Assume $\{\eta_t \}_{t=0}^\infty$ satisfies  ${\forall s\in \mathcal{S}}$, ${\forall a\in \mathcal{A}}$, ${\sum_{t=0}^\infty \eta_t =\infty}$, and ${\sum_{t=0}^\infty \eta_t^2 <\infty}$, where $\eta_t \equiv \eta_t (s_t=s,a_{\mathrm{env},t}=a)$. Then, for both \ref{alg:expected alpha}, \ref{alg: surrogate Q} the sequence $\{q_n\}_{n=0}^\infty$ converges w.p. 1 to $q^{\pi^\alpha(\pi^*_\alpha,\pi_0)}$, and for \ref{alg: surrogate Q}, $\{q_{\alpha,n}\}_{n=0}^\infty$ converges w.p. 1 to $q^*_\alpha$.
\end{theorem}

%


\subsection{Continuous Control}\label{sec: alg continuous control}

Building on the two approaches for solving Exploration Conscious criteria, we suggest two techniques to find an optimal Gaussian  policy \eqref{eq:continuous eps_optimization} using gradient based Deep RL (DRL) algorithms, and specifically, DDPG \cite{lillicrap2015continuous}. Nonetheless, the techniques are generalizable to other actor-critic, DRL algorithms  \cite{schulman2017proximal}. 

Assume we wish to find an optimal Gaussian policy by parameterizing its mean~$\mu(\phi)$. \citet{nachum2018smoothed}[Eq. 13] showed the gradient of the value w.r.t. $\phi$ is similar to \citet{silver2014deterministic},
\begin{align}
\nabla_\phi v^{\pi_{\mu,\sigma}} = \int_{\mathcal{S}} \partial_a q_\sigma^{\pi_{\pi_{\mu,\sigma}}}(s,a) \nabla_\phi \mu^\theta(s) d\rho^{\pi_{\mu,\sigma}}(s), \label{eq: gradients of surrogate}
\end{align}
where ${q_{\sigma}^{\mu}(s,a) = r_\sigma(s,a)\!+\!\gamma\int_{\mathcal{S}} p_\sigma(s'\mid s,a)v^{\pi_{\mu,\sigma}}(s')ds'}$, is the $q$-function of the surrogate MDP. In light of previous section, we interpret $q_{\sigma}^{\mu}$ as the $q$-function of the surrogate MDP's $\mathcal{M}_\sigma$ \eqref{eq: surrogate gaussian MDP reward and dynamics}. Furthermore, we have the following relation between the surrogate and expected $q$-functions, ${q_{\sigma}^{\mu}(s,a) = \int_{a'\in \mathcal{A}} \mathcal{N}(a'\mid a,\sigma) q^{\pi_{\mu,\sigma} }(s,a') da'}$, from which it is easy to verify that (see Appendix~\ref{supp: eq gradient equivalence}),
\begin{align}
 {{\nabla }_{u}}{q^{\pi_{\mu,\sigma} }_{\sigma}} (s,b)\!=\! \int_{ \mathcal{A}} \! \mathcal{N}(b \mid a,\sigma){\nabla }_{b} q^{\pi_{\mu,\sigma} }(s,b)db\label{eq: gradients of expected}.
\end{align}

Thus, we can update the actor in two inequivalent ways, by using gradients on the surrogate MDP's $q$-function \eqref{eq: gradients of surrogate}, or by using gradients of the expected $q$-function \eqref{eq: gradients of expected}.

The updates of the critic, $q_\sigma^\mu$ or $q^{\pi_{\mu,\sigma}}$, can be done using the same notion that led to the two forms of updates in~\eqref{eq: naive alpha q learning 1}-\eqref{eq: expected q learning}. When using Gaussian noise, one performs the two stages defined in Section~\ref{sec: algorithms}, where $\aChosen$ is the output of the actor $\mu(s)$, and $\aEnv \sim \N(\aChosen,\sigma)$.
Then, the sample $(s,\aChosen,\aEnv,r,s')$ is obtained by interacting with the environment. Based on the the fixed policy TD-error defined in \eqref{eq: naive alpha q learning 1},  we define the following loss function, for learning $q_\sigma^\mu$, q-function of the fixed policy $\mu$ over $\M_\sigma$,
\begin{align*}
    \left(q_\sigma^\theta(s,\aChosen) - r-\gamma q_{\sigma}^{\theta-}(s',\mu^{\phi-}(s'))\right)^2.
\end{align*}

On the other hand, we can define a loss function derived from the fixed-policy TD-error defined in \eqref{eq: expected q learning}, for learning $q^{\pi_{\mu,\sigma}}$, the $q$-function of the Gaussian policy with mean and variance $\mu,\sigma^2$ over $\M$,
\begin{align*}
    \big(q^\theta(s,\aEnv) \!-\! r-\gamma \! \int_\mathcal{A}\!\mathcal{N}(b\mid \mu^{\phi-}(s'),s') q^{\theta-}(s',b) db \big)^2.
\end{align*}

\section{Experiments}\label{sec: experiments} 
In this section, we test the theory and algorithms \footnote{Implementation of the proposed algorithms can be found in https://github.com/shanlior/ExplorationConsciousRL.} suggested in this work. In all experiments we used $\gamma=0.99$. The tested DRL algorithms in this section (See Appendix \ref{supp: algorithms psuedocode})  are simple variations of DDQN \cite{van2016deep} and DDPG \cite{lillicrap2015continuous},  without any parameter tuning, and based on Section \ref{sec: algorithms}. For example, for the surrogate approach in both DDQN and DDPG we merely save $(s,\aChosen,r,s')$ instead of $(s,\aEnv,r,s')$ in the replay buffer (see Section \ref{sec: algorithms} for definitions of $\aEnv,\aChosen$). 

We observe a significant improved empirical performance, both in {\bf training} and {\bf evaluation} for both the surrogate and expected approaches relatively to the baseline performance. The improved training performance is predictable; the learned policy is optimal w.r.t. the noise which is being played. In large portion of the results, the exploration-conscious criteria leads to better performance in evaluation.

\subsection{Exploration Consciousness with Prior Knowledge}

\begin{figure}
	\centering
 \begin{subfigure}{0.23\textwidth}    
\includegraphics[width=\textwidth]{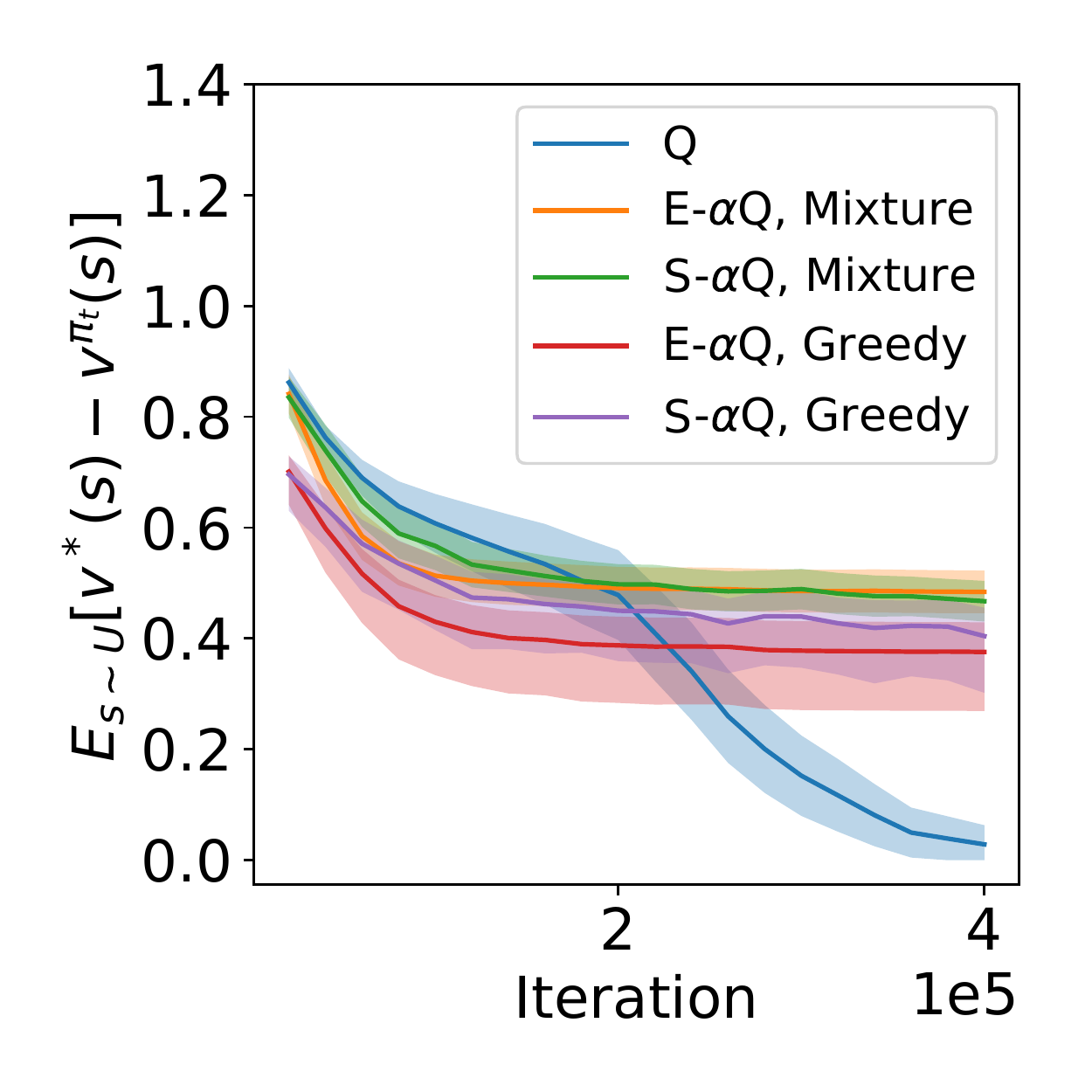} 
\end{subfigure} %
\hfill
 \begin{subfigure}{.23\textwidth}    
	\centering
\includegraphics[width=\textwidth]{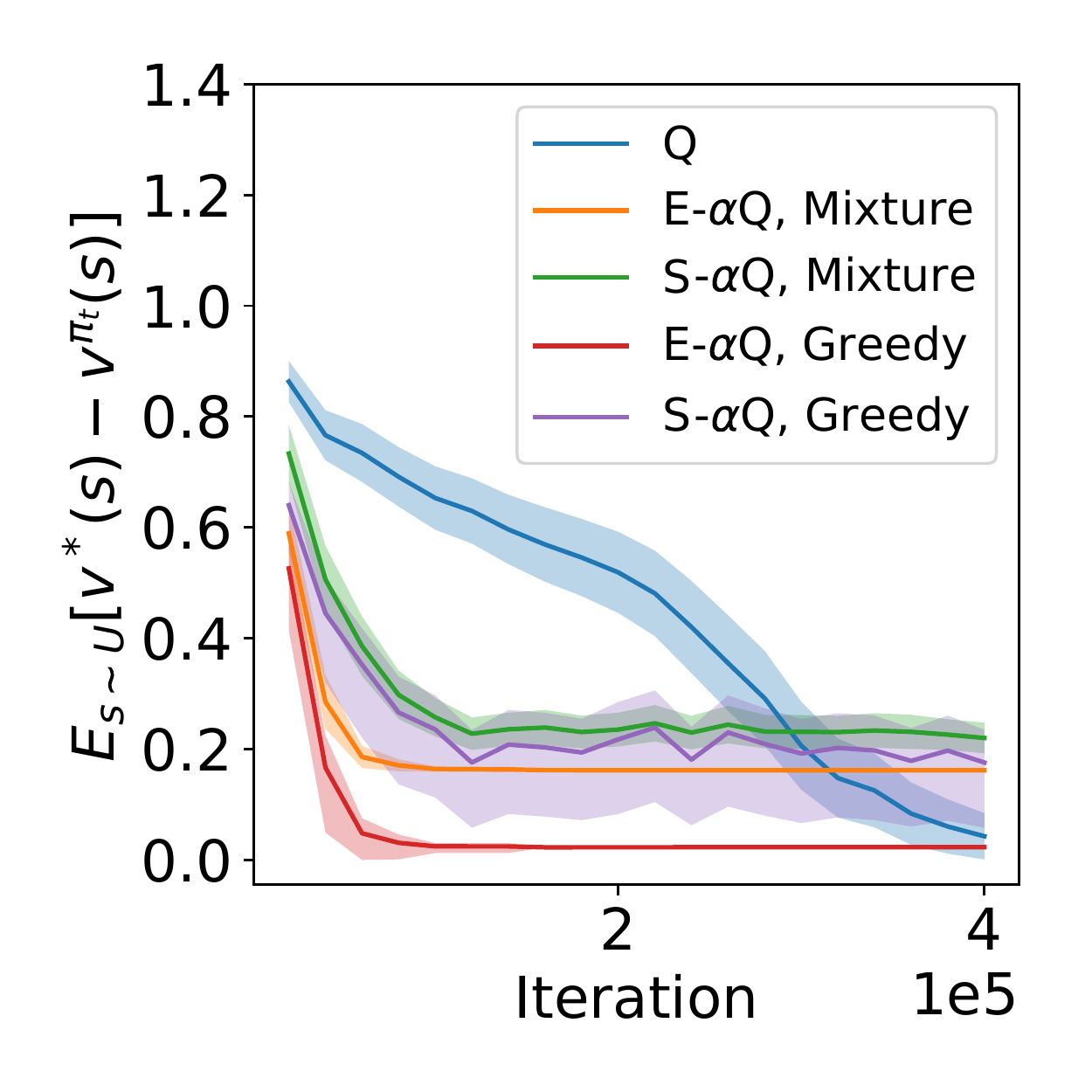} 
\end{subfigure} %
\caption{T-Cliff-Walking for the expected (E) and surrogate (S) approaches. (Left) $\alpha$=0.3. (Right) $\alpha(s)$ from prior knowledge.} 
	\label{fig:exp2}
\end{figure}

We use an adaptation of the Cliff-Walking maze \cite{sutton1998reinforcement} we term T-Cliff-Walking (see Appendix \ref{supp: experiments}). The agent starts at the bottom-left side of a maze, and needs to get to the bottom-right side goal state with value $+1$. If the agent falls off the cliff, the episode terminates with reward $-1$. When the agent visits any of the first three steps on top of the cliff, it gets a reward of $0.01\cdot(1-\gamma)$.

We tested Expected $\alpha$-Q-learning, Surrogate $\alpha$-Q-learning, and compared their performance to Q-learning in the presence of $\epsilon$-greedy exploration. Figure~\ref{fig:exp2} stresses the typical behaviour of the $\alpha$-optimality criterion. It is easier to approximate $\pi^\alpha(\pi^*_\alpha,\pi_0)$ than the optimal policy. Further, by being exploration-consciousness, the value of the approximated policy improves faster using the $\alpha$-optimal algorithms; it learns faster which regions to avoid. As Proposition \ref{theorem:performance model free} suggests, the value of the learned policy is biased w.r.t $v^*$. Next, as suggested by Proposition \ref{proposition: alpha optimal improvement}, acting greedily w.r.t. the approximated value attains better performance. Such improvement is not guaranteed while the value had not yet converged to $v^*_\alpha$. However, the empirical results suggest that if the agent performs well over the mixture policy, it is worth using the greedy policy.

We show that it is possible to incorporate prior knowledge to decrease the bias caused by being Exploration-Conscious. The T-Cliff-Walking example demands high exploration, $\alpha=\epsilon=0.3$, because of the bottleneck state between the two sides of the maze. The $\alpha$-optimal policy in such case is to stay at the left part of the maze. We used the prior knowledge that $L(s)$ close to the barrier is high. The knowledge was injected through the choice of $\alpha$, i.e., we chose a state-wise exploration scheme with $\alpha(s)=\epsilon(s)=0.1$ in the passage and the two states around it, and $\alpha(s)=0.3$ elsewhere, for all three algorithms. The results in Figure~\ref{fig:exp2} suggests that using prior knowledge to set $\alpha(s)$, can increase the performance by reducing the bias. In contrast, such prior knowledge does not help the baseline q-learning.

\subsection{Exploration Consciousness in Atari}

We tested the $\alpha$-optimal criterion in the more complex function approximation setting (see Appendix Alg. \ref{alg:Expected alpha DDQN}, \ref{alg:Surrogate alpha DDQN}). We used five Atari 2600 games (\ref{fig:DeepGames}) from the ALE \cite{bellemare2013arcade}. We chose games that resemble the Cliff Walking scenario, where the wrong choice of action can lead to a sudden termination of the episode. Thus, being unaware of the exploration strategy can lead to poor training results. We used the same deep neural network as in DQN \cite{mnih2015human}, using the openAI Baselines implementation \cite{baselines}, without \emph{any parameter tuning}, except for the update equations. We chose to use the Double-DQN variant of DQN \cite{van2016deep} for simplicity and generality. Nonetheless, changing the optimality criterion is orthogonal to any of the suggested add-ons to DQN \cite{hessel2017rainbow}. We used $\alpha=\epsilon=0.01$ in the train phase, and $\epsilon=0.001$ in the evaluation phase. For the \emph{surrogate} version, we used a naive implementation based on equation~\eqref{eq: naive alpha q learning 1}.

\begin{table}\caption{Train and Test rewards for the Atari 2600 environment, with 90\% confidence interval}
\centering
\label{fig: DeepResultsAtari}
    \begin{tabular}{L{0.6cm}|L{1.5cm}|P{1.3cm}|P{1.3cm}|P{1.3cm}}
	\toprule
	 & Game  & \begin{tabular}{@{}c@{}} DDQN\end{tabular}  
	 & \begin{tabular}{@{}c@{}} Expected \\ $\alpha$-DDQN\end{tabular} & \begin{tabular}{@{}c@{}} Surrogate \\ $\alpha$-DDQN\end{tabular}\\
	 \toprule
	\multirow{5}{*}{Train} & Breakout & \val{350}{4} & \val{356}{6} & \valbf{357}{4} \\
	& FishingDer & \val{-45}{9} & \val{-35}{27} & \valbf{-8}{8} \\
	& Frostbite & \val{1191}{171} & \val{794}{158} & \valbf{1908}{162} \\
	& Qbert & \val{13221}{565} & \val{13431}{178} & \valbf{14240}{225} \\
	& Riverraid & \val{8602}{205} & \val{8811}{645} & \valbf{11476}{79} \\
	\hline
	\multirow{5}{*}{Test} & Breakout & \valbf{402}{14} & \val{390}{5} & \val{392}{5} \\
	& FishingDer & \val{-37}{15} & \val{-19}{34} & \valbf{-3}{19} \\
	& Frostbite & \val{1720}{191} & \val{1638}{292} & \valbf{2686}{278} \\
	& Qbert & \val{15627}{497} & \val{15780}{206} & \valbf{16082}{338} \\
	& Riverraid & \val{9049}{443} & \val{9491}{802} & \valbf{12846}{241} \\
	\bottomrule
	\end{tabular}
\end{table}	
Table~\ref{fig: DeepResultsAtari} shows that our method improves upon using the optimal criterion. That is, while bias exists, the algorithm still converges to a better policy. This result holds both on the exploratory training regime and the evaluation regime. Again, acting greedy w.r.t. the approximation of the $\alpha$-optimal policy proved beneficial: The evaluation phase results surpasses the train phase results as shown in the table, and the training figures in Appendix~(\ref{fig: DeepGraphs}). The evaluation is usually done with an $\epsilon=0.001>0$. Proposition \ref{proposition: alpha optimal improvement} put formal grounds for using smaller $\epsilon$ in the evaluation phase than in the training phase; improvement is assured. Being accurate is extremely important in most Atari games, so Exploration-Consciousness can also hurt the performance. Still, one can use prior knowledge to overcome this obstacle.

\subsection{Exploration Consciousness in MuJoCo}

We tested the Expected $\sigma$-DDPG~(\ref{alg:Expected sigma DDPG}) and Surrogate $\sigma$-DDPG~(\ref{alg:Surrogate sigma DDPG}) on continuous control tasks from the MuJoCo environment \cite{todorov2012mujoco}. We used the OpenAI implementation of DDPG as the baseline, where we only changed the update equations to match our proposed algorithms. We used the default hyper-parameters, and independent Gaussian noise with $\sigma=0.2$, for all tasks and algorithms. The results in Table~\ref{fig: DeepResults Mujoco} were averaged over 10 different seeds. The performance of the $\sigma$-optimal variants superseded the baseline DDPG, for most of the training and test results. Interestingly, although improvement is not guaranteed (Proposition \ref{prop: NoImprovementContinuous}), the $\sigma$-optimal policy improved when using $\mu^\phi$ deterministically, i.e., in the test phase. This suggests that improvement can be expected on certain scenarios, although that generally it is not guaranteed. We also found that the training process was faster using the $\sigma$-optimal algorithms, as can be seen in the learning curves in Appendix~\ref{fig: MujocoGraphs}. Interestingly, again, the surrogate approach proved superior.

\begin{table}\caption{Train and Test rewards for the MuJoCo environment.}
\centering
\label{fig: DeepResults Mujoco}
	\begin{tabular}{L{0.6cm}|L{1.6cm}|P{1.3cm}|P{1.3cm}|P{1.3cm}}
	\toprule
	 & Game  & \begin{tabular}{@{}c@{}} DDPG \end{tabular}  
	 & \begin{tabular}{@{}c@{}} Expected \\ $\sigma$-DDPG\end{tabular} & \begin{tabular}{@{}c@{}} Surrogate \\ $\sigma$-DDPG\end{tabular} \\\toprule
	\multirow{6}{*}{Train}
	& Ant & \val{809}{47} & \valbf{1013}{49} & \val{993}{110} \\
	& HalfCheetah & \val{2255}{804} & \val{2634}{828} & \valbf{3848}{248} \\
	& Hopper & \val{1864}{139} & \val{1866}{132} & \valbf{2566}{155} \\
	& Humanoid & \val{1281}{142} & \val{1416}{155} & \valbf{1703}{272} \\
	& InPendulum & \val{694}{109} & \val{882}{33} & \valbf{998}{3} \\
	& Walker & \val{1722}{170} & \val{2144}{145} & \valbf{2587}{214} \\
	\hline
	\multirow{6}{*}{Test}
	& Ant & \val{1611}{120} & \valbf{1924}{126} & \val{1754}{184} \\
	& HalfCheetah & \val{2729}{936} & \val{3147}{986} & \valbf{4579}{298} \\
	& Hopper & \valbf{3099}{113} & \val{3071}{50} & \val{3037}{78} \\
	& Humanoid & \val{1688}{223} & \val{1994}{389} & \valbf{2154}{408} \\
	& InPendulum & \val{999}{2} & \val{1000}{0} & \valbf{1000}{0} \\
	& Walker & \val{3031}{298} & \val{3315}{147} & \valbf{3501}{240} \\
	\bottomrule
	\end{tabular} 
\end{table}

\section{Relation to existing work}

Lately, several works have tackled the exploration problem for deep RL. In some, like Bootstrapped-DQN (see appendix [D.1] in \cite{osband2016deep}), the authors still employ an $\epsilon$-greedy mechanism on top of their methods. Moreover, methods like Distributional-DQN \cite{bellemare2017distributional,dabney2018distributional} and the state-of-the-art Ape-X DQN \cite{horgan2018distributed}, still uses $\epsilon$-greedy and Gaussian noise, for discrete and continuous actions, respectively. Hence, all the above works are applicable for the $\alpha$-optimal criterion by using the simple techniques described in Section \ref{sec: algorithms}.


Existing on-policy methods produce variants of Exploration-Consciousness. In TRPO and A3C \cite{schulman2015trust,mnih2016asynchronous}, the exploration is \emph{implicitly injected} into the agent policy through entropy regularization, and the agent improves upon the value of the explorative policy. Simple derivation shows the $\alpha$-greedy and the Gaussian approaches are both equivalent to regularizing the entropy to be higher than a certain value by setting $\alpha$ or $\sigma$ appropriately.


Expected $\alpha$-Q-learning highlights a relation to algorithms analysed in \cite{john1994best,littman1997generalized} and to Expected-Sarsa (ES)  \cite{van2009theoretical}. The focus of \cite{john1994best,littman1997generalized} is exploration-conscious q-based methods. In ES, when setting the `estimation policy' \cite{van2009theoretical} to be $\pi = (1-\alpha_t)\pi_{\G}+\alpha_t \pi_0$, we get similar updating equations as in lines \ref{eq: expected alpha eq update}-\ref{eq: expected update q es}, and similarly to \cite{john1994best,littman1997generalized}. However, in ES $\alpha_t$ decays to zero, and the \emph{optimal policy} is obtained in the infinite time limit. In \cite{nachum2018smoothed}, the authors offer a gradient based mechanism for updating the mean and variance of the actor. Here, we offer and analyze the approach of setting $\alpha_t$ and $\sigma_t$ to a constant value. This would be of interest especially when a `good' mechanism for decaying $\alpha_t$ and $\sigma_t$ lacks; the decay mechanism is usually chosen by trial-and-error, and is not clear how it should be set.

Lastly, \eqref{eq:eps_optimization} and \eqref{eq:continuous eps_optimization} can be understood as defining a `surrogate problem', rather than finding an optimal policy. In this sense, it offers an alternative approach to biasing the problem by lowering the discount-factor, i.e., solve a surrogate MDP with $\bar{\gamma}<\gamma$ \cite{petrik2009biasing,jiang2015dependence}. Interestingly, the introduced bias when solving \eqref{eq:eps_optimization} is proportional to a \emph{local property} of $v^*$, $L(s)$, that can be estimated using prior-knowledge on the MDP, where solving an MDP with $\bar{\gamma}$ introduces a bias proportional to a \emph{non-local} term, which is harder to estimate.
More importantly, the performance of an $\alpha$-optimal policy $\pi^*_\alpha$ is assured to improve when tested on the original MDP $\mathcal{M}$ (Proposition~\ref{proposition: alpha optimal improvement}), while  the performance of an optimal policy in an MDP with $\bar{\gamma}$ might decline when tested on $\M$ with $\gamma$-discounting.




\section{Summary}

In this paper, we revisited the notion of an agent being conscious to an exploration process. To our view, this notion did not receive the proper attention, though it is implicitly and repeatedly used.  

We started by formally defining \emph{optimal policy} w.r.t. an exploration mechanism \eqref{eq:eps_optimization}, \eqref{eq:continuous eps_optimization}.  This expanded the view on exploration-conscious q-learning \cite{john1994best,littman1997generalized} to a more general one, and lead us to derive new algorithms, as well as re-interpreting existing ones \cite{van2009theoretical}. We formulated the surrogate MDP notion, which helped us to establish that exploration-conscious criteria can be solved by Dynamic Programming, or, more generally, by an MDP solver. From the practical side, based on the theory, we tested DRL algorithms -- by simply modifying existing ones, with no further hyper-parameter tuning -- and empirically showed their superiority.

Although a bias - error sensitivity tradeoff was formulated, we did not prove \eqref{eq:eps_optimization}, \eqref{eq:continuous eps_optimization} are easier to solve than an MDP. We believe proving whether the claim is true is of interest. Furthermore, analyzing more exploration-conscious criteria, e.g., exploration-conscious  w.r.t. Ornstein-Uhlenbeck noise, is of interest, as well as defining a unified framework for exploration-conscious criteria.




\section*{Acknowledgments}
We would like to thank Chen Tessler, Nadav Merlis and Tom Zahavy for helpful discussions.

\bibliography{main_ICML}
\bibliographystyle{icml2019}

\onecolumn

\begin{appendices}


\onecolumn
\section{Training graphs for the Atari and MuJoCo experiments}
\begin{figure}[H]
\centering
\begin{subfigure}{0.5\textwidth}   
 	\centering
 \includegraphics[width=0.42\textwidth]{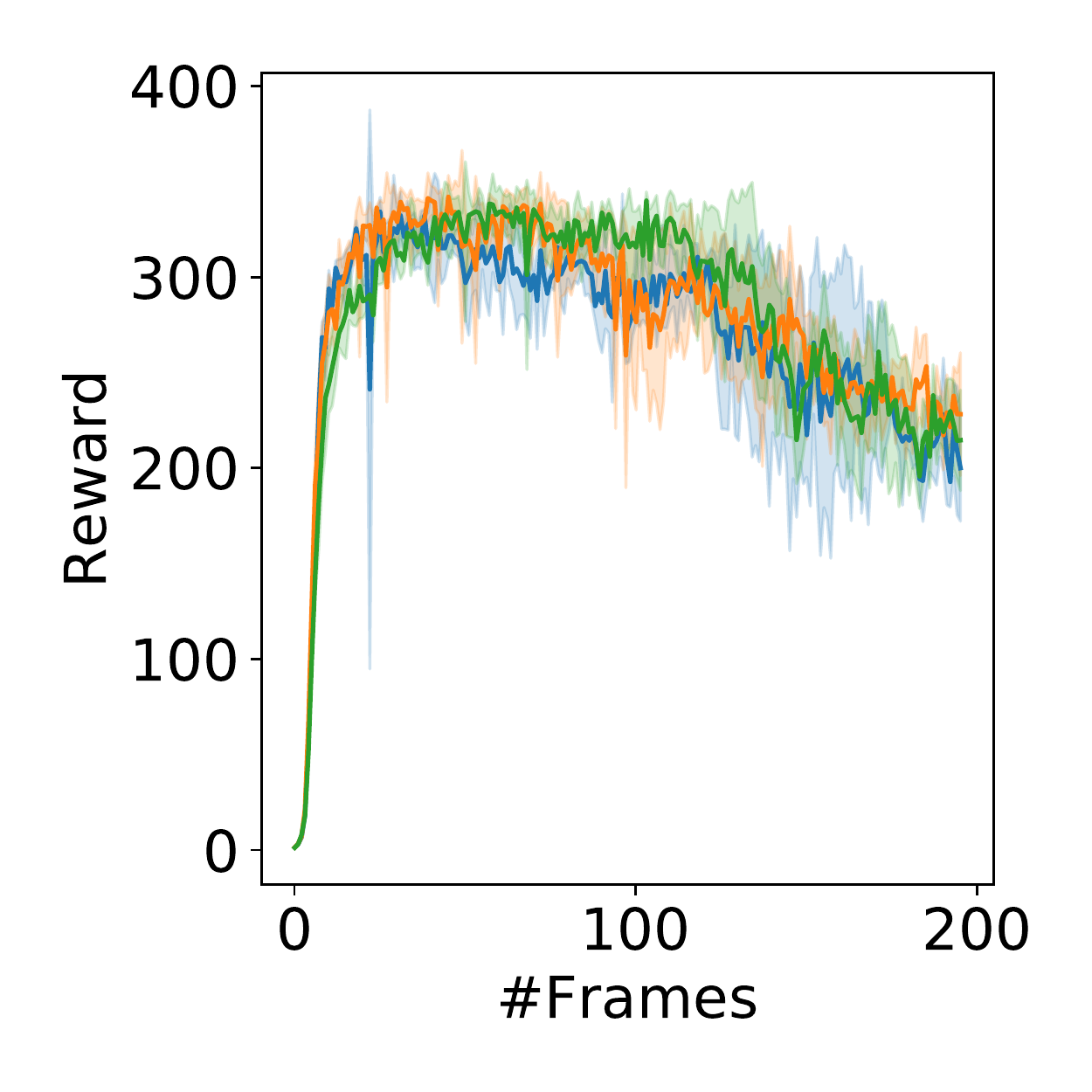}
 \hfill
 \includegraphics[width=0.42\textwidth]{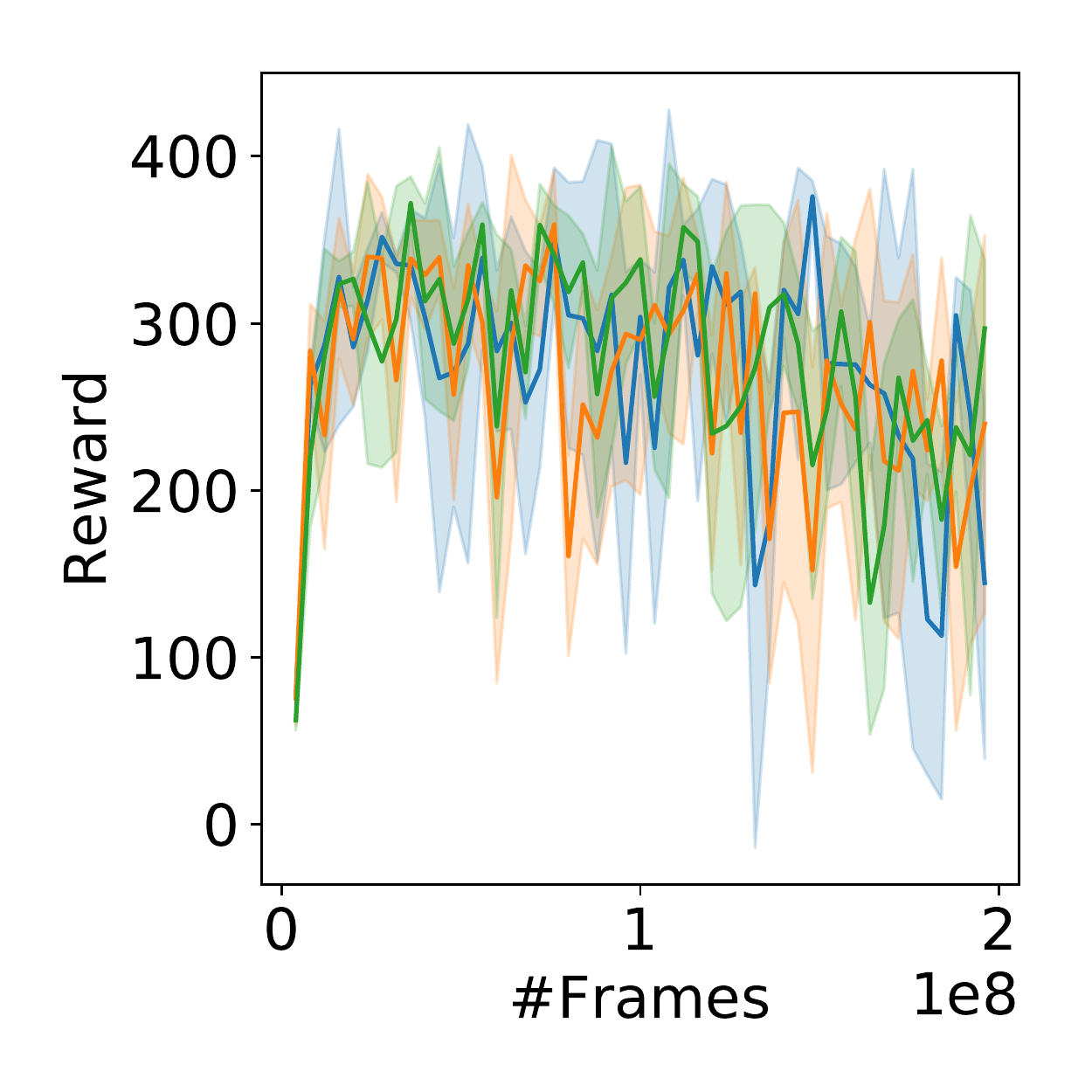} 
\end{subfigure} \\
\centering
\begin{subfigure}{0.5\textwidth}  
 	\centering
 \includegraphics[width=0.42\textwidth]{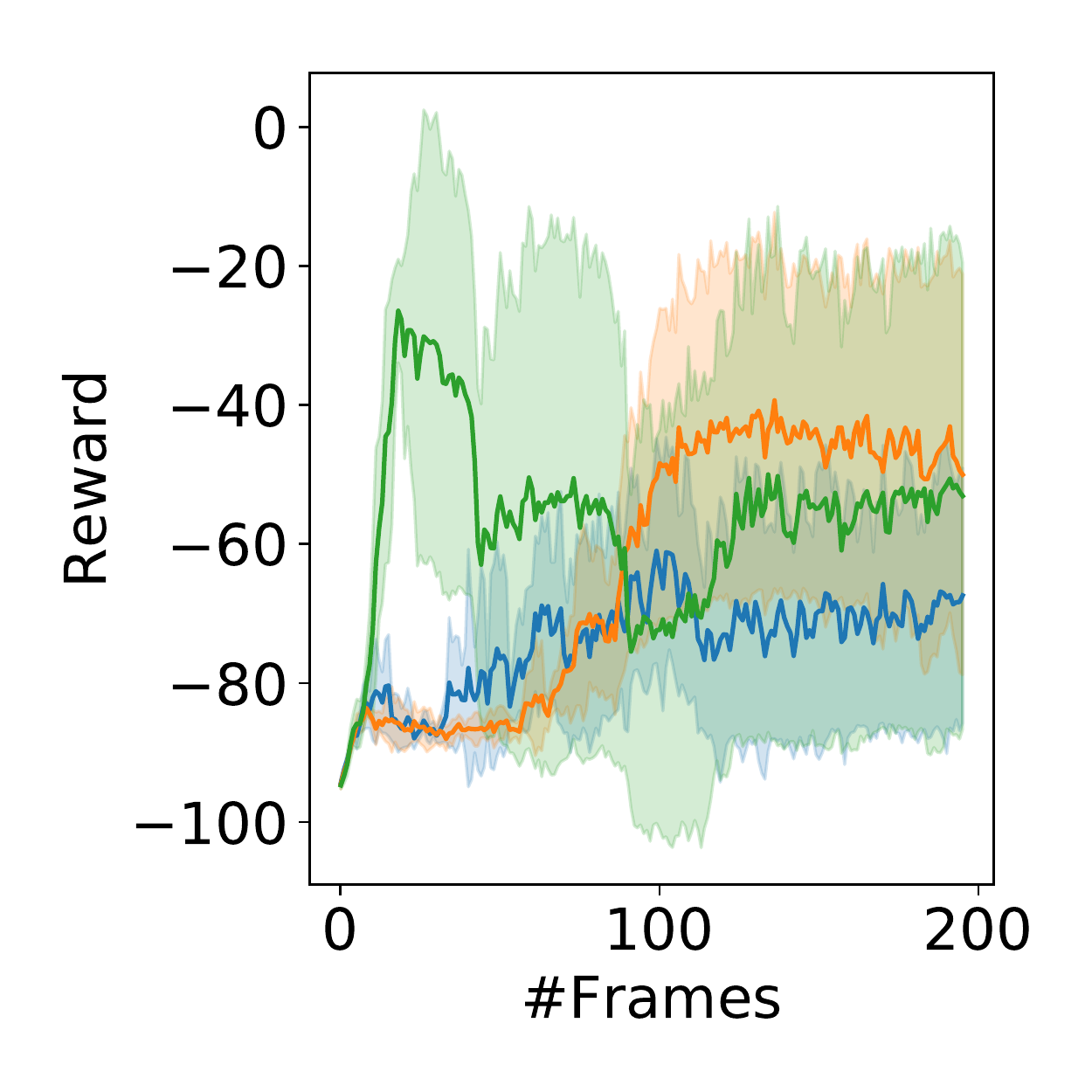}
  \hfill
 \includegraphics[width=0.42\textwidth]{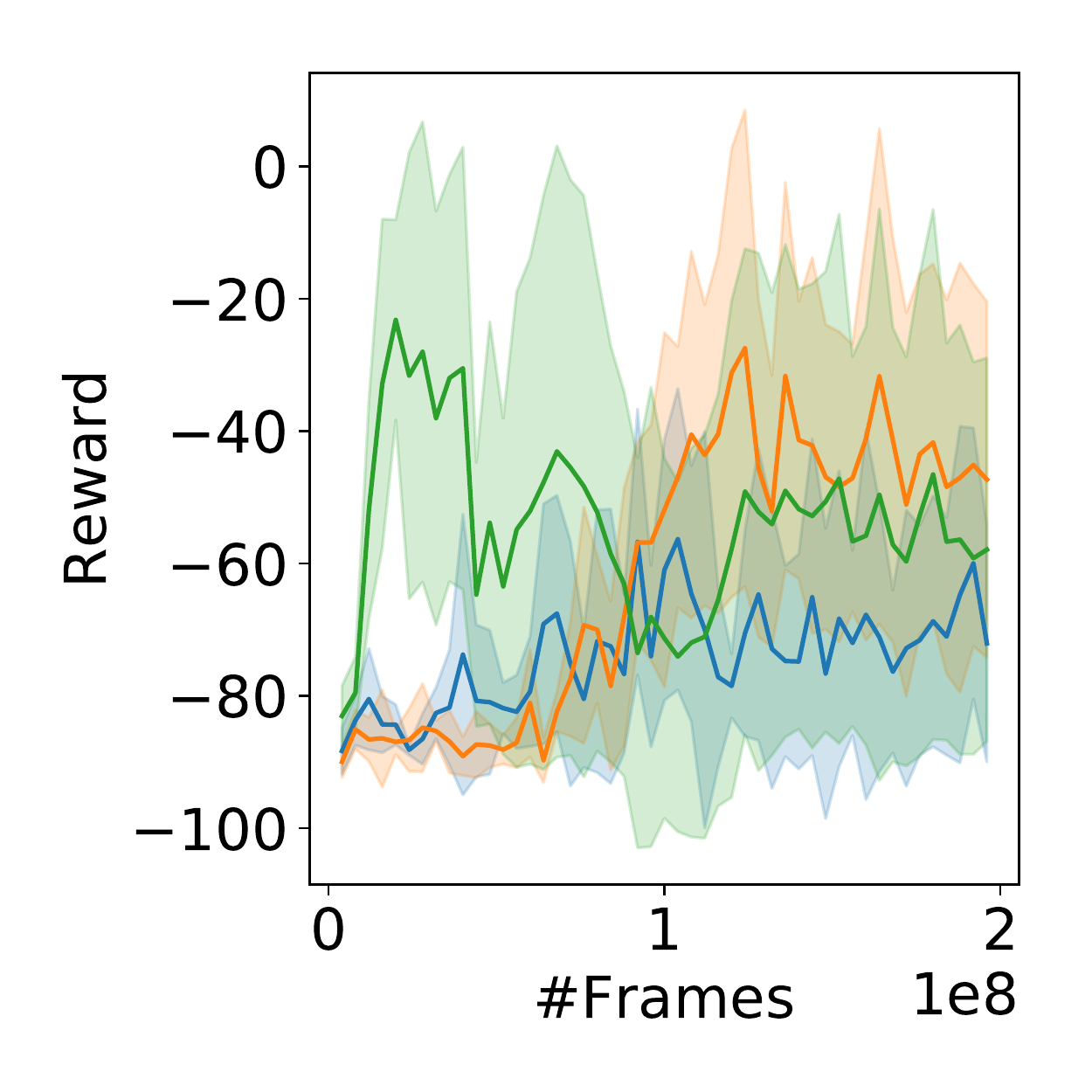}
\end{subfigure} \\
\begin{subfigure}{0.5\textwidth}   
	\centering
\includegraphics[width=0.42\textwidth]{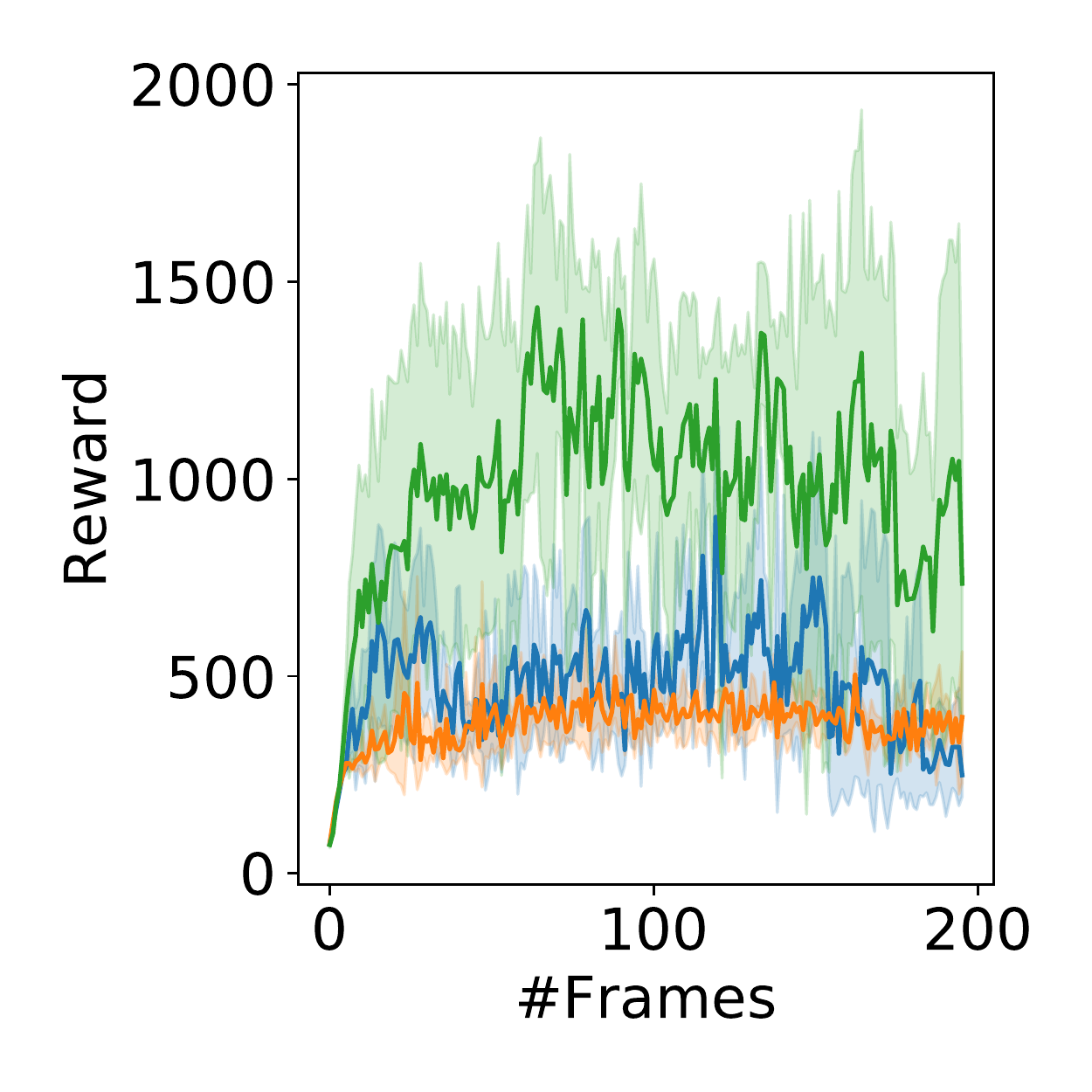}
\hfill
\includegraphics[width=0.42\textwidth]{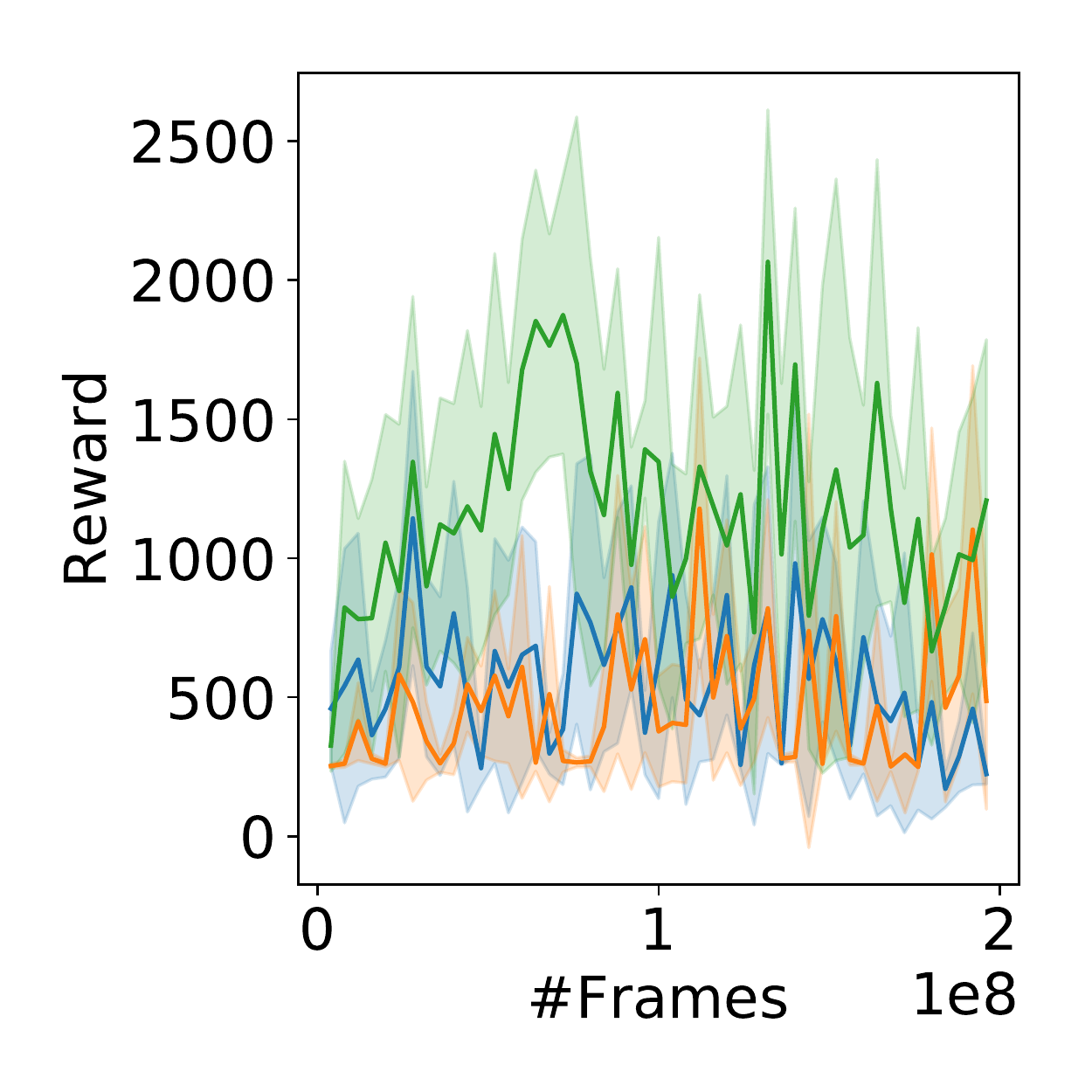} 
\end{subfigure} \\
\begin{subfigure}{0.5\textwidth}     
 	\centering
 \includegraphics[width=0.42\textwidth]{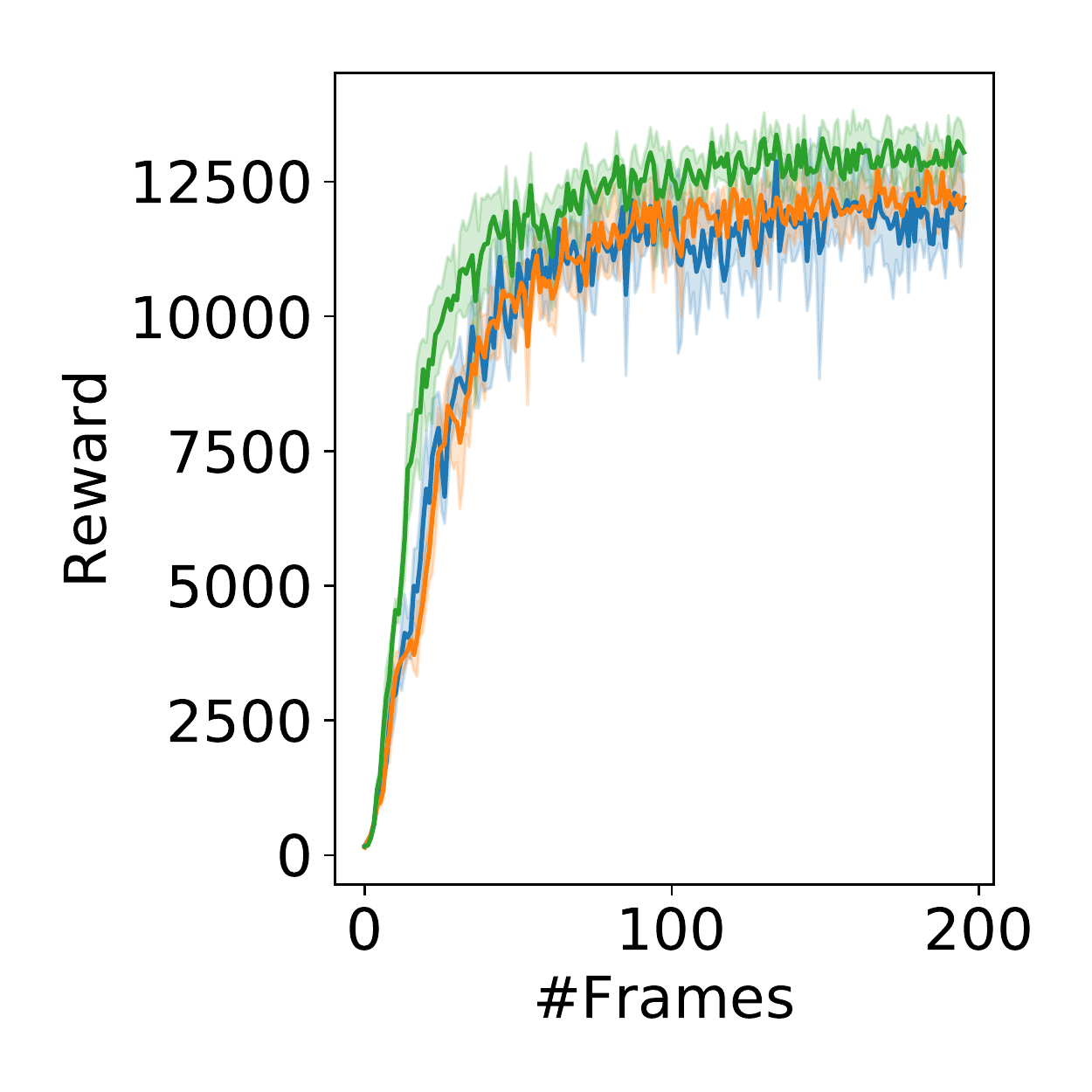}
 \hfill
 \includegraphics[width=0.42\textwidth]{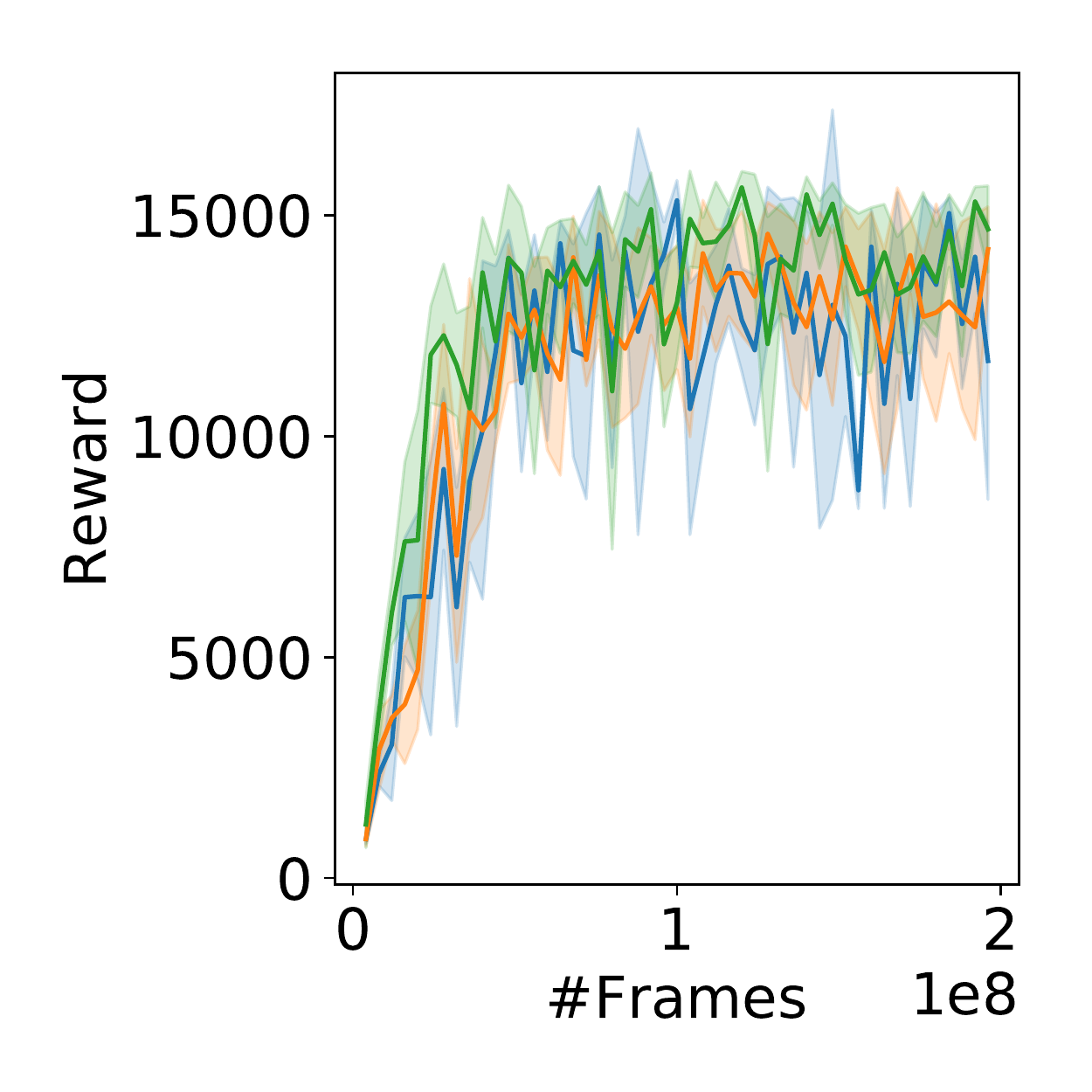} 
\end{subfigure} \\
\begin{subfigure}{0.5\textwidth}   
\centering
\includegraphics[width=0.42\textwidth]{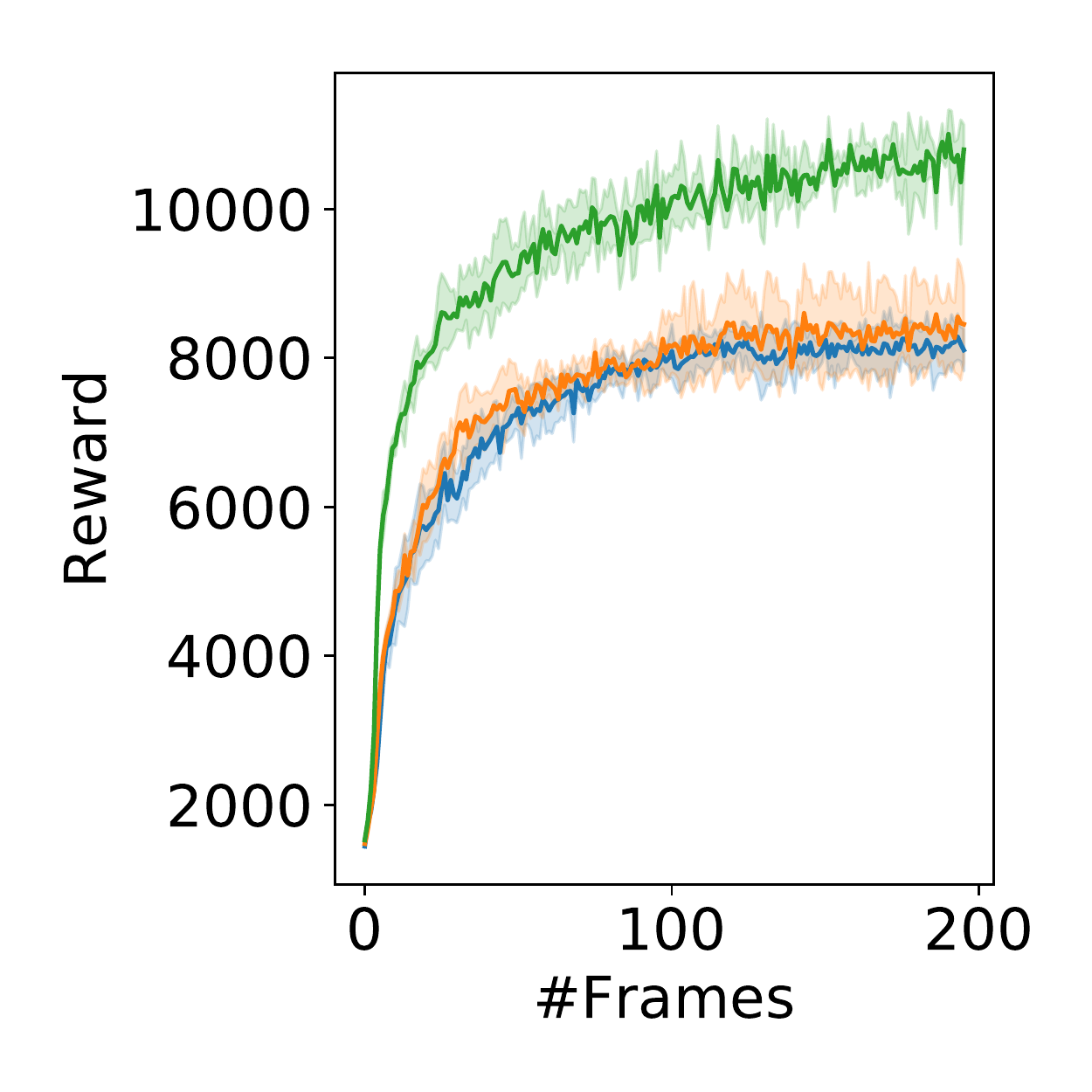}
 \hfill
\includegraphics[width=0.42\textwidth]{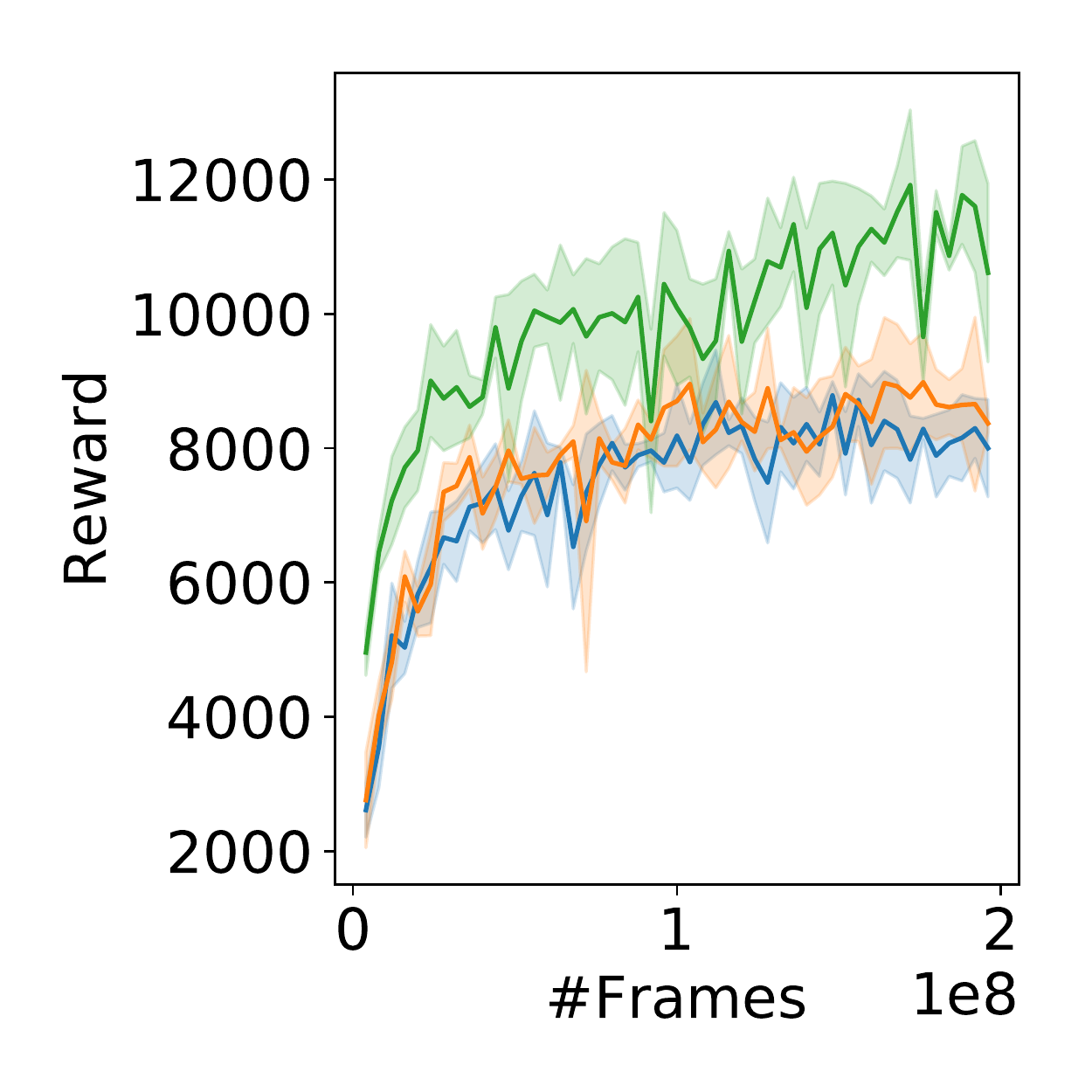} 
\end{subfigure} \\
\begin{subfigure}{0.6\textwidth}   
\centering
\includegraphics[width=\linewidth,center]{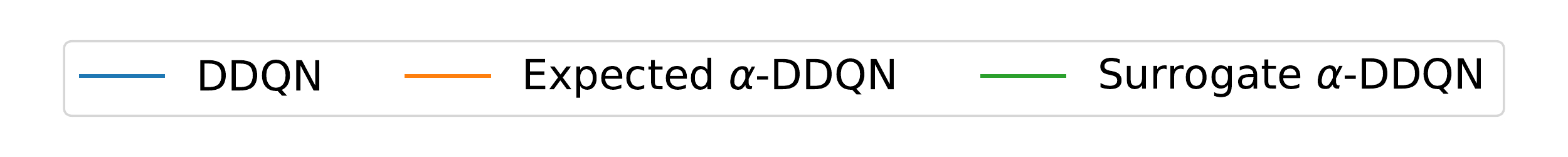}
\end{subfigure}

\caption{Simulation results for the Atari 2600 environment: From up to bottom: Breakout, Fishing Derby, Frostbite, Qbert and Riverraid. (Left) Training. (Right) Test.}\label{fig: DeepGraphs}
\end{figure}

\begin{figure}
\centering
\begin{subfigure}{0.5\textwidth}   
 	\centering
 \includegraphics[width=0.4\textwidth]{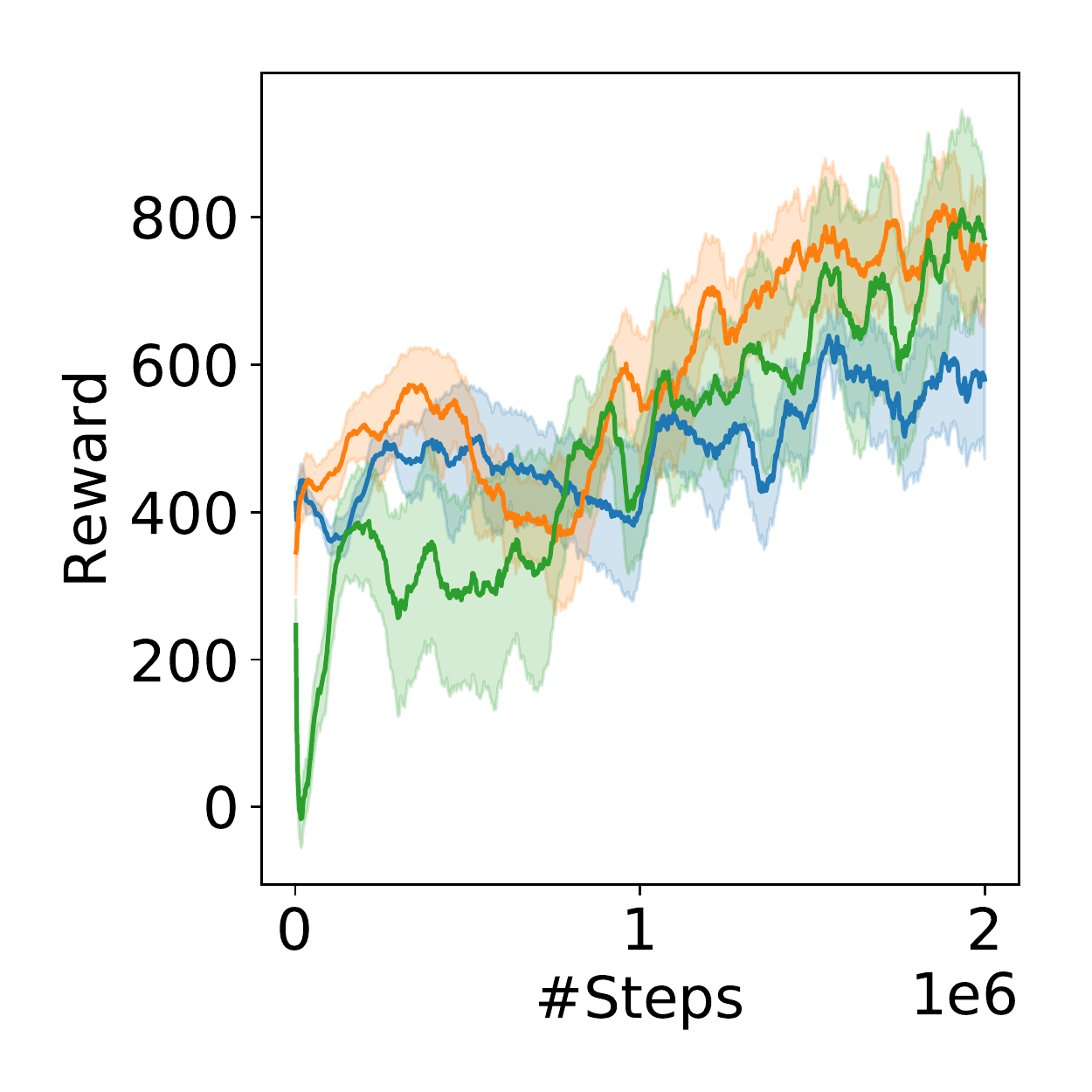}
 \hfill
 \includegraphics[width=0.4\textwidth]{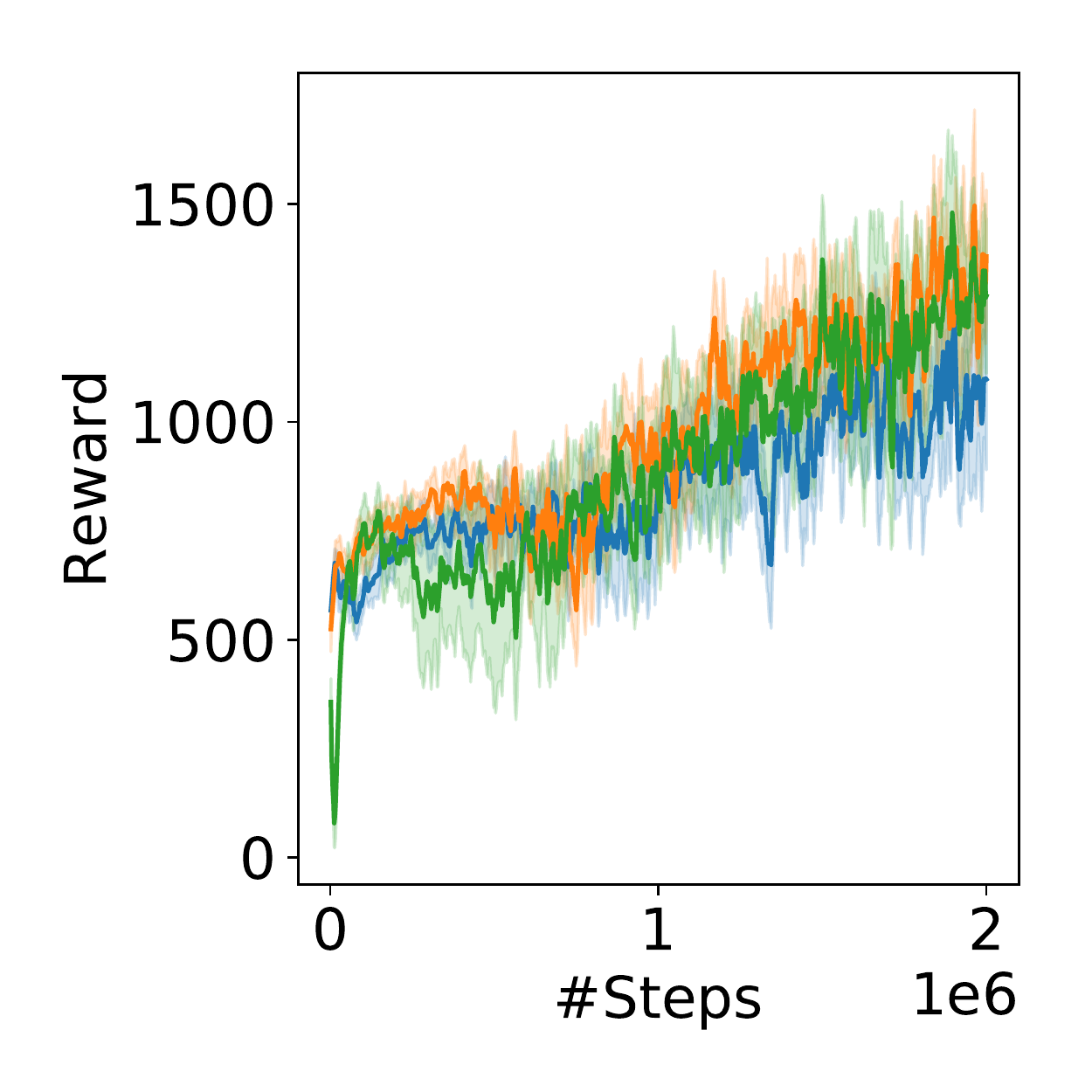} 

\end{subfigure}

\begin{subfigure}[b]{0.5\textwidth}   
	\centering
\includegraphics[width=0.4\textwidth]{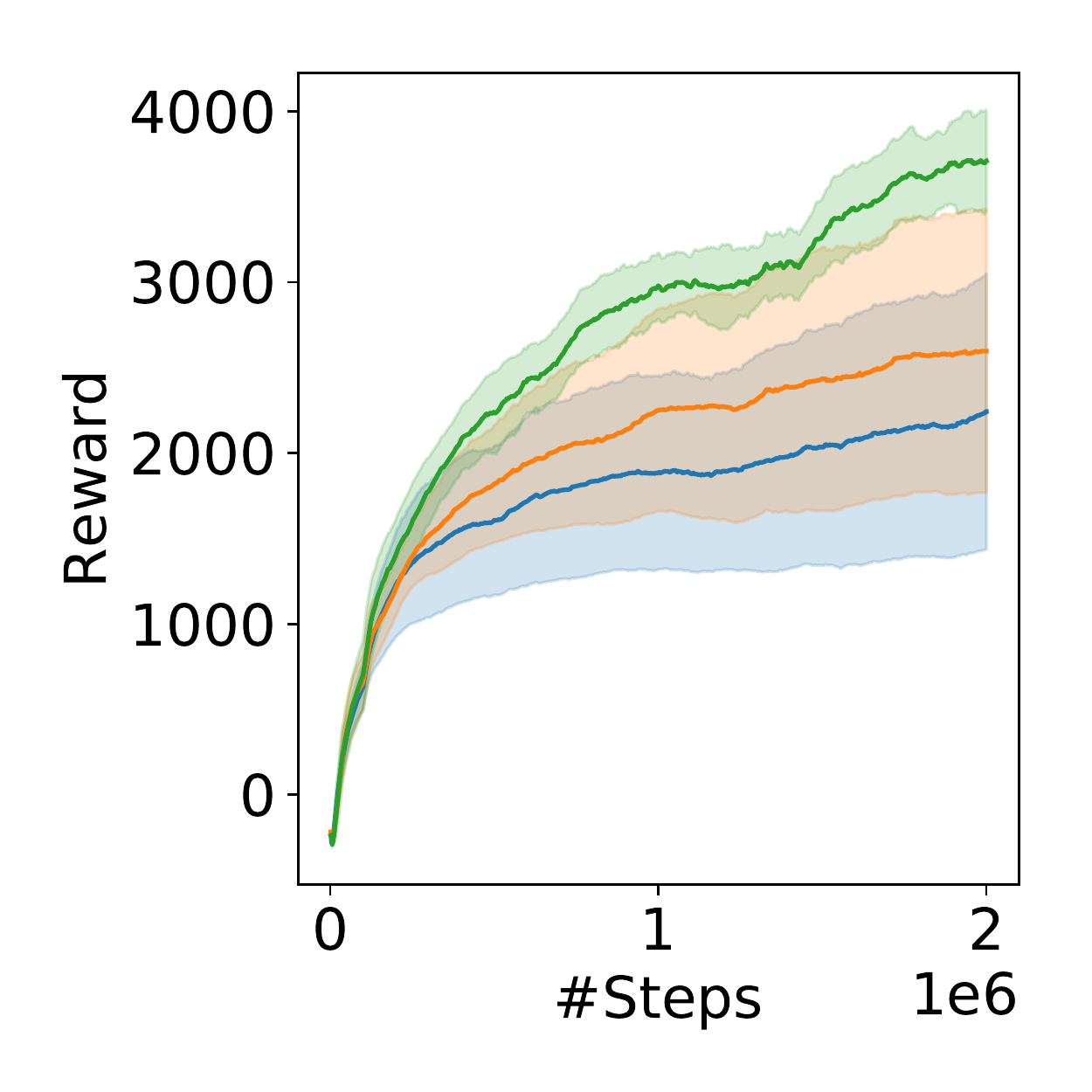}
\hfill
\includegraphics[width=0.4\textwidth]{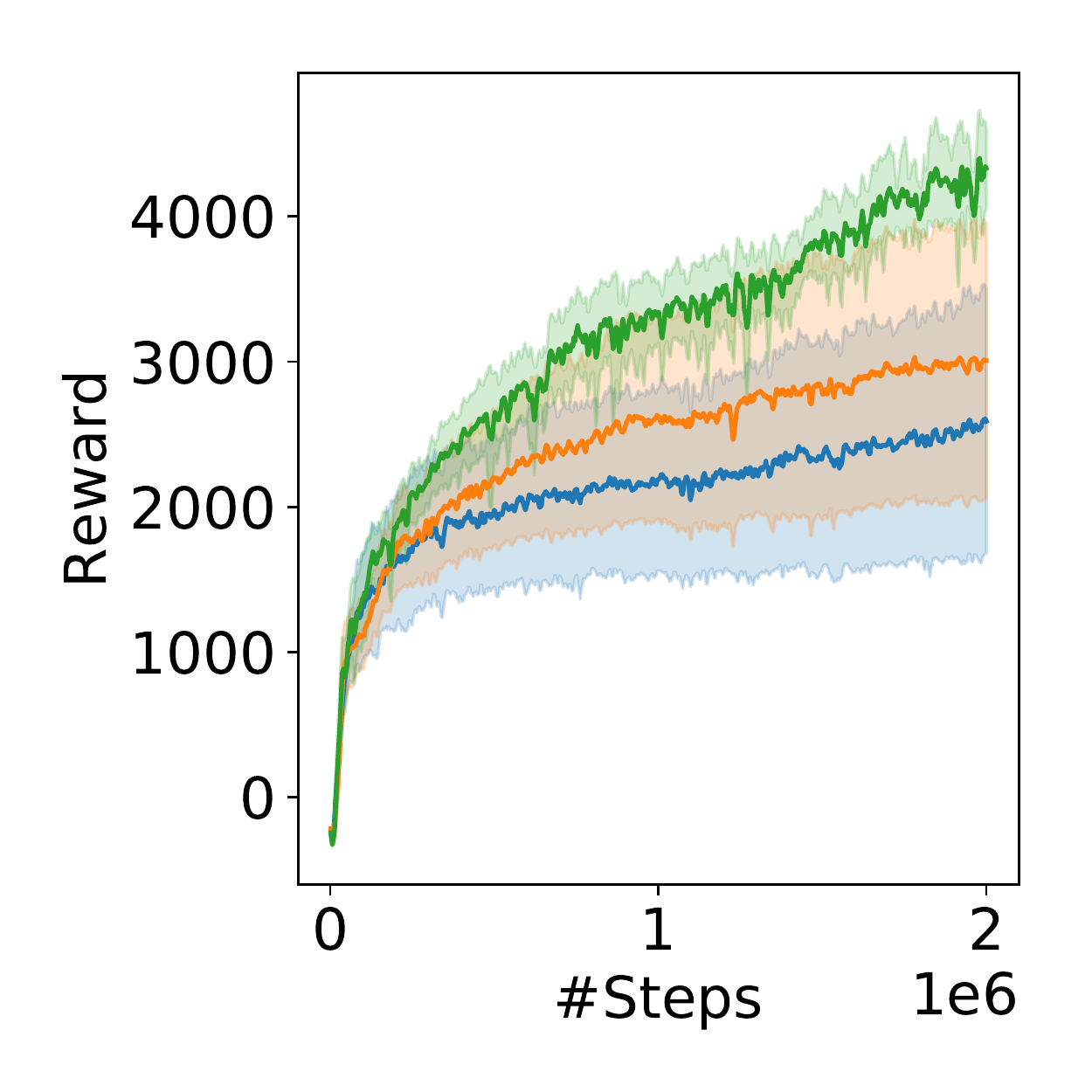} 

\end{subfigure} 

\begin{subfigure}{0.5\textwidth}   
 	\centering
 \includegraphics[width=0.4\textwidth]{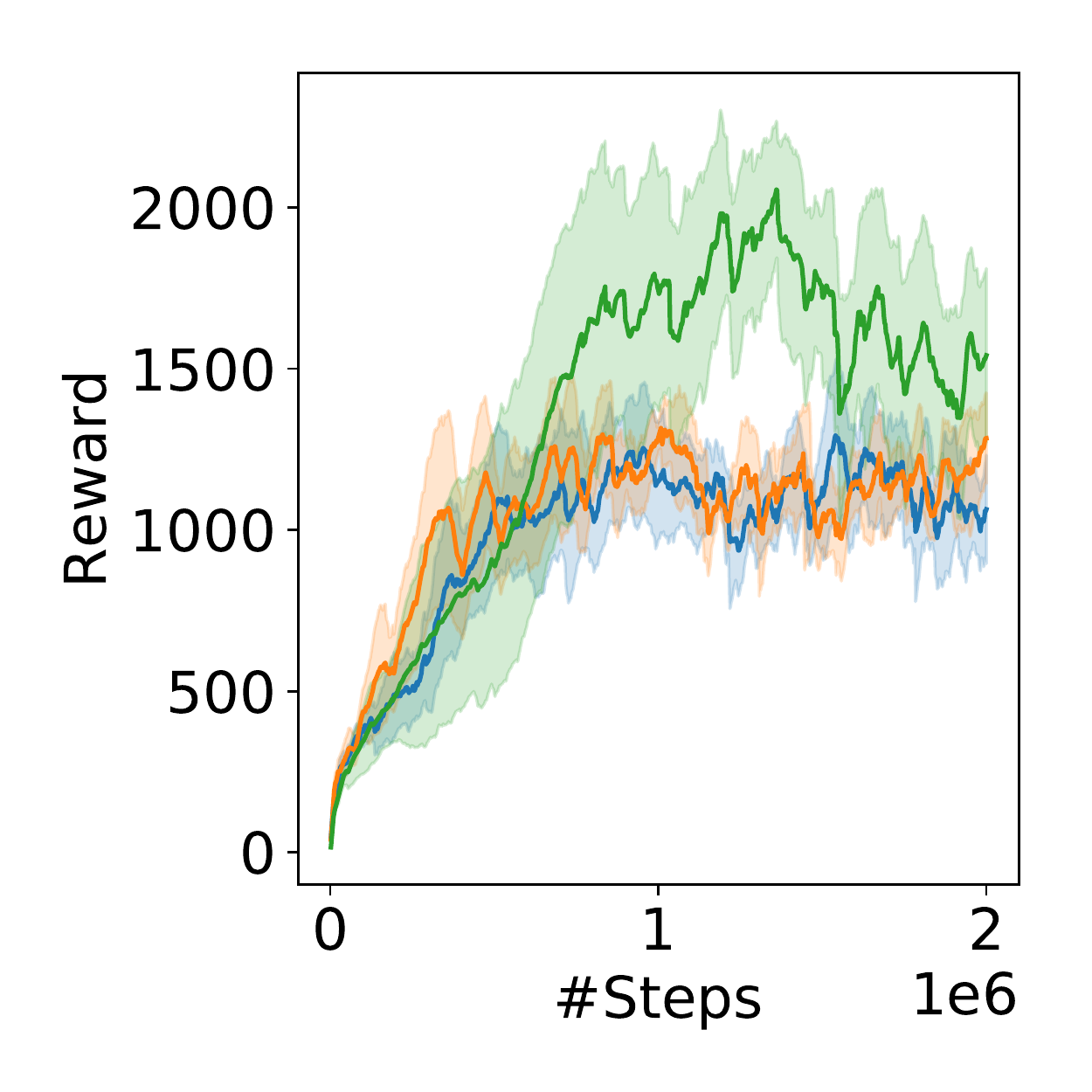}
 \hfill
 \includegraphics[width=0.4\textwidth]{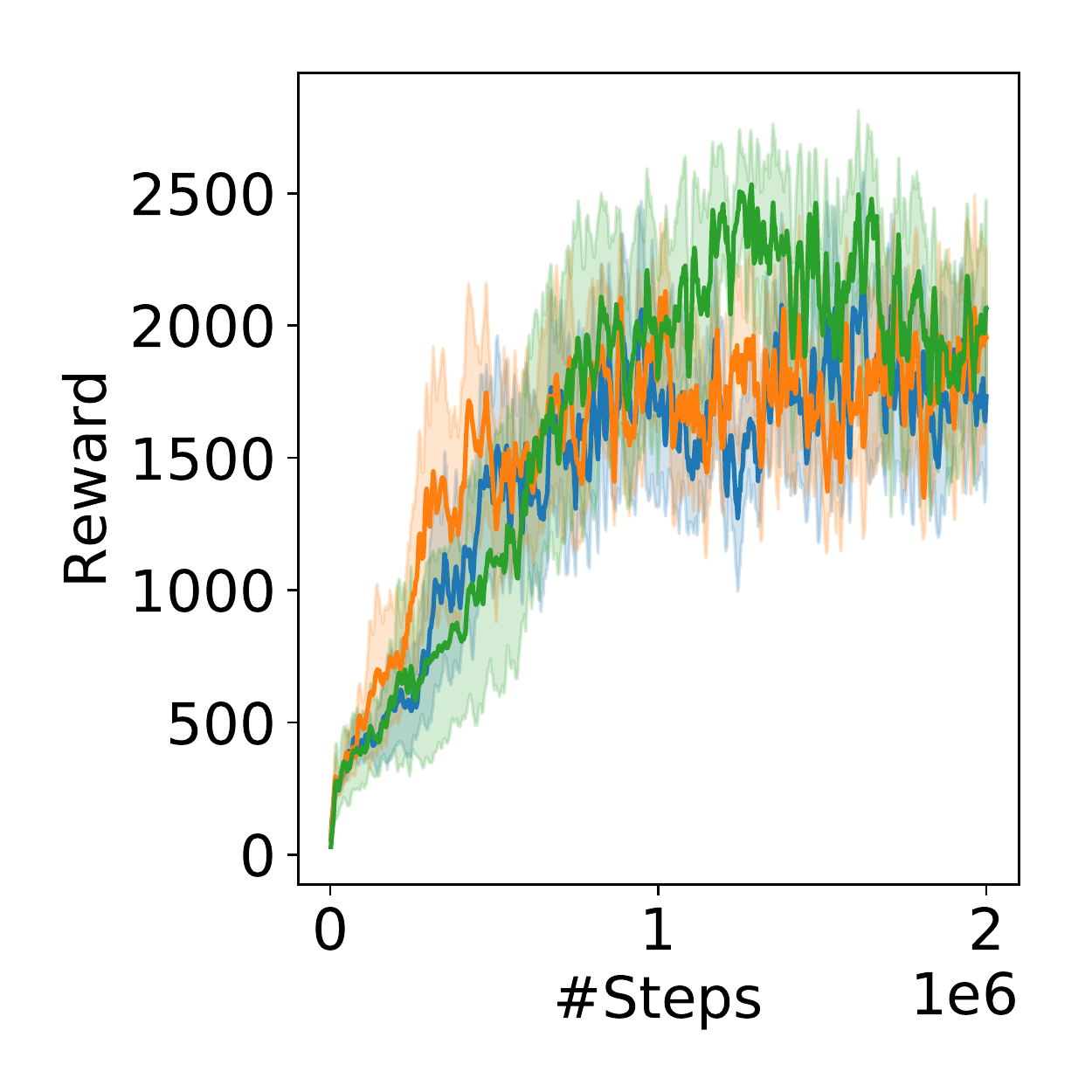} 
 
\end{subfigure} 

\begin{subfigure}{0.5\textwidth}
\centering
\includegraphics[width=0.4\textwidth]{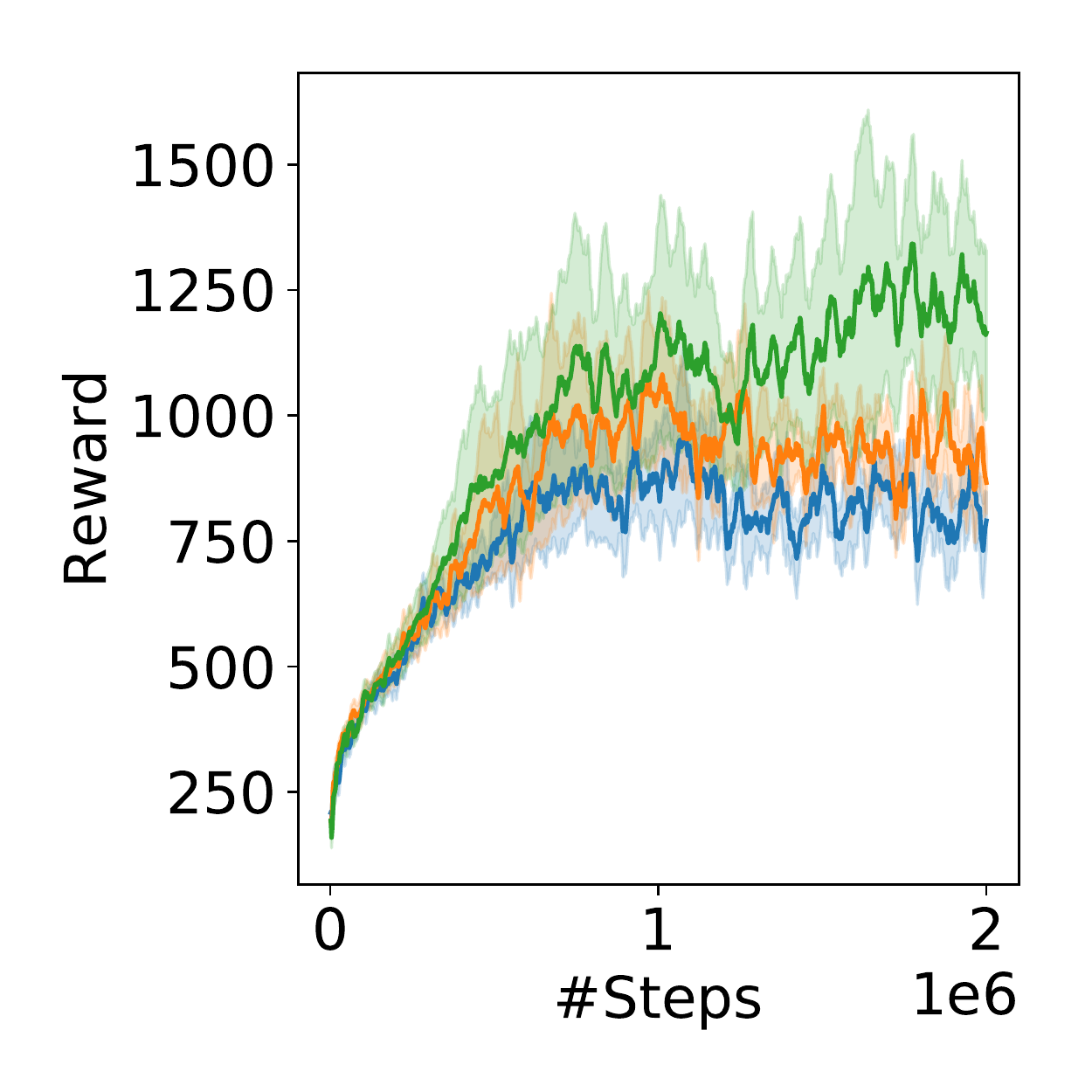}
\hfill
\includegraphics[width=0.4\textwidth]{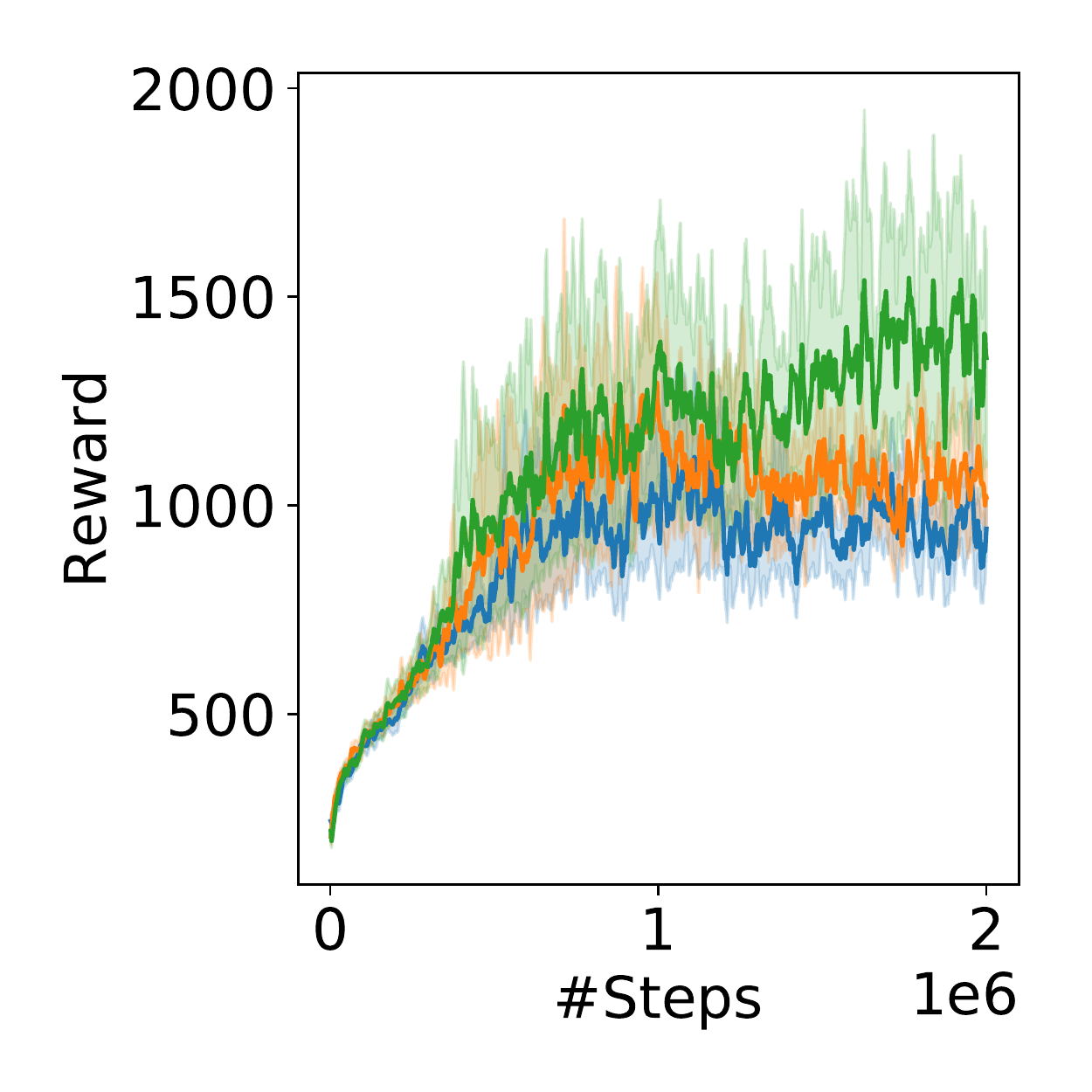} 

\end{subfigure}

\begin{subfigure}[b]{0.5\textwidth}
\centering
\includegraphics[width=0.4\textwidth]{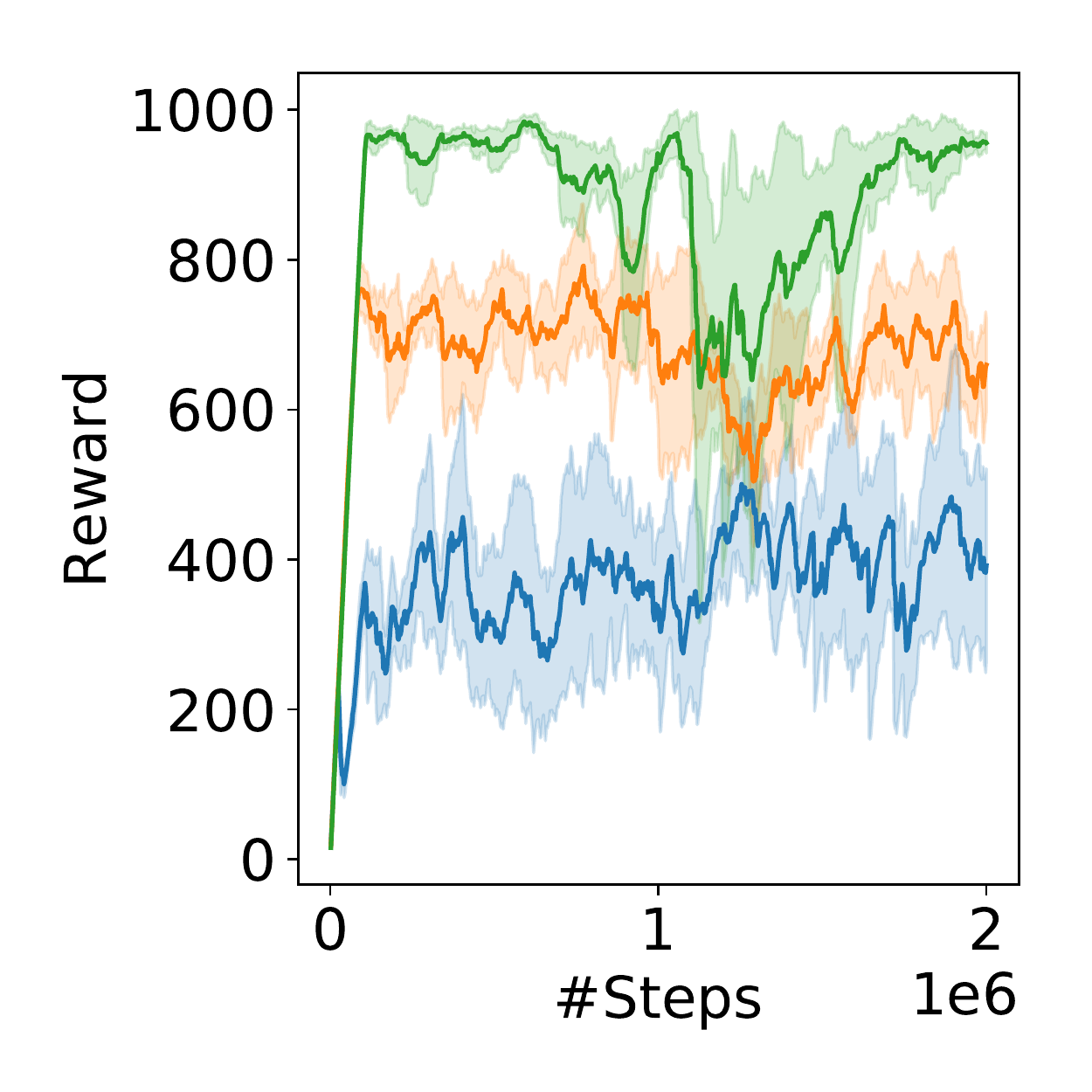}
\hfill
\includegraphics[width=0.4\textwidth]{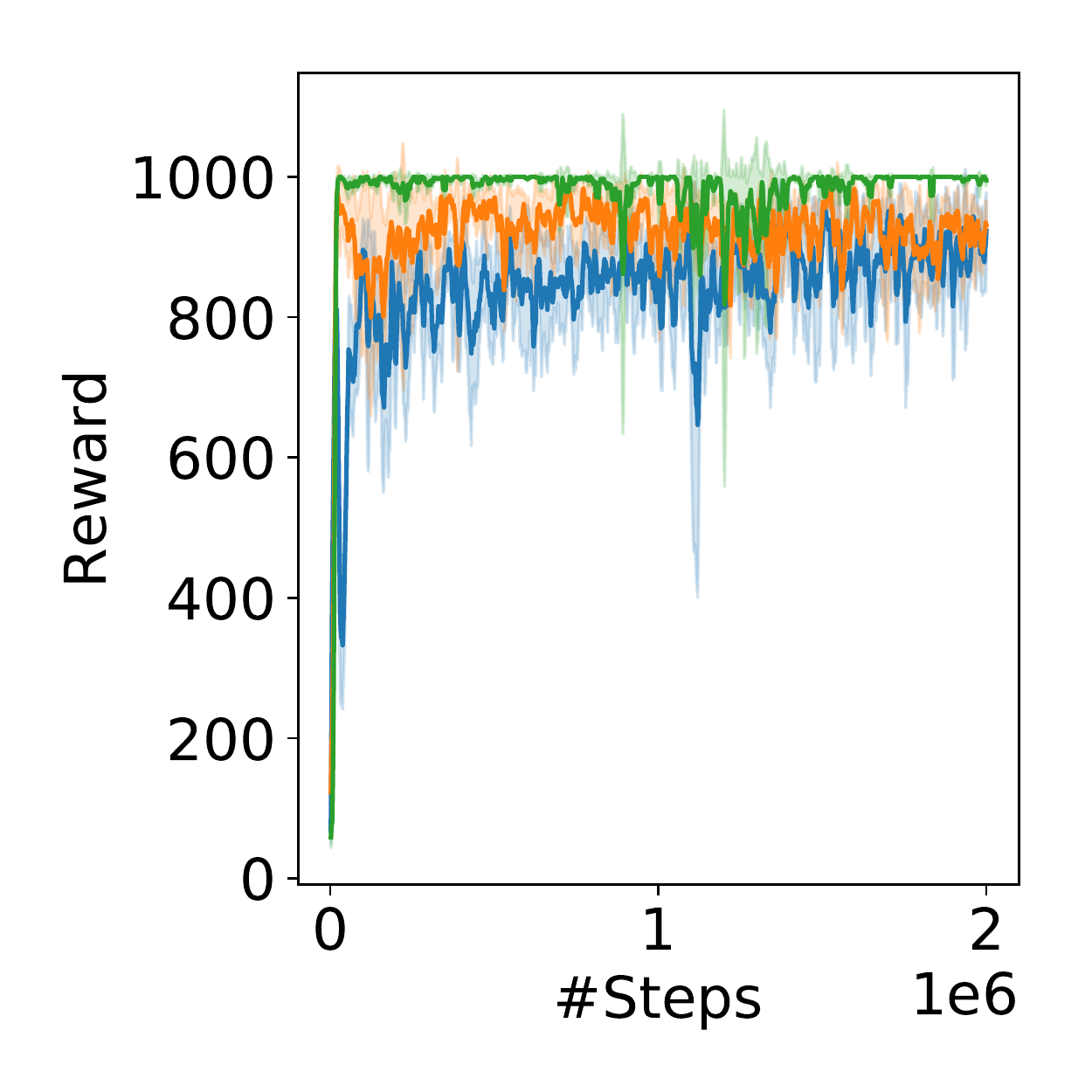} 

\end{subfigure}

\begin{subfigure}{0.5\textwidth}
\centering
\includegraphics[width=0.4\textwidth]{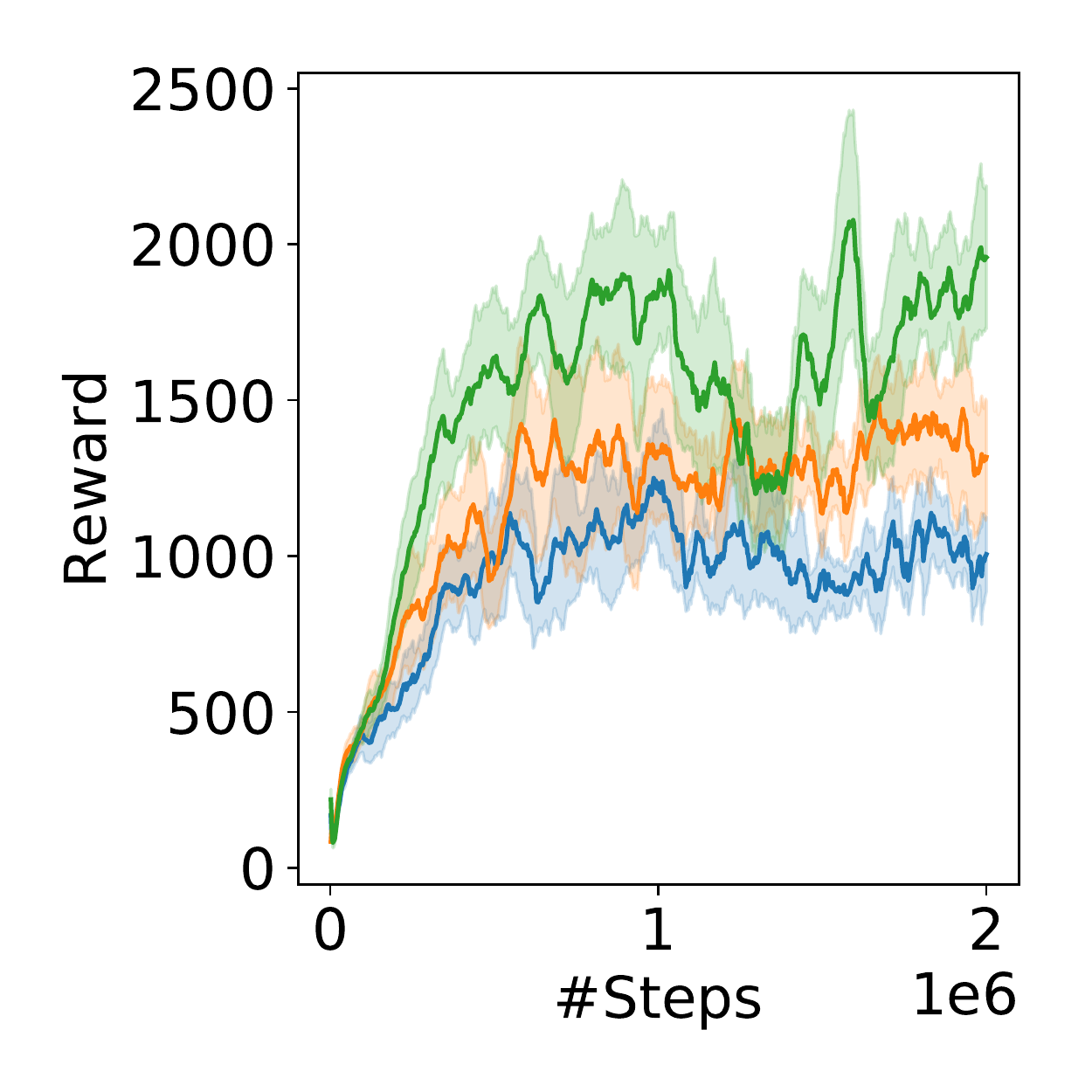}
\hfill
\includegraphics[width=0.4\textwidth]{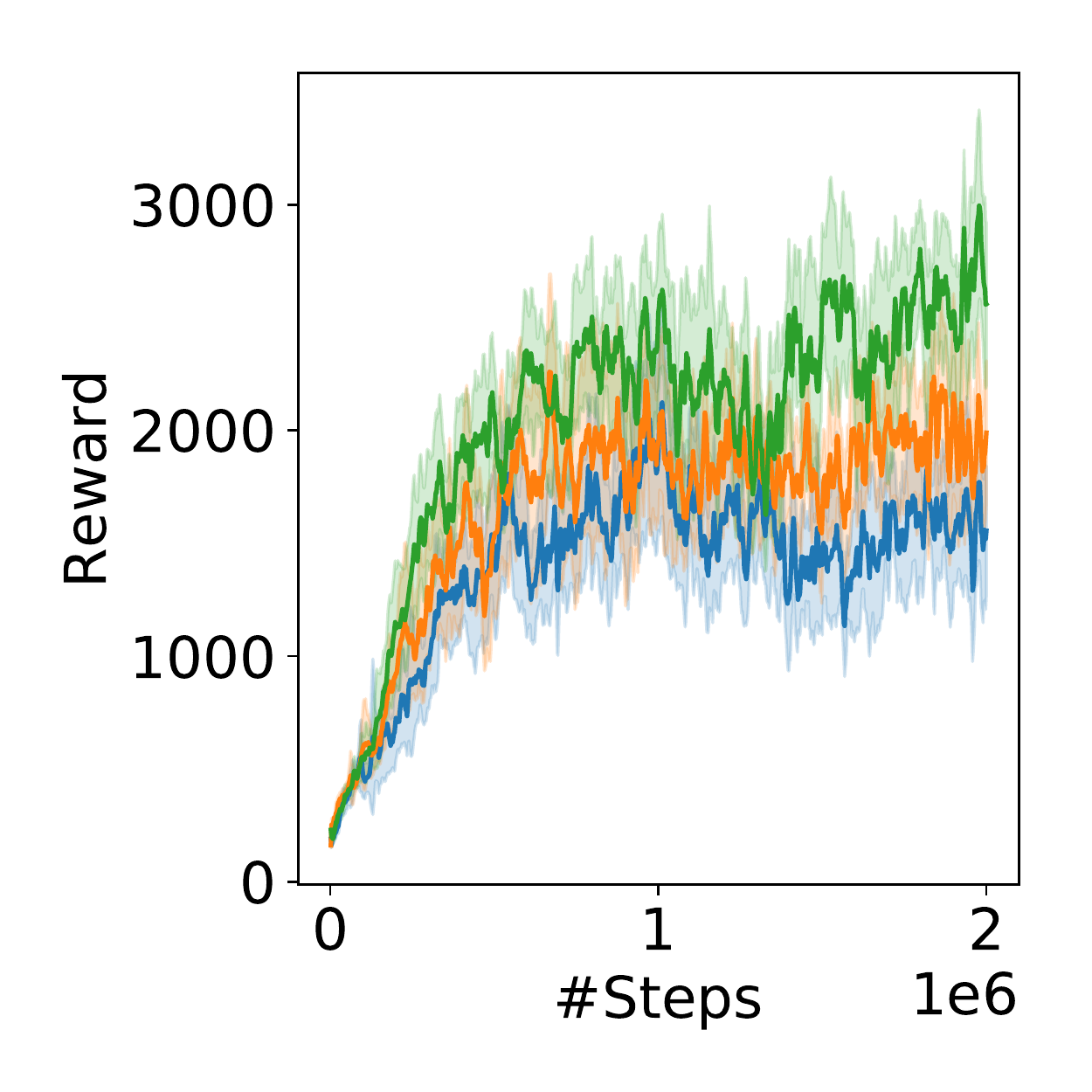} 

\end{subfigure}

\begin{subfigure}{0.6\textwidth}
\centering
\includegraphics[width=0.9\textwidth]{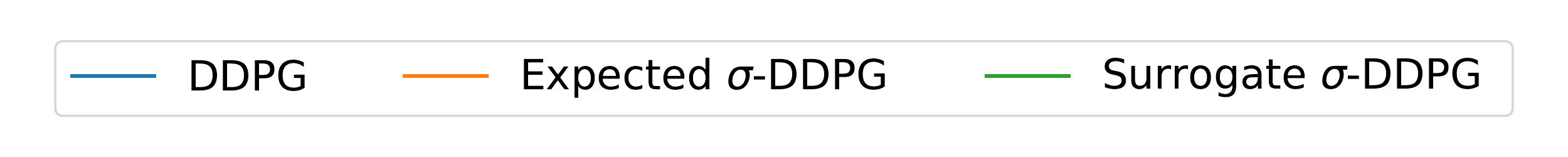}
\end{subfigure}

\caption{Simulation results for the MuJoCo environment: From up to bottom: Ant, HalfCheetah, Hopper, Humanoid, InvertedPendulum and Walker2d. (Left) Training. (Right) Test.}\label{fig: MujocoGraphs}
\end{figure}

\newpage

\section{Deep RL Exploration Conscious algorithms}\label{supp: algorithms psuedocode}
The algorithms in this section are the adjusted DDQN \cite{van2016deep} and DDPG \cite{lillicrap2015continuous}  to solve the $\alpha$-optimal and $\sigma$-optimal policies, respectively. For the surrogate approach the change is merely the gathered data; the action $\aChosen$ is saved and not $\aEnv$. For the expected approach, the expectation is calculated by an explicit averaging Algorithm \ref{alg:Expected alpha DDQN} or by simple sampling technique Algorithm \ref{alg:Expected sigma DDPG}.
\vspace{-1cm}
\begin{figure}[H]
\begin{minipage}{0.48\linewidth}
\begin{algorithm}[H]
\caption{Expected $\alpha$-DDQN}\label{alg:Expected alpha DDQN}
\begin{algorithmic}[0]
\INITIALIZE{ Network parameters $\theta$, $\theta^- \gets \theta$\\ 
			Replay buffer $R$, Target network update time $N^-$}
\FOR{episode$=1,M$}
	\FOR{ $t=1,T$ do}
		\STATE $\aChosen \gets \arg\max_a q(s_t,a|\theta)$
		\STATE $X_t \sim Bernoulli (1-\alpha)$
		\STATE $\aEnv = \begin{cases} \aChosen,\ \mathrm{if}\ X_t=1 \\
					a\sim \pi_0(\cdot\mid s),\ \mathrm{if}\ X_t=0 
				\end{cases}$
		\STATE $r_t, s_{t+1} \gets ACT(\aEnv)$
		\STATE Store $(s_t,\aEnv,r_t,s_{t+1})$ in $R$
		\STATE Sample $N$ tuples $(s_i,\aEnv^i,r_i,s'_i)$ from $R$
		\STATE $a_i \gets \arg\max_a q(s'_i,a|\theta)$
		\STATE $v_i \gets (1-\alpha) q(s'_i,a_i|\theta^-) + \alpha v^{\pi_0}(s'_i|\theta^-)$
		\STATE $y_i \gets r_i + \gamma v_i$
		\STATE Minimize $L=\frac{1}{N}\sum_i{\left(y_i-q(s_i,\aEnv^i|\theta\right)^2}$
		\STATE Update $\theta^-\gets\theta$ every $N^-$ steps
	\ENDFOR
\ENDFOR
\RETURN $\pi\in \arg\max_a q(\cdot,a)$
\end{algorithmic}
\end{algorithm}
\end{minipage} \hfill
\begin{minipage}{0.48\linewidth}
\begin{algorithm}[H]
\caption{Surrogate $\alpha$-DDQN}\label{alg:Surrogate alpha DDQN}
\begin{algorithmic}[0]
\INITIALIZE{ Network parameters $\theta$, $\theta^- \gets \theta$\\ 
			Replay buffer $R$, Target network update time $N^-$}
\FOR{episode$=1,M$}
	\FOR{ $t=1,T$ do}
		\STATE $\aChosen \gets \arg\max_a q_\alpha(s_t,a|\theta)$
		\STATE $X_t \sim Bernoulli (1-\alpha)$
		\STATE $\aEnv = \begin{cases} \aChosen,\ \mathrm{if}\ X_t=1 \\
					a\sim \pi_0(\cdot\mid s),\ \mathrm{if}\ X_t=0 
				\end{cases}$
		\STATE $r_t, s_{t+1} \gets ACT(\aEnv)$
		\STATE Store $(s_t,\aChosen,r_t,s_{t+1})$ in $R$
		\STATE Sample $N$ tuples $(s_i,\aChosen^i,r_i,s'_i)$ from $R$
		\STATE $a_i \gets \arg\max_a q_\alpha(s'_i,a|\theta)$
		\STATE $y_i \gets r_i + \gamma q_\alpha(s'_i,a_i|\theta^-)$
		\STATE Minimize $L=\frac{1}{N}\sum_i{\left(y_i-q_\alpha(s_i,\aChosen^i|\theta\right)^2}$
		\STATE Update $\theta^-\gets\theta$ every $N^-$ steps
	\ENDFOR
\ENDFOR
\RETURN $\pi\in \arg\max_a q_\alpha(\cdot,a)$
\vspace{0.39cm}
\end{algorithmic}
\end{algorithm}
\end{minipage} \\ 
\begin{minipage}{0.48\linewidth}
\begin{algorithm}[H]
\caption{Expected $\sigma$-DDPG}\label{alg:Expected sigma DDPG}
\begin{algorithmic}[0]
\INITIALIZE{Critic and Actor networks $ q(s,a|\theta) $, $\mu (s|\phi) $} \\
			Target networks weights: $\theta^- \gets \theta$ and $\phi^{-} \gets \phi$ \\
			Replay buffer $R$, Target network update time $N^-$
\FOR{episode$=1,M$}
	\STATE Initialize random markovian exploration process $\N$ \\ Receive initial observation state $s_1$
	\FOR{ $t=1,T$ do}
		\STATE $\aEnv \gets \mu(s_t|\phi)+\N_t$
		\STATE $r_t, s_{t+1} \gets ACT(a_t)$
		\STATE Store $(s_t,\aEnv,r_t,s_{t+1},\N_t)$ in $R$
		\STATE Sample $N$ transitions $(s_i,a_i,r_i,s'_i,\N_i)$ from $R$
		\STATE Sample $D_1$ noise terms $n_j$ given $\N_i$
		\STATE $y_i \gets r_i + \gamma \frac{1}{D_1}\sum_j{q(s'_i,\mu(s'_i)+n_j|\phi^-)|	\theta^-)}$
		\STATE Critic Loss: $L=\frac{1}{N}\sum_i{\left(y_i-q(s_i,a_i|\theta\right)^2}$
		\STATE Sample $D_2$ noise terms $n_j$ given $\N_i$
		\STATE Approximate gradient policy gradient: \\  $\phantom{a}$
		 $\nabla_\pi q^\pi(s_i) \approx \frac{1}{D_2} \sum_j{ \nabla_a q(s_i,a|\theta)\big|_{a=\mu(s_i)+n_j}}$
		\STATE Update actor using policy gradient: \\ $\phantom{a}$
		$\nabla_{\phi}V \approx \frac{1}{N} \sum_i {\nabla_\pi q^\pi(s_i) \nabla_{\phi}\mu(s_i|\phi)}$
		\STATE Update target networks every $N^-$ steps

	\ENDFOR
\ENDFOR
\RETURN $\mu(\cdot|\phi)$
\end{algorithmic}
\end{algorithm}
\end{minipage}
\hfill
\begin{minipage}{0.48\linewidth}
\begin{algorithm}[H]
\caption{Surrogate $\sigma$-DDPG}\label{alg:Surrogate sigma DDPG}
\begin{algorithmic}[0]
\INITIALIZE{Critic and Actor networks $ q_\sigma(s,a|\theta) $, $\mu (s|\phi) $} \\
			Target networks weights: $\theta^- \gets \theta$ and $\phi^- \gets \phi$ \\
			Replay buffer $R$, Target network update time $N^-$
\FOR{episode$=1,M$}
	\STATE Initialize random markovian exploration process $\N$ \\ Receive initial observation state $s_1$
	\FOR{ $t=1,T$ do}
		\STATE $\aChosen \gets \mu(s_t|\phi)$
		\STATE $\aEnv \gets\aChosen +\N_t$
		\STATE $r_t, s_{t+1} \gets ACT(\aEnv)$
		\STATE Store $(s_t,\aChosen,r_t,s_{t+1})$ in $R$
		\STATE Sample $N$ transitions $(s_i,a_i,r_i,s'_i)$ from $R$
		\STATE $y_i \gets r_i + \gamma {q_\sigma(s'_i,\mu(s'_i|\phi^{-}))|\theta^-)}$
		\STATE Critic Loss: $L=\frac{1}{N}\sum_i{\left(y_i-q_\sigma(s_i,a_i|\theta\right)^2}$
		\STATE Update actor using policy gradient: \\ 
			$\nabla_{\phi}V = \frac{1}{N} \sum_i {\nabla_a q_\sigma(s_i,a|\theta)\big|_{a=\mu(s_i)} \nabla_{\phi}\mu(s_i|\phi)}$
		\STATE Update target networks every $N^-$ steps
	\ENDFOR
\ENDFOR
\RETURN $\mu(\cdot|\phi)$
\vspace{1.31cm}

\end{algorithmic}
\end{algorithm}
\end{minipage}
\end{figure}

\section{Experimental details}\label{supp: experiments}

In this section we will discuss some technicalities that are related to the experiments done in this paper.

\subsection{Cliff Walking}
We used the T-Cliff-Walking scenario in Figure \ref{fig:CliffWalkingScenario}: The size of the cliff is $(h,w)=(4,12)$. We added small reward of $0.01r_{max}$ (green states) in order to create some small bias between the optimal and the $\alpha$-optimal policy. The maximal reward in this example is $r_{max}=1-\gamma$. We first checked to see that that $alpha=\epsilon=0.1$ performed bad. Then, we raised the $\epsilon$ value. The bottleneck passage between to sides of the maze, creates a scenario where high exploration is needed. We performed 2,000 runs for each of the algorithms. Finally, the test error was evaluated with high precision using the fixed value iteration procedure.

\begin{figure}[t]
\centering
\begin{subfigure}[t]{.3\textwidth}
	\includegraphics[width=\textwidth]{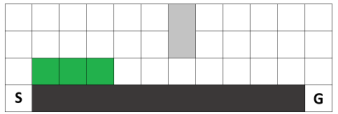} 
\end{subfigure}
	\caption{T-Cliff-Walking: The bright gray area is an impenetrable barrier.
The cliff is colored in dark gray. The green states are with a small reward of $0.01\cdot(1-\gamma)$.}
	\label{fig:CliffWalkingScenario}
\end{figure}

\begin{figure}[t]
 	\centering
 \begin{subfigure}[b]{0.25\textwidth}    
 	\centering
\includegraphics[width=\textwidth]{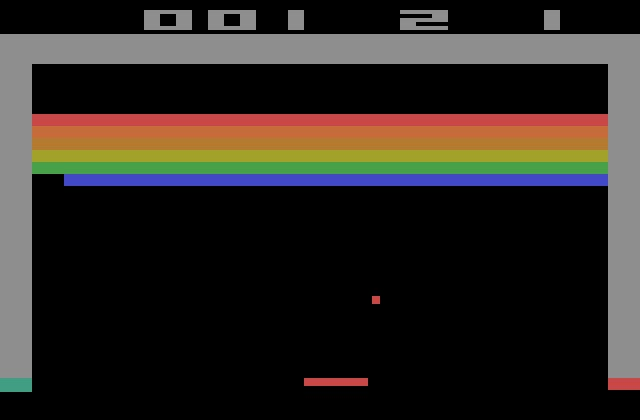} 
\end{subfigure} %
 \begin{subfigure}[b]{0.25\textwidth}    
 	\centering
\includegraphics[width=\textwidth]{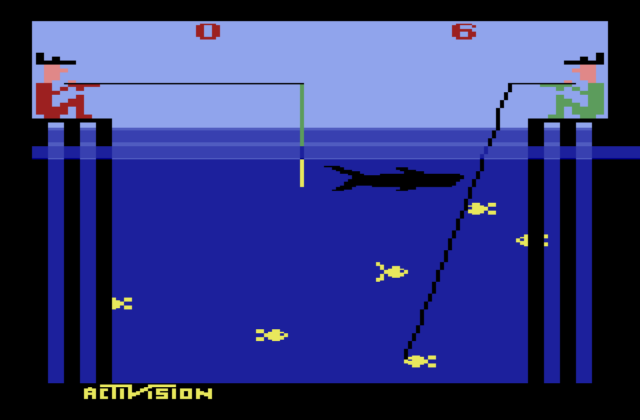} 
\end{subfigure} %
 \begin{subfigure}[b]{.25\textwidth}    
	\centering
\includegraphics[width=\textwidth]{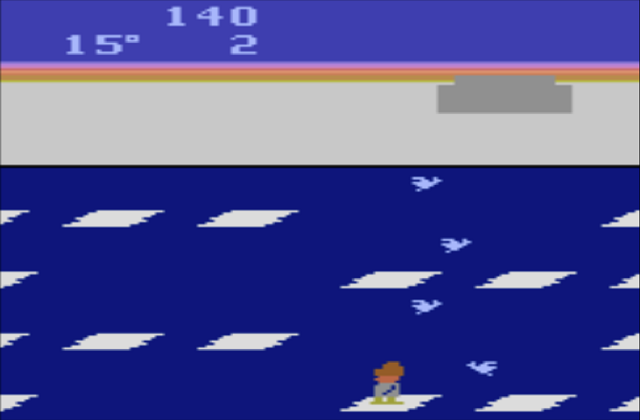} 
\end{subfigure}
\hfill
	\centering
 \begin{subfigure}[b]{.25\textwidth}    
	\centering
\includegraphics[width=\textwidth]{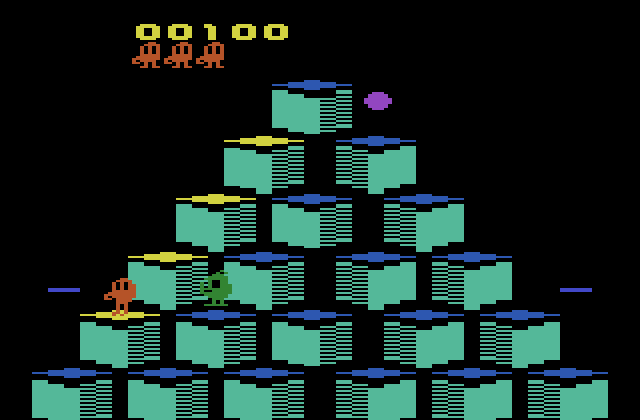} 
\end{subfigure} 
\begin{subfigure}[b]{.25\textwidth}    
	\centering
\includegraphics[width=\textwidth]{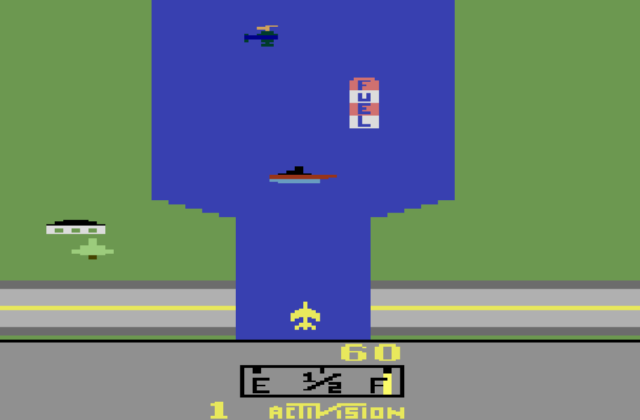} 
\end{subfigure} %
\caption{Atari games. From left to right: Breakout, Fishing Derby, Frostbite, Qbert, Riverraid.}
\label{fig:DeepGames}
\end{figure}

\section{Proof of Lemma \ref{lemma: equivalence}}\label{supp: lemma equivalence}

For any policy $\pi$ the following equalities hold.
\begin{align*}
v^\pi_\alpha &= (I-\gamma P^\pi_\alpha)^{-1}r^{\pi}_\alpha\\
&= (I-\gamma ((1-\alpha)P^\pi+\alpha P^{\pi_0}))^{-1}((1-\alpha) r^{\pi}+\alpha r^{\pi_0})\\
&=(I-\gamma P^{\pi^\alpha(\pi,\pi_0)})^{-1}r^{\pi^\alpha(\pi,\pi_0)}=v^{\pi^\alpha(\pi,\pi_0)}.
\end{align*}

\section{Proof of Proposition \ref{prop: alpha surrogate mdp bellman}}\label{supp: prop contraction}

\begin{proof}
Let $v\in \mathbb{R}^{|\mathcal{S}|}$ and consider the surrogate MDP, $\mathcal{M}_\alpha$. Its fixed policy Bellman operator (see \eqref{eq: T opt}) is given by:
\begin{align}
T^{\pi}_\alpha v  &= r^{\pi}_\alpha+\gamma P^{\pi}_\alpha v \nonumber\\
&= (1-\alpha)(r^\pi+\gamma P^\pi v) +\alpha (r^{\pi_0}+\gamma P^{\pi_0}v) \nonumber\\
&= (1-\alpha)T^\pi v +\alpha T^{\pi_0}v \label{lemma: fixed policy bellman operator}.
\end{align}
The second relation is by plugging $P^{\pi}_\alpha,r^{\pi}_\alpha$ from \eqref{eq: surrogate MDP reward and dynamics}, and rearranging. The fixed point of $T^\pi_\alpha$ is $v^\pi_\alpha$, the value of $\pi$ measured in $\mathcal{M}_\alpha$. Due to Lemma \ref{lemma: equivalence}, $v^\pi_\alpha = v^{\pi^\alpha(\pi,\pi_0)}$.

The optimal Bellman operator of $\mathcal{M}_\alpha$ is (see \eqref{eq: T opt}):
\begin{align}
T_\alpha v &= \max_\pi T^{\pi}_\alpha v \nonumber\\
& =\max_\pi (1-\alpha)T^\pi v +\alpha T^{\pi_0}v \nonumber\\
& =(1-\alpha)\max_\pi T^\pi v +\alpha T^{\pi_0}v = (1-\alpha)T + \alpha T^{\pi_0}, \label{lemma: opt policy bellman operator}
\end{align}
where the second relation holds by \eqref{lemma: fixed policy bellman operator}. The fixed point of $T_\alpha$ is, by construction, $v^*_\alpha$, the optimal value on $\mathcal{M}_\alpha$.
Moreover, $v^*_\alpha$ is the optimal value of a policy on $\mathcal{M}_\alpha$. By Lemma \ref{lemma: equivalence}, the policy that achieves the optimal value on $\mathcal{M}_\alpha$ achieves the $\alpha$-optimal value,  ${\max_{\pi'} v^{\pi^\alpha(\pi',\pi_0)} = v^{\pi^\alpha(\pi_\alpha^*,\pi_0)}} $. Thus, this policy is the $\alpha$-optimal policy, $\pi^*_\alpha$, and ${v^*_\alpha = v^{\pi_\alpha^*} = v^{\pi^\alpha(\pi_\alpha^*,\pi_0)}}$.

 Since $\mathcal{M}_\alpha$ is an MDP, its optimal policy is in the greedy set w.r.t. $v^*_\alpha$ (see \eqref{eq: G greeedy}). Thus,
\begin{align*}
\pi_\alpha^* & \in \{\pi: T_\alpha^\pi v_\alpha^* = T_\alpha v_\alpha^*\}\\
&=\{\pi: (1-\alpha)T^\pi v_\alpha^*+\alpha T^{\pi_0} v_\alpha^* = (1-\alpha)T v_\alpha^*+\alpha T^{\pi_0} v_\alpha^* \}\\
&=\{\pi: T^\pi v_\alpha^* = T v_\alpha^* \} = \G(v_\alpha^*).
\end{align*} 
\end{proof}

\section{Proof of Theorem \ref{proposition: alpha optimal improvement}} \label{supp: proof policy improvement}
For completness we give two useful lemmas that are in use. The first one has several instances in the literature.
\begin{lemma}\label{supp: help lemma difference}
Let $v^{\pi}$ and $v^{\pi'}$ be the correspondsing values of the policies $\pi$ and $\pi'$. Then,
\begin{align}
v^{\pi'}-v^{\pi} = (I-\gamma P^{\pi'})^{-1}(T^{\pi'}v^{\pi}-v^{\pi})
\end{align}
\end{lemma}
\begin{proof}
\begin{align*}
v^{\pi'}-v^{\pi} &= (I-\gamma P^{\pi'})^{-1}r^{\pi'} - v^{\pi}\\
&=(I-\gamma P^{\pi'})^{-1}(r^{\pi'} +\gamma P^{\pi'} v^{\pi} -v^{\pi}) \\
&= (I-\gamma P^{\pi'})^{-1}(T^{\pi'}v^{\pi}-v^{\pi}).
\end{align*}
\end{proof}

The following Lemma has several instrances in previous literature:
\begin{lemma} \label{supp: lemma strict improvement first relation}
Let $\pi$ be any policy and $\pi_{1-\mathrm{step}}\in\G(v^{\pi})$. Then,
\begin{align*}
v^{\pi}\leq v^{\pi^\alpha(\pi_{1-\mathrm{step}},\pi)},
\end{align*}
where the inequality is strict at least in one-component if $\pi\neq\pi^*$, if $\pi$ is not the optimal policy. 
\end{lemma}
\begin{proof}
\begin{align} 
v^{\pi^\alpha(\pi_{1-\mathrm{step}},\pi)}-v^{\pi} \nonumber &= (I-\gamma P^{\pi^\alpha(\pi_{1-\mathrm{step}},\pi)})^{-1}(T^{\pi^\alpha(\pi_{1-\mathrm{step}},\pi)}v^{\pi}-v^{\pi}), 
\end{align}\label{eq: supp help lemma 2}
where the first relation holds due to Lemma \ref{supp: help lemma difference}. See that,
\begin{align*}
T^{\pi^\alpha(\pi_{1-\mathrm{step}},\pi)}v^{\pi}-v^{\pi} &
=(1-\alpha) T^{\pi_{1-\mathrm{step}}}v^{\pi} + \alpha T{^\pi} v^{\pi} - v^{\pi}\\
& =(1-\alpha) T^{\pi_{1-\mathrm{step}}}v^{\pi} + \alpha v^{\pi} - v^{\pi}\\
& =(1-\alpha) \left( T^{\pi_{1-\mathrm{step}}}v^{\pi} - v^{\pi}\right) = (1-\alpha) \left( T v^{\pi} - v^{\pi}\right)
\end{align*}
Plugging it into \eqref{eq: supp help lemma 2} yields,
\begin{align*}
&v^{\pi^\alpha(\pi_{1-\mathrm{step}},\pi)}-v^{\pi}=(1-\alpha)(I-\gamma P^{\pi^\alpha(\pi_{1-\mathrm{step}},\pi)})^{-1}( T v^{\pi} - v^{\pi}).
\end{align*}

We have that $P^{\pi^\alpha(\pi_{1-\mathrm{step}},\pi)})^{-1}\geq 0$ since it is a $\gamma$-discounted weighted sum of stochastic matrices. Furthermore,
\begin{align*}
v^{\pi} = T^{\pi}v^{\pi} \leq T v^{\pi},
\end{align*}
where the last inequality is strict at least in one component if $v^{\pi}\neq v^*$, i.e, if $\pi\neq \pi^*$.
\end{proof}

We now prove the result. The first relation holds almost by construction. We have that,
\begin{align}\label{eq: improvement strict or equal}
v^{\pi^\alpha(\pi_\alpha^*,\pi_0)} = \max_{\pi'} v^{\pi^\alpha(\pi',\pi_0)}\geq v^{\pi^\alpha(\pi_0,\pi_0)} = v^{\pi_0}
\end{align}
where the first relation is due to the definition of the $\alpha$-optimal value \eqref{eq:eps_optimization}, the second relation holds by definition and the third relation holds since 
\begin{align*}
\pi_\alpha(\pi_0,\pi_0)=(1-\epsilon)\pi_0+\epsilon \pi_0 = \pi_0.
\end{align*}
As long as $\pi_0\neq \pi^*$, the policy $\pi_{1-\mathrm{step}}\in \G(v^{\pi_0})$ acheives strict improvement in \eqref{eq: improvement strict or equal}. Meaning,
\begin{align*}
v^{\pi^\alpha(\pi_{1-\mathrm{step}},\pi_0)} \geq v^{\pi_0}.
\end{align*}
This means that the improvement in \eqref{eq: improvement strict or equal} is strict as long as $\pi_0\neq \pi^*$.If $\pi_0$ is not optimal we have that
\begin{align*}
 v^{\pi_0} \leq v^{\pi^\alpha(\pi_{1-\mathrm{step}},\pi_0)} \leq v^{\pi^\alpha(\pi_\alpha^*,\pi_0)}.
\end{align*}
The first relation is strict due to Lemma \ref{supp: lemma strict improvement first relation}, and the second relation holds by the definition of the $\alpha$-optimal policy.

We now prove the second relation of the lemma. Let $\beta \in [0,\alpha]$. Then,

\begin{align}
v^{\pi^\beta(\pi_\alpha^*,\pi_0)}- v^*_\alpha \nonumber=
 (I-\gamma P^{\pi^\beta(\pi_\alpha^*,\pi_0)})^{-1}(T^{\pi^\beta(\pi_\alpha^*,\pi_0)}v^*_\alpha-v^*_\alpha).
\end{align} \label{eq: supp improvement second rel}
We have that,
\begin{align*}
T^{\pi^\beta(\pi_\alpha^*,\pi_0)}v^*_\alpha-v^*_\alpha&=T^{\pi^\beta(\pi_\alpha^*,\pi_0)}v^*_\alpha-T_\alpha v^*_\alpha\\
&=(1-\beta)T^{\pi_\alpha^*}v^*_\alpha +\beta T^{\pi_0}v^*_\alpha - (1-\alpha)T v^*_\alpha - \alpha T^{\pi_0} v^*_\alpha\\
&=(\alpha-\beta) \left(Tv^*_\alpha - T^{\pi_0}v^*_\alpha \right),
\end{align*}
where in the last relation we used $T^{\pi_\alpha^*}v^*_\alpha = T v^*_\alpha$ (see Proposition \ref{prop: alpha surrogate mdp bellman}). Plugging into \eqref{eq: supp improvement second rel} yields,
\begin{align*}
v^{\pi^\beta(\pi_\alpha^*,\pi_0)}- v^*_\alpha
&=(\alpha-\beta) (I-\gamma P^{\pi^\beta(\pi_\alpha^*,\pi_0)})^{-1}\left(Tv^*_\alpha - T^{\pi_0}v^*_\alpha \right).
\end{align*}
We have that $(I-\gamma P^{\pi^\beta(\pi_\alpha^*,\pi_0)})^{-1}\geq 0$ since it is a $\gamma$-discounted sum of stochastic matrices, and $T v^*_\alpha \geq T^{\pi_0}v^*_\alpha$ with equality if and only if $\pi_0$ is optimal; if and only if $\pi_0$ is optimal $v^*_\alpha=v^*$ due to the first part of this proof.

\subsection{Counter example for monotonous improvement for the $\alpha$-optimal criterion}
\begin{figure}[t]
	\centering
	\resizebox{2.8in}{!}{
		\begin{tikzpicture}[->,>=stealth',shorten >=1pt,auto,node distance=2.8cm,
		semithick, state/.style={circle, draw, minimum size=1.1cm}]
		\tikzstyle{every state}=[thick]
		]
		
		\node[state] (S0) {\large $s_0$};
		\node[state] (S1) [below right of=S0] {\large $s_1$};
		\node[state] (S2) [below left of=S0] {\large $s_2$};
	
		\path
		(S0) edge  [bend left]   node[pos=0.1,above ]{ }         node [above] {\large $a_1,0$} (S1)
			edge  [bend right]   node[pos=0.1,below ]{ }         node [above] {\large $a_2,0$} (S2)
		(S1) edge  [loop right] node[pos=1,right]{} node {\large $a_1,0.8$} (S1)
	    (S2) edge  [loop left]   node[pos=0.1,above]{}        node {\large $a_2,0$} (S2)
	    (S2) edge  [loop right]   node[pos=0.1,above]{}        node {\large $a_1,1$} (S2);
		
		\end{tikzpicture}
	}
	\caption{Counter exmple for an MDP with no monotonous improvement for the $\alpha$-optimal criterion \ref{proposition: alpha optimal improvement}.}
	\label{fig: No Monotonous Improvement}
\end{figure}
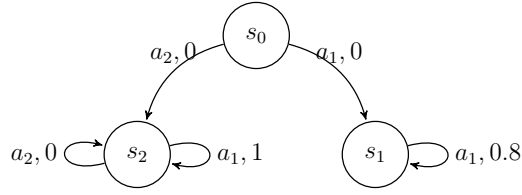

In this section, we give a counter example that proves that the improvement in Proposition~\ref{proposition: alpha optimal improvement} is not monotonous w.r.t $\beta$.
Let the MDP given in Figure~\ref{fig: No Monotonous Improvement} be a $\gamma$-discounted MDP for some $\gamma\in(0,1)$. Let $\pi_0$ be a deterministic policy which always chooses action $a_2$. For $\alpha=0.25$, It is easy to verify that $v^*_\alpha(s_1)=0.8$ and $v^*_\alpha(s_2)=(1-\alpha)=0.75$. Now, $q^*_\alpha(s_0,a_1)=\gamma (1-\alpha)v^*_\alpha(s_1)+\alpha v^*_\alpha(s_2)= 0.7875\gamma$, and 
$q^*_\alpha(s_0,a_2)=0.75\gamma$. Thus, the $\alpha$-optimal policy on $s_0$ is to choose $a_1$, and $v^*_\alpha(s_0)=0.7875\gamma$.

Now, we consider acting according to the mixture policy $\pi^\beta(\pi^*_\alpha,\pi_0)$ for some $\beta<\alpha$. For the greedy policy, i.e. $\beta=0$, we get that $v^{\pi^0(\pi^*_\alpha,\pi_0)}=v^{\pi^*_\alpha}=0.8\gamma$. For $\beta=0.1$, we get that $v^{\pi^{0.1}(\pi^*_\alpha,\pi_0)}=\gamma(0.9\cdot0.8 + 0.1\cdot(0.9\cdot1) )=0.81\gamma$. 
To conclude, as the lemma~\ref{proposition: alpha optimal improvement} suggests, we get improvement for both inspected $\beta$, i.e. $v^*_\alpha < v^{\pi^*_\alpha}$ and $v^*_\alpha < v^{\pi^{0.1}(\pi^*_\alpha,\pi_0)}$. However, the improvement does not increase monotonically as we decrease $\beta$, as $v^{\pi^*_\alpha} =0.8\gamma < 0.81\gamma = v^{\pi^{0.1}(\pi^*_\alpha,\pi_0)}$.


\section{Generalization of \cite{bertsekas1995neuro}[Proposition 6.1] for any policy class}\label{supp: generalized sensitivity bound}

In this section, we prove a generalization of \cite{bertsekas1995neuro}[Proposition 6.1] for any class of policies.
\begin{proposition}\label{prop: generalized sensitivity bound}
Let $\sigma$ a set of fixed parameters of some distribution class. Assume $\hat{v}_\sigma^*$ is an approximate $\sigma$-optimal value s.t. $\norm{v_\sigma^*-\hat{v}_\sigma^*}=\delta$ for some $\delta>0$. Then,
\begin{align*}
 \left\| v_{\sigma }^{*}-{{v}^{\hat{\pi }_{\sigma }^{*}}} \right\|\le \frac{\gamma \delta {{\left\| \pi _{\sigma }^{*}-\hat{\pi }_{\sigma }^{*} \right\|}_{TV}}}{1-\gamma }.
 \end{align*}
\begin{proof}
\begin{align*}
  & v_{\sigma }^{*}-v_{\sigma }^{\hat{\pi }_{\sigma }^{*}}={{T}_{\sigma }}v_{\sigma }^{*}-{{T}^{\hat{\pi }_{\sigma }^{*}}}{{v}^{\hat{\pi }_{\sigma }^{*}}}={{T}_{\sigma }}v_{\sigma }^{*}-{{T}_{\sigma }}\hat{v}_{\sigma }^{*}+{{T}_{\sigma }}\hat{v}_{\sigma }^{*}-{{T}^{\hat{\pi }_{\sigma }^{*}}}{{v}^{\hat{\pi }_{\sigma }^{*}}} \\ 
 & ={{T}_{\sigma }}v_{\sigma }^{*}-{{T}_{\sigma }}\hat{v}_{\sigma }^{*}+{{T}^{\hat{\pi }_{\sigma }^{*}}}v_{\sigma }^{*}-{{T}^{\hat{\pi }_{\sigma }^{*}}}v_{\sigma }^{*}+{{T}^{\hat{\pi }_{\sigma }^{*}}}\hat{v}_{\sigma }^{*}-{{T}^{\hat{\pi }_{\sigma }^{*}}}\hat{v}_{\sigma }^{*}+{{T}_{\sigma }}\hat{v}_{\sigma }^{*}-{{T}^{\hat{\pi }_{\sigma }^{*}}}{{v}^{\hat{\pi }_{\sigma }^{*}}} \\ 
 & =\left( {{T}_{\sigma }}v_{\sigma }^{*}-{{T}_{\sigma }}\hat{v}_{\sigma }^{*} \right)+\left( {{T}^{\hat{\pi }_{\sigma }^{*}}}\hat{v}_{\sigma }^{*}-{{T}^{\hat{\pi }_{\sigma }^{*}}}v_{\sigma }^{*} \right)+\left( {{T}^{\hat{\pi }_{\sigma }^{*}}}v_{\sigma }^{*}-{{T}^{\hat{\pi }_{\sigma }^{*}}}\hat{v}_{\sigma }^{*}+{{T}_{\sigma }}\hat{v}_{\sigma }^{*}-{{T}^{\hat{\pi }_{\sigma }^{*}}}{{v}^{\hat{\pi }_{\sigma }^{*}}} \right) \\ 
 & \overset{(a)}{\le} \left( {{T}^{\pi _{\sigma }^{*}}}v_{\sigma }^{*}-{{T}^{\pi _{\sigma }^{*}}}\hat{v}_{\sigma }^{*} \right)+\left( {{T}^{\hat{\pi }_{\sigma }^{*}}}\hat{v}_{\sigma }^{*}-{{T}^{\hat{\pi }_{\sigma }^{*}}}v_{\sigma }^{*} \right)+\left( {{T}^{\hat{\pi }_{\sigma }^{*}}}v_{\sigma }^{*}-{{T}^{\hat{\pi }_{\sigma }^{*}}}\hat{v}_{\sigma }^{*}+{{T}_{\sigma }}\hat{v}_{\sigma }^{*}-{{T}^{\hat{\pi }_{\sigma }^{*}}}{{v}^{\hat{\pi }_{\sigma }^{*}}} \right) \\ 
 & \overset{(b)}{=}\gamma {{P}^{\pi _{\sigma }^{*}}}\left( v_{\sigma }^{*}-\hat{v}_{\sigma }^{*} \right)-\gamma {{P}^{\hat{\pi }_{\sigma }^{*}}}\left( v_{\sigma }^{*}-\hat{v}_{\sigma }^{*} \right)+\left( {{T}^{\hat{\pi }_{\sigma }^{*}}}v_{\sigma }^{*}-{{T}_{\sigma }}\hat{v}_{\sigma }^{*}+{{T}_{\sigma }}\hat{v}_{\sigma }^{*}-{{T}^{\hat{\pi }_{\sigma }^{*}}}{{v}^{\hat{\pi }_{\sigma }^{*}}} \right) \\ 
 & =\gamma \left( {{P}^{\pi _{\sigma }^{*}}}-{{P}^{\hat{\pi }_{\sigma }^{*}}} \right)\left( v_{\sigma }^{*}-\hat{v}_{\sigma }^{*} \right)+\left( {{T}^{\hat{\pi }_{\sigma }^{*}}}v_{\sigma }^{*}-{{T}^{\hat{\pi }_{\sigma }^{*}}}{{v}^{\hat{\pi }_{\sigma }^{*}}} \right)
\end{align*}
Where (a) is due to the fact that for any $v$ and $\pi$, $T_\sigma^\pi \leq T_\sigma v$, and (b) is due to the definition of the $\sigma$-greedy operator.

Taking the max-norm,
\begin{align*}
  & \left| v_{\sigma }^{*}\left( s \right)-{{v}^{\hat{\pi }_{\sigma }^{*}}}\left( s \right) \right|\le \gamma \left| \left( \left( {{P}^{\pi _{\sigma }^{*}}}-{{P}^{\hat{\pi }_{\sigma }^{*}}} \right)\left( v_{\sigma }^{*}-\hat{v}_{\sigma }^{*} \right) \right)\left( s \right) \right|+\left| \left( {{T}^{\hat{\pi }_{\sigma }^{*}}}v_{\sigma }^{*}-{{T}^{\hat{\pi }_{\sigma }^{*}}}{{v}^{\hat{\pi }_{\sigma }^{*}}} \right)\left( s \right) \right| \\ 
 & \le \gamma \left| \sum\limits_{s',a}^{{}}{p\left( s|s',a \right)}\left( \pi _{\sigma }^{*}\left( a|s' \right)-\hat{\pi }_{\sigma }^{*}\left( a|s' \right) \right)\left( v_{\sigma }^{*}\left( s' \right)-\hat{v}_{\sigma }^{*}\left( s' \right) \right) \right|+\gamma \left\| v_{\sigma }^{*}-{{v}^{\hat{\pi }_{\sigma }^{*}}} \right\|= \\ 
 & \le \gamma {{\max }_{s'}}\left| \sum\limits_{a}^{{}}{\left( \pi _{\sigma }^{*}\left( a|s' \right)-\hat{\pi }_{\sigma }^{*}\left( a|s' \right) \right)}\left( v_{\sigma }^{*}\left( s' \right)-\hat{v}_{\sigma }^{*}\left( s' \right) \right) \right|+\gamma \left\| v_{\sigma }^{*}-{{v}^{\hat{\pi }_{\sigma }^{*}}} \right\| \\ 
 & \le \gamma \left\| v_{\sigma }^{*}-\hat{v}_{\sigma }^{*} \right\|{{\left\| \pi _{\sigma }^{*}-\hat{\pi }_{\sigma }^{*} \right\|}_{TV}}+\gamma \left\| v_{\sigma }^{*}-{{v}^{\hat{\pi }_{\sigma }^{*}}} \right\| \end{align*}
 Where the $\norm{\cdot}_{TV}$ accounts for the maximal total-variation distance over all states.
 Finally,
 \begin{align*}
 \left\| v_{\sigma }^{*}-{{v}^{\hat{\pi }_{\sigma }^{*}}} \right\|\le \frac{\gamma \delta {{\left\| \pi _{\sigma }^{*}-\hat{\pi }_{\sigma }^{*} \right\|}_{TV}}}{1-\gamma }.
 \end{align*}
 \end{proof}
\end{proposition}

Finally, this bound is a generalization of \cite{bertsekas1995neuro}[Proposition 6.1], for any class of distributions. Notice that the total variation distance is not bigger than $2$, which is the case of two different deterministic policies. This leads back to the familiar bound.

\section{Proof of Theorem~\ref{theorem:performance model free}: Bias-Error Sensitivity in the $\alpha$-greedy case}\label{supp: theorem gaussian tradeoff}\label{supp: theorem performance model free}

In order to prove the theorem, we first prove the following two propositions \ref{prop:biasBound},\ref{prop:epsOptBound}. Then, we plug the results in the following triangle inequality:
\begin{align*}
\norm{v^* - v^{\pi^\alpha(\hat{\pi}_\alpha^*,\pi_0)}} \leq \norm{v^* - v^*_\alpha} + \norm{v^*_\alpha - v^{\pi^\alpha(\hat{\pi}_\alpha^*,\pi_0)}}
\end{align*}

\begin{proposition}\label{prop:biasBound}
Let ${\forall s\in \mathcal{S}, \ \alpha(s)\in[0,1]}$, be a state-dependent function. Let $\pi^*_\alpha$ be the $\alpha$-optimal policy, and $L(s)$ the MDP Lipschitz constant, both relatively to $\pi_0$. Define ${B(\alpha) \triangleq \max_s \alpha(s) L(s)}$. The following bounds hold,
\begin{align*}
  \norm{v^*-v^{\pi_\alpha^*}} \le\norm{v^*-v^{\pi^{\alpha}(\pi_g,\pi_0)}} \le \frac{B(\alpha)}{1-\gamma},
\end{align*}
If $\forall s\in \mathcal{S},\ \alpha(s)=\alpha\in [0,1]$ then $B(\alpha)=\alpha L$ (see Definition \ref{defn: lipschitz constant}). Furthermore, this bound is tight.

\begin{proof}

We have that for any $s\in \mathcal{S}$,
\begin{align}
v^* - v^*_\alpha (s)
&=  (Tv^* - T_\alpha v^*)(s) + (T_\alpha v^* - T_\alpha v^*_\alpha ) (s) \nonumber\\
&\leq \norm{Tv^* - T_\alpha v^*} + \norm{T_\alpha v^* - T_\alpha v^*_\alpha } \nonumber\\
&\leq \norm{Tv^* - T_\alpha v^*} + \gamma \norm{ v^* - v^*_\alpha } \nonumber,
\end{align}\label{eq: biasProp LHS}
in the last relation we used the fact that $T_\alpha$ is a $\gamma$ contraction in the max-norm. Moreover, we have that for any $s\in \mathcal{S}$,
\begin{align}
Tv^*(s) - T_\alpha v^*(s) \nonumber
&=Tv^* - (1-\alpha(s))T v^* (s)  - \alpha(s)  T^{\pi_0} v^* (s) \nonumber\\
 &=\alpha(s)  \left( T v^*(s)  -  T^{\pi_0} v^*(s) \right)  \\
 &=\alpha(s)  \left( v^*(s)  -  T^{\pi_0} v^*(s) \right) = \alpha(s) L(s) .\nonumber
\end{align}\label{eq: supp bias less tight}
In the third relation we used the fact that $T v^* = v^*$ component-wise, since $v^*$ is the fixed-point of $T$. Thus, we see that,
\begin{align*}
\norm{Tv^* - T_\alpha v^*} = \max_s \alpha(s) L(s) = B(\alpha),
\end{align*}
and that $L(s)\geq 0$ since $v^*(s)  -  T^{\pi_0} v^*(s)\geq 0$. By taking the max-norm on \eqref{eq: biasProp LHS}, which is possible since it is positive, and simple algebraic manipulation we conclude the result.

We can continue and bound the above to get the bound in \eqref{eq: supp bias less tight}, which is less tight. We have that,
\begin{align}
|Tv^* - T^{\pi_0} v^*|(s) &=| T^{\pi^*}v^* - T_\alpha v^*|(s)  \label{eq: second bounds supp first}\\
&\leq\ \sum_{a} |\pi^*(a\mid s)-\pi_0(a\mid s)| \times \left| r(s,a) + \gamma \sum_{s'} P(s'\mid s,a)v^*(s') \right| \nonumber, 
\end{align} \label{eq: second bounds supp last}
where the first relation is by using the triangle inequality, and then use $|a\cdot b|\leq |a|\cdot|b|$. We further have that,
\begin{align*}
\left| r(s,a) + \gamma \sum_{s'} P(s'\mid s,a)v^*(s') \right| \leq \frac{R_{\mathrm{max}}}{1-\gamma}.
\end{align*}

Thus, continuing from \eqref{eq: second bounds supp last}, we can further bound \eqref{eq: second bounds supp first},
\begin{align*}
|Tv^* - T^{\pi_0} v^*|(s)&\leq \frac{R_{\mathrm{max}}}{1-\gamma}\sum_{a} |\pi^*(a\mid s)-\pi_0(a\mid s)|.
\end{align*}

Thus,
\begin{align*}
\alpha(s) (Tv^* - T^{\pi_0} v^*)(s) \leq \max\frac{\alpha(s) \norm{\pi^*-\pi_0}_{TV}(s) R_{\mathrm{max}}}{1-\gamma}
\end{align*}
where ${\norm{\pi^*-\pi_0}_{TV}(s) =  \sum_{a} |\pi^*(a\mid s)-\pi_0(a\mid s)|}$, is the total variation of $\pi^*$ and $\pi_0$ in state $s$.

Finally, the bound is proved tight by an example which attains it as described below:

For the MDP described in figure \ref{fig:bounds2MDP}, it is easy to see that for the uniform $\pi_0$:
\begin{align*}
v^*-v^{{\pi}^*_\alpha} = \frac{1}{1-\gamma} - \frac{1-\alpha/2}{1-\gamma}=\frac{\alpha/2}{1-\gamma}
\end{align*}
Next:
\begin{align*}
\frac{\alpha}{1-\gamma}\norm{v^*(s) - \sum_a \pi_0(a|s)q^*(s,a)}&=\frac{\alpha}{1-\gamma}\norm{\frac{1}{1-\gamma}-\frac{1/2}{1-\gamma}-\frac{\gamma/2}{1-\gamma}} =\frac{\alpha/2}{1-\gamma}
\end{align*}
\end{proof}
\end{proposition}
\begin{figure}
	\centering
	\resizebox{1.8in}{!}{
		\begin{tikzpicture}[->,>=stealth',shorten >=1pt,auto,node distance=2.8cm,
		semithick, state/.style={circle, draw, minimum size=1.1cm}]
		\tikzstyle{every state}=[thick]
		]
		\node[state] (S0) {$s_0$};
		\path
		(S0) edge  [loop left] node[pos=0.8,above]{\Large $a_0$} node {\Large $0$} (S0)
      		edge  [loop right]  node[pos=0.1,above]{\Large $a_1$}         node {\Large $1$} (S0);
		\end{tikzpicture}
	}
	\caption{One State MDP that attains the bound in Proposition \ref{prop:biasBound}}
	\label{fig:bounds2MDP}
\end{figure}
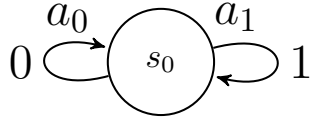

\begin{proposition}\label{prop:epsOptBound}
Let $\alpha\in [0,1]$. Assume $\hat{v}^*_\alpha$ is an approximate $\alpha$-optimal value s.t $\norm{v^*_\alpha-\hat{v}^*_\alpha}=\delta$ for some $\delta\geq 0$. Let $\pi_g$ be the greedy policy w.r.t. $v$,  $\hat{\pi}_\alpha^*\in \G(\hat{v}^*_\alpha
)$. Then
\begin{align*}
\norm{v^*_\alpha-v^{\pi^{\alpha}(\hat{\pi}_\alpha^*,\pi_0)}} \leq \frac{2(1-\alpha)\gamma\delta}{1-\gamma}
\end{align*}
Furthermore, there exists some $\delta_0>0$ such that if $\delta<\delta_0$, then $\hat{\pi}_\alpha^* = \pi_\alpha^*$, and this bound is tight.

\begin{proof}
First, notice that for any two $\alpha$-greedy policies, $\pi^\alpha(\pi_1,\pi_0),\pi^\alpha(\pi_2,\pi_0)$,
\begin{align*}
    \norm{\pi^\alpha(\pi_1,\pi_0)-\pi^\alpha(\pi_2,\pi_0)}_{TV}&=\norm{(1-\alpha)\pi_1 + \alpha\pi_0-(1-\alpha)\pi_2-\alpha \pi_0}_{TV} \\
    & = (1-\alpha) \norm{\pi_1 - \pi_2}_{TV} \\
    & \leq 2(1-\alpha)
\end{align*}
Where the last transition is due to the fact that for the total-
variation between distributions is always smaller than $2$, which is the case of two different deterministic policies.
Plugging in the result in Proposition~\ref{prop: generalized sensitivity bound}, we get the required bound.

Finally, we prove that this bound is tight (see that different MDP then in \cite{bertsekas1995neuro} is used). Observe at the MDP described in Figure \ref{fig:epsOptBound}. The policy $\pi^*_\alpha$ is to always choose action $a_1$.
Hence,
\begin{align*}
v^*_\alpha=\sum_{n=0}^{\infty}{\gamma^n \left[\gamma\delta(1-\frac{\alpha}{2}) - \gamma\delta\frac{\alpha}{2} \right]}=\frac{\gamma\delta(1-\alpha)}{1-\gamma}
\end{align*}
Now, given value estimation $\hat{v}^*_\alpha$, such that $\hat{v}^*_\alpha(s_0)=\delta,  \hat{v}^*_\alpha(s_1)=-\delta$, taking always $a_1$ is an $\alpha$-greedy policy with respect to $\hat{v}^*_\alpha$:
\begin{align*}
(1-\frac{\alpha}{2})(\gamma\delta+\gamma \hat{v}^*_\alpha(s_1)) + \frac{\alpha}{2}(-\gamma\delta + \gamma \hat{v}^*_\alpha(s_0)) = 0 = (1-\frac{\alpha}{2})(-\gamma\delta+\gamma \hat{v}^*_\alpha(s_0)) + \frac{\alpha}{2}(\gamma\delta + \gamma \hat{v}^*_\alpha(s_1)) 
\end{align*}
Hence, 
\begin{align*}
v^{\pi^\alpha(\hat{\pi}^*_\alpha,\pi_0)}=\sum_{n=0}^{\infty}{\gamma^n \left[-\gamma\delta(1-\frac{\alpha}{2}) + \gamma\delta\frac{\alpha}{2} \right]}=\frac{\gamma\delta(\alpha-1)}{1-\gamma} 
\end{align*}
Simple arithmetic show that this MDP attains the upper bound.
\end{proof}
\end{proposition}

\begin{figure}
	\centering
	\resizebox{2.2in}{!}{
		\begin{tikzpicture}[->,>=stealth',shorten >=1pt,auto,node distance=2.8cm,
		semithick, state/.style={circle, draw, minimum size=1.1cm}]
		\tikzstyle{every state}=[thick]
		]
		
		\node[state] (S0) {$s_0$};
		\node[state] (S1) [right of=S0] {\large $s_1$};

		\path
		(S0) edge  [loop above] node[pos=0.05,left]{\Large $a_0$} node {\Large $-\gamma\delta$} (S0)
      		edge  [bend left]   node[pos=0.15,above ]{\Large $a_1$}         node [above] {\Large $\gamma\delta$} (S1)
		(S1) edge  [loop above] node[pos=1,right]{\Large $a_1$} node {\Large $\gamma\delta$} (S1)
		      edge  [bend left]   node[pos=0.05,below]{\Large $a_0$}         node [below] {\Large $-\gamma\delta$} (S0);
		
		\end{tikzpicture}
	}
	\caption{Two State MDP that attains the bound in Proposition \ref{prop:epsOptBound} over a uniform $\pi_0$.}
	\label{fig:epsOptBound}
\end{figure}
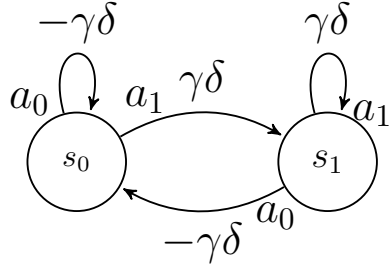

\section{Proof of Proposition \ref{prop:ContinuousContraction}}

In this section, we will prove Proposition~\ref{prop:ContinuousContraction}.
First, we define the sufficient conditions for an MDP on which Proposition~\ref{prop:ContinuousContraction} is true:

\begin{defn}\label{defn: bounded continuous MDP}
An MDP $\M=(\S,\A,P,R,\gamma)$ is a bounded continuous MDP if the following holds:
\begin{enumerate}
\item{$\A$ is a metric space, s.t. $\A=\R^{|\A|}$}
\item{ $\forall s\in\S$ and $a \in \A$, the state-wise reward function is positive, continuous, and bounded $r(s,a)$}
\item{$\forall s\in\S$, the state-wise reward function $r(s,a)$ is continuous in $a\in\A$}
\item{ $\forall s,s' \in \S$, the transition probability density function $p(s'|s,a)$ is continuous in $a\in\A$.}
\end{enumerate}
\end{defn}

Furthermore, we assume $\mathcal{S}$ is finite. Yet, we believe it is possible to extend our result to continuous space as well. This we leave for future work.
 
Next, we state again the definition the optimal policy with respect to the Gaussian noise:

\begin{align}\label{eq:continuous optimization}
&\mu_{\sigma}^* \in \arg\max_{\mu \in \tilde{\A}} {\E}^{\pi_{\mu,\sigma}} \left[\sum_{t=0}^{\infty} \gamma^t r(s_t,a_t) \right],
\end{align}
Where the optimization is restricted to $\tilde{\A}$, a compact subset of $\A$.

We are now state again our main theorem regarding the $\sigma$-optimal optimization criterion:

\begin{lemma}
Let $\M=(\S,\A,P,R,\gamma)$ be a bounded continuous MDP (see (\ref{defn: bounded continuous MDP})). Let $\N(\mu,\sigma)$ be the Gaussian measure with mean $\mu \in \mathbb{R}^n$ and $\sigma \geq 0$ and let $\tilde{\A}\subset\A$ be a compact metric space.
Then, the following claims hold:
\begin{enumerate}
\item $T_\sigma^{\mu} =  \E^{\pi\sim\pi_{\mu,\sigma}} T^\pi $, with fixed point $v^\mu_\sigma\!\!=\!\!v^{\pi_{\mu,\sigma}}$.
\item $T_\sigma \!=\! \max_{\mu\in\tilde{\A}} \! \E^{\pi\sim\pi_{\mu,\sigma}} T^\pi$, with fixed point $v^*_\sigma\!\!=\!\!v^{\pi_{\mu_\sigma^*,\sigma}}$.
\item A $\sigma$-optimal policy is an optimal policy of $\mathcal{M}_\sigma$ and is Gaussian w.r.t. $v^*_\alpha$, $\mu^*_\sigma \in \mathcal{N}_\sigma(v_\sigma^*)=\{\mu: T^{\pi_{\mu,\sigma}} v_\sigma^* = \max_{\mu} T^{\pi_{\mu,\sigma}} v_\sigma^* \}.$
\end{enumerate}
\begin{proof}

We define the surrogate MDP $\M_\sigma$ to have the following reward and dynamics,
\begin{align*}
&r_\sigma(s,a) = \int \mathcal{N}(a'\mid a,\sigma)r(s,a')da',\\
&p_\sigma(s'\mid s,a) = \int \mathcal{N}(a'\mid a,\sigma)p(s'\mid s,a')da'.
\end{align*}

Notice that 
\begin{align*}
&\sum_{s'} p_\sigma(s'\mid s,a) = \int \mathcal{N}(a'\mid a,\sigma)\sum_{s'} p(s'\mid s,a')da'\\
&=\int \mathcal{N}(a'\mid a,\sigma)da'=1.
\end{align*}
First, we show that the surrogate MDP $\M_\sigma$ is equivalent to a Gaussian policy on $\M$. More specifically, we show that the fixed policy bellman operator for a deterministic policy on $\M_\sigma$ is equivalent to the bellman operator of a Gaussian policy on $\M$. Then, we show similar relation for the bellman optimality operator.
\begin{lemma}\label{lem: equivalence of gaussian surrogate MDP}
The following claims hold:
\begin{enumerate}
\item The fixed-policy bellman operator on $\M_\sigma$, $T^\mu_\sigma$ and $T^{\pi_{\mu,\sigma}}$ are equivalent.
\item The bellman operator on $\M_\sigma$, $T_\sigma$ and $\max_\mu T^{\pi_{\mu,\sigma}}$ are equivalent.

\end{enumerate}
\begin{proof}
\begin{align*}
T^\mu_\sigma v & = r^\mu_\sigma + \gamma p^\mu_\sigma v \\
& = {\E}^{\pi\sim\pi_{\mu,\sigma}} {r^\pi} +   {\E}^{\pi\sim\pi_{\mu,\sigma}} {p^\pi} v  \\
& = {\E}^{\pi\sim\pi_{\mu,\sigma}} {r^\pi + \gamma p^\pi v} \\
& = {\E}^{\pi\sim\pi_{\mu,\sigma}} T^\pi
\end{align*}
The second relation holds directly from taking the maximum over both sides.
\end{proof}
\end{lemma}

By Lemma~\ref{lem: equivalence of gaussian surrogate MDP}, the connection between operators is stated for both (1) and (2). By the definition of $\M_\sigma$, for any Gaussian policy with mean $\mu$, $\pi_{\mu,\sigma}$, it holds that $v_\sigma^\mu=v_{\mu,\sigma}$.

Next, we prove the second relation. Again, we start by proving the following Lemma:

\begin{lemma}\label{lem: existence of sigma optimal Gaussian policy}
There exists a $\sigma$-optimal Gaussian policy
\begin{proof}
The functions $r_\sigma (s,a)$ and $p_\sigma (s'|s,a)$ are defined as the expectation of $r(s,\cdot)$ and $p(s'|s,\cdot)$ on the Gaussian measure with mean $a$ respectively. For every $s,s'\in\S$, define the integrand $g(a,\mu)=\phi(a|\mu,\sigma)f(s',s,a)$, where $f(s',s,a)$ represents $r(s,a)$ or $p(s'|s,a)$. The derivative of $\phi_\mu(a|\mu,\sigma)$ exists $\forall \mu \in \R^{|\A|}$. Thus, (a) $g_\mu(a,\mu)$ exists $\forall \mu \in \R^{|\A|}$. 
Next, For all $s,s' \in \S$, $r(s,a)$ and $p(s'|s,a)$ are continuous and bounded in $a$. $\forall \mu$, the Gaussian function is lebesgue-integrable function of $a$. Thus,  (b) $\forall \mu, g(a,\mu)$ is a Lebesgue-integrable function of $a$.
Now,  there exist $c>0$, such that, $| \phi_\mu | \leq c |a-\mu|\phi(a|\mu,\sigma)$. Furthermore, $f(s',s,a)$ is bounded. Hence, there exists $C>0$, such that, $|g_\mu(a,\mu)|\leq C |a-\mu| \phi(a|\mu,\sigma) \triangleq h(a,\mu)$. Then, $\forall \mu$, we can take an open ball of radius $r$, $B_r (\mu)$. Define, $ t(a)=\max_{x \in B_r (\mu)} {h(a,x)}$. $t$ is integrable for every $a \in \A$ by construction. In other words, (c) there is an integrable function $t:A \rightarrow \R$ such that $|g_\mu (a,\mu)| \leq t(a)$ for all $\mu \in B_r (\mu)$.

Finally, From (a),(b) and (c), by the Dominated convergence theorem, Leibniz integral rule applies, which means that $r_\sigma(s,a)$ and $p_\sigma(s'|s,a)$ are differentiable in $a\in \A$, and thus continuous in $a\in \A$, for every $s,s' \in \S$.


Now, (1) let $\M_{\sigma}$ be the surrogate MDP, and assume the state space is discrete. (2) 
For all $s,s' \in \S$, $r_\sigma (s,a)$ and $p_\sigma (s'|s,a)$ are continuous in $a$.
(3) By the definition of the optimality criterion, we consider only actions $a\in\A$. Hence, the action space of $\M_\sigma$ is compact.

Then, by theorem [6.2.10] in \cite{puterman1994markov}, there exist an optimal deterministic policy for the surrogate MDP, $\M_\sigma$.

By the definition of the $\M_\sigma$ and Lemma~\ref{lem: equivalence of gaussian surrogate MDP}, a deterministic policy $\mu$ in $\M_\sigma$ is equivalent to a Gaussian policy $\pi_{\mu,\sigma}$ on $\M$. Denote the optimal deterministic policy on the surrogate MDP as $\mu_\sigma^*$. Thus, the policy $\pi_{{\mu_\sigma^*},\sigma}$ is an $\sigma$-optimal Gaussian policy on $\M$.

\end{proof}
\end{lemma}

Finally, we show that solving the surrogate MDP is equivalent to solving \eqref{eq:continuous optimization} 
$T_\sigma$ is the greedy bellman operator on the surrogate MDP. Therefore, it is a $\gamma$-contraction. Thus, (a) by the Banach fixed point theorem and Theorem [6.2.2] in \cite{puterman1994markov}, $v_\sigma^*$ is the unique solution to the optimality equation, $T_\sigma v_\sigma^* = v_\sigma^*$. (b) By Lemma~\ref{lem: existence of sigma optimal Gaussian policy}, there exists a deterministic optimal policy. Combining (a) and (b), we get that the greedy policy w.r.t. $v_\sigma^*$, $\mu_\sigma^*$, is an optimal policy in the surrogate MDP. By transforming back to the original MDP we get that $\pi_\sigma^* = \pi_{\mu_\sigma^*,\sigma}$:
\begin{align*}
&\mu_\sigma^* \in \{\mu: T_\sigma^\mu v_\sigma^* = T_\sigma v_\sigma^*\}\\
&=\{\mu: {\E}^{\pi\sim\pi_{\mu,\sigma}} T^\pi v_\sigma^* = \max_{\mu} {\E}^{\pi\sim\pi_{\mu,\sigma}} T^\pi v_\sigma^* \}\\
&=\mathcal{N}_\sigma(v_\sigma^*).
\end{align*} 
%
\end{proof}
\end{lemma}

\subsection{MDP with bounded action space}

In this section we explain how to apply the $\sigma$-optimal criterion to an MDP with bounded action space. Let $\M$ be a bounded continuous MDP with a compact action-space $\A$. Proposition~\ref{prop:ContinuousContraction} demands the action space to be defined on the support of the Gaussian measure. Thus, we need to formalize how the Gaussian noise which is defined over $\R^{|\A|}$ operates on the bounded action set $\A$. Intuitively, we choose to project any action chosen outside the action set $a \notin \A$ onto the action set boundary. Formally, the noise operates on the extended MDP, $\M_{ext}$, as defined here.
\begin{defn}\label{defn: Gaussian Noise over bounded action space}
For a bounded continuous MDP $\M=(\S,\A,P,R,\gamma)$, we define the extended MDP, $\M_{ext}$, with action space $\A_{ext}=\R^{|\A|}$, such that:
\begin{enumerate}
\item{ $R_{ext}(s,a)= R(s,\mathcal{P}_{\A}(a))$, for all $s\in\S$.}
\item{$P_{ext}(s,a)= P(s,\mathcal{P}_{\A}(a))$, for all $s,s'\in \S$ }
\end{enumerate}
Where, $\mathcal{P}_{\A} (a)$ is the orthogonal projection of the action $a$ onto the set $\A$
\end{defn}

The MDP $\M_{ext}$ is a bounded continuous MDP, with action space $\R^{|\A|}$. Therefore, by \ref{prop:ContinuousContraction}, it is possible to find the optimal policy w.r.t. the $\sigma$-optimal criterion, over any bounded action space. Finally, most naturally, one can apply the criterion to the original action space $\A$.

\section{No Improvement in Continuous Control}\label{supp: prop NoImprovementContinuous}

We give here the proof, the improvement is not always guaranteed in the continuous case.

\begin{proposition}\label{Supp: prop NoImprovementContinuous}
Let $0\leq\sigma'<\sigma$ and let $\mu^*$ be the $\sigma$-optimal policy. There exists an MDP such that ${v^{{\pi}_{\mu^*,\sigma}}>v^{{\pi}_{\mu^*,\sigma'}}}$. Decreasing the stochasticity can hurt the performance of the agent, and improvement is not guaranteed.
\end{proposition}

\begin{proof}
Let $\mathcal{M}$ be a one-state MDP, with the following reward: 
$r(u)=\frac{1}{2} \frac{1}{\sqrt{\pi}} e^{-(u-1)^2}+\frac{1}{2} \frac{1}{\sqrt{\pi}} e^{-(u+1)^2}$.
The expected reward under a Gaussian policy with $\mu$ and $\sigma=1$ is:
$r^\pi = \frac{1}{2} \frac{1}{\sqrt{3\pi}} e^{-(\mu-1)^2/3}+\frac{1}{2} \frac{1}{\sqrt{3\pi}} e^{-(\mu+1)^2/3}$. It is easy to calculate that the maximum of $r^\pi$ is attained when $\mu=0$ and its value lower bounded by $0.23$. Hence, the $\sigma$-optimal policy with $\sigma=1$ is $\pi(u|s)=\mathcal{N}(0,1)$. However, acting greedily w.r.t the mean of the $\sigma$-optimal, i.e., acting always with $u=0$ can be upper bounded by $0.21$. Thus, $r^{\pi_{\sigma}^*}>r^{\pi^{\mu_\sigma},0}$
\end{proof}

An illustration of such a case is given in figure~\ref{fig: NoImprovementIllustration}.

\begin{figure}[t]
\centering
	\includegraphics[width=0.4\textwidth]{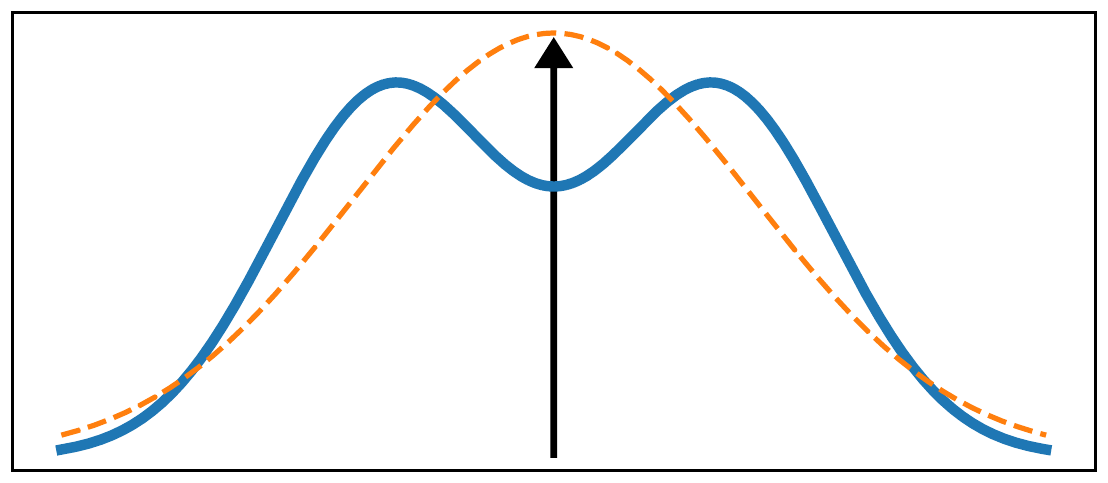} 
\caption{Illustration of a typical case where there is no improvement: (Blue) The state-action value function as a function of the action taken. (Orange) The $\sigma$-optimal policy is with $\mu_\sigma^*=0$ due to the smoothing effect of the Gaussian policy. (Black) A deterministic policy around the $\mu_\sigma^*$. It can be easily seen that decreasing the noise degrades the performance of the agent.}
\end{figure}\label{fig: NoImprovementIllustration}%

While in the general case there is no improvement, it is easy to verify that a sufficient condition for improvement is that the state-wise variance of the ${q}^{\pi _{\sigma }^{*}}$  w.r.t. every smaller noise level, $\tilde{\sigma} < \sigma$, is less than the noise level itself:
\begin{align*}
\frac{{\E}_{a\sim\pi_{{{\mu }_{\sigma }^*},\tilde{\sigma }}(\cdot\mid s)}\left[ {{\left( a-{{\mu }^*_{\sigma }}\left( s \right) \right)}^{2}}{{q}^{\pi _{\sigma }^{*}}}\left( s,a \right) \right]}{{\E}_{a\sim\pi_{{{\mu }_{\sigma }^*},\tilde{\sigma }}(\cdot \mid s)} {{q}^{\pi _{\sigma }^{*}}}\left( s,a \right)} \le {{\tilde{\sigma }}^{2}}.
\end{align*}

\section{Proof of Theorem~\ref{theorem: gaussian tradeoff}: Bias-Error Sensitivity in the Gaussian case}\label{supp: theorem gaussian tradeoff}

In this section we prove a bias-error sensitivity result for the Gaussian noise case, similarly to \ref{theorem:performance model free}. Theorem~\ref{theorem: gaussian tradeoff} exhibits a Bias-Sensitivity trade-off w.r.t. the noise parameter $\sigma$. When $\sigma$ grows, the bias increases in $\norm{\sigma}_1$, but the sensitivity term decreases. In the limit where $\sigma$ goes to infinity, the approximation error tend to zero. In the other limit, where the noise reduces to zero, we return to the case of a greedy optimal policy. Indeed, as the bound shows, we get an unbiased solution, and the sensitivity term reduces to the classical bound of \citet{bertsekas1995neuro}. Unsurprisingly, we get a better sensitivity bound only when there is a sufficient overlap between the two policies.




In order to prove the theorem, we will first prove two propositions: A bias proposition \ref{prop: gaussian bias} and a sensitivity proposition \ref{prop: gaussian sensitivity}. Then, we plug the results in the following triangle inequality:
\begin{align*}
\norm{v^* - v^{\hat{\mu},\sigma}} \leq \norm{v^* - v^*_\sigma} + \norm{v^*_\sigma - v^{\hat{\mu},\sigma}}
\end{align*}

First, we derive the bias proposition,
\begin{proposition}\label{prop: gaussian bias}
Let $\sigma\geq 0$ and let $\pi_\sigma^*$ be the $\sigma$-optimal policy. Assume an MDP $\M$ is Lipschitz, i.e., there exists $L_r\geq0$ and $L_p\geq 0$, such that, $\forall s,s'\in \mathcal{S}$ and $\forall a_1,a_2 \in \mathcal{A}$, $\left|r(s,a_1)-r(s,a_1)\right|<L_r\norm{a_1-a_2}_1$ and $\left|p(s'|s,a_1)-p(s'|s,a_1)\right|<L_p\norm{a_1-a_2}_1$. Then, the following holds,
\begin{align*}
\norm{v^*-v_\sigma^*} \leq \sqrt{\frac{2}{\pi}}\frac{ \left(1-\gamma\right)L_r + \gamma L_p R_{max}}{(1-\gamma)^2} \sigma
\end{align*}
\begin{proof}

\begin{align*}
\left\| {{v}^{*}}-v_{\sigma }^{*} \right\|&=\left\| {{v}^{*}}-{T_\sigma}v_{\sigma }^{*} \right\| \\
&\le \left\| {{v}^{*}}-{T_\sigma}{{v}^{*}} \right\|+\left\| {T_\sigma}{{v}^{*}}-{T_\sigma}v_{\sigma }^{*} \right\| \\
& \le \left\| {{v}^{*}}-{T_\sigma}{{v}^{*}} \right\|+\gamma \left\| {{v}^{*}}-v_{\sigma }^{*} \right\|
\end{align*}
Where the inequality is due to the fact that $T_\sigma$ is a $\gamma$-contraction.
Simple algebra gives $\norm{{{v}^{*}}-v_{\sigma }^{*}} \leq \frac{\norm{v^*-T_\sigma v^*}}{1-\gamma}$

Next, we bound the nominator:
\begin{align*}
v^*(s)-(T_\sigma v^*)(s) & = (T^*v^*)(s)-(T_\sigma v^*)(s) \\
& =  \max_a r(s,a) + \gamma \sum_{s' \in \S} p(s'|s,a) v^*(s') - \max_\mu \int {\N(a|\mu,\sigma)\left[ r(s,a) + \gamma \sum_{s' \in \S} {p(s'|s,a) v^*(s')}\right]da} \\
& \leq r(s,a^*) + \gamma \sum_{s' \in \S} p(s'|s,a^*) v^*(s') - \int {\N(a|a^*,\sigma)\left[ r(s,a) + \gamma \sum_{s' \in \S} {p(s'|s,a)v^*(s')}\right]da} \\
& = \int {\N(a|a^*,\sigma)\left[ \left(r(s,a^*) - r(s,a)\right) + \gamma \sum_{s' \in \S} {\left(p(s'|s,a^*) - p(s'|s,a)\right) v^*(s')}\right]da} \\
& \leq \int {\N(a|a^*,\sigma)\left[ \left(r(s,a^*) - r(s,a)\right) + \gamma \sum_{s' \in \S} {\left|p(s'|s,a^*) - p(s'|s,a)\right| v^*(s')}\right]da} \\
& \leq \int {\N(a|a^*,\sigma)\left[ \left(r(s,a^*) - r(s,a)\right) +\frac{\gamma R_{max}}{1-\gamma} \sum_{s' \in \S} {\left|p(s'|s,a^*) - p(s'|s,a)\right|} \right]da} \\
& \leq \int {\N(a|a^*,\sigma)\left[ L_r \norm{a^*-a}_1 + \gamma \norm{p(\cdot \mid s,a^*)-p(\cdot \mid s,a)}_{TV} \frac{R_{max}}{1-\gamma} \right]da} \\
& \leq \int {\N(a|a^*,\sigma)\left[ L_r \norm{a^*-a}_1 + \gamma L_p \norm{a^*-a}_1 \frac{R_{max}}{1-\gamma} \right]da} \\
& = \left( L_r  + \gamma L_p \frac{R_{max}}{1-\gamma} \right) \int {\N(a|a^*,\sigma)\norm{a^*-a}_1da} \\ 
& = \left( L_r  + \gamma L_p \frac{R_{max}}{1-\gamma} \right) \sqrt{\frac{2}{\pi}} \norm{\sigma}_1
 \end{align*}

Where the first transition is due to $a^*\in \arg\max r(s,a) + \gamma \sum_{s' \in \S} p(s'|s,a) v^*(s')$, and the last is due to the absolute first moment of the Gaussian distribution.

We get,
\begin{align*}
\norm{v^*-T_\sigma v^*} \leq  \sqrt{\frac{2}{\pi}}\left( L_r  + \gamma L_p \frac{R_{max}}{1-\gamma} \right) \norm{\sigma}_1
\end{align*}
Finally, combining the two results gives:
\begin{align*}
\norm{{{v}^{*}}-v_{\sigma }^{*}} \leq \sqrt{\frac{2}{\pi}}\frac{(1-\gamma)L_r  + \gamma L_p R_{max}}{\left(1-\gamma\right)^2} \norm{\sigma}_1
\end{align*}
\end{proof}
\end{proposition}

Finally, we prove the following sensitivity proposition using:
\begin{proposition}\label{prop: gaussian sensitivity}
Let $\sigma \in {\R}^{|\A|}_{+}$. Assume $\hat{v}_\sigma^*$ is an approximate $\sigma$-optimal value s.t. $\norm{v_\sigma^*-\hat{v}_\sigma^*}=\delta$ for some $\delta>0$. Let ${\mu}_\sigma^*,\hat{\mu}_\sigma^* \in {\R}^{\A}$ be the greedy mean policy w.r.t. ${v}_\sigma^*$ and $\hat{v}_\sigma^*$ respectively. Then,
\begin{align*}
\norm{v^*_\sigma-v^{\pi_{\hat{\mu}_\sigma^*,\sigma}}} \leq  \frac{1}{2} \frac{\gamma\delta 
\min \{\norm{{\mu}_\sigma^*-\hat{\mu}_\sigma^*}_{\sigma^{-2}} ,4\}}{1-\gamma},
\end{align*}
where $\norm{\cdot}_{\sigma^{-2}}$ is the $\sigma^{-2}$-weighted euclidean norm.
\begin{proof}
First, notice that the total variation distance is not bigger than $2$, which is the case of two different deterministic policies, as seen in \cite{bertsekas1995neuro}[Proposition 6.1].
Next, the Kullback-Leibler divergence between two Gaussian distributions with the same variance $\sigma$ is $\frac{1}{2}\norm{\mu_\sigma^*-\hat{\mu}_\sigma^*}^2_{\sigma^{-2}}$, where $\norm{\cdot}_{\sigma^{-2}}$ is the $\sigma^{-2}$-weighted euclidean norm.
Finally, by using Pinsker's inequality to bound the total variation distance, and plugging in the closed form of the Kullback-Leibler divergence, one gets the required result.
\end{proof}
\end{proposition}

\section{Supplementary material for Section~\ref{sec: algorithms}}\label{supp: algorithms}


In this section we give the proofs for the algorithms proposed in Section~\ref{sec: alg alpha optimal}.

The proof of Lemma \ref{lemma: q alpha q pi mixture} is given as follows:
\begin{proof}
By using the definition of $T^{Eq}_\alpha$, and due to $v^*_\alpha = \max_a q^*_\alpha(\cdot,a)$, we have that,
\begin{align*}
q^*_\alpha(s,a) &= T^{Eq}_\alpha q^*_\alpha(s,a) \\
&= r_\alpha(s,a)+ \gamma \sum_{s'}P_\alpha (s'\mid s,a)\max_{a'}q^*_\alpha(s',a')\\
&= (1-\alpha)\left(r(s,a)+ \gamma \sum_{s'}P (s'\mid s,a)v^*_\alpha(s')\right)
+\alpha\sum_a \pi(a'\mid s)\left(r(s,a')+ \gamma \sum_{s'}P (s'\mid s,a')v^*_\alpha(s')\right)\\
&= (1-\alpha)q^{\pi^\alpha(\pi^*_\alpha,\pi_0)}(s,a) + \alpha\sum_a \pi_0(a'\mid s)q^{\pi^\alpha(\pi^*_\alpha,\pi_0)}(s,a'),
\end{align*}
where in the last relation we used \eqref{eq: def q pi mixture}.
\end{proof}

We now prove the following lemma:

\begin{lemma}\label{lemma: T EQ fixed point}
The operator $T^{Eq}_\alpha$ is a $\gamma$-contraction, and its fixed point is $q^{\pi^\alpha(\pi^*_\alpha,\pi_0)}$

\begin{proof}
It is easy to verify this operator is a $\gamma$-contraction using standard arguments \cite{bertsekas1995neuro}. We prove that the fixed point of $T^{Eq}_\alpha $ is $q^{\pi^\alpha(\pi_\alpha^*,\pi_0)}$. First, by using the max operator w.r.t. the action on the result in Lemma \ref{lemma: q alpha q pi mixture}, we get
\begin{align}
v^*_\alpha = (1-\alpha)\max_a q^{\pi^\alpha(\pi^*_\alpha,\pi_0)}(\cdot,a) + \alpha \Pi_0 q^{\pi^\alpha(\pi^*_\alpha,\pi_0)}. \label{eq: v alpha star and q mixture optimal}
\end{align}

Consider the definition of $q^{\pi^\alpha(\pi_\alpha^*,\pi_0)}$ \eqref{eq: def q pi mixture}. We have that,
\begin{align*}
q^{\pi^\alpha(\pi^*_\alpha,\pi_0)}(s,a) &= r(s,a) + \gamma \sum_{s'} P(s'\mid s,a)v^*_\alpha(s') \\
&= r(s,a)
+ \gamma(1-\alpha) \sum_{s'} P(s'\mid s,a) \max_{a'} q^{\pi^\alpha(\pi^*_\alpha,\pi_0)}(s',a') \\ 
& \quad+ \gamma\alpha\sum_{s',a'} P(s'\mid s,a)\pi_0(a'\mid s')q^{\pi^\alpha(\pi^*_\alpha,\pi_0)}(s',a')\\
&= T^{Eq}_\alpha q^{\pi^\alpha(\pi^*_\alpha,\pi_0)}(s,a),
\end{align*}
where the first relation holds by plugging \eqref{eq: v alpha star and q mixture optimal} and the third relation holds by identifying the operator $T^{Eq}_\alpha $.
\end{proof}
\end{lemma}

\subsection{Convergence of Expected $\alpha$-Q-Learning}\label{supp: expected q alpha learning}

Now, we move on to prove the convergence of Expected $\alpha$-Q-Learning:
\begin{theorem}\label{supp: theorem expected q alpha}
Consider the process described in Algorithm~\ref{alg:expected alpha}. Assume the sequence $\{\eta_t \}_{t=0}^\infty$ satisfies  ${\forall s\in \S}$, ${\forall a\in \A}$, ${\sum_{t=0}^\infty \eta_t(s_t=s,\aEnvt=a) =\infty}$, and ${\sum_{t=0}^\infty \eta_t^2(s_t=s,\aEnvt=a) <\infty}$. Then, the sequence $\{q_n\}_{n=0}^\infty$ converges w.p 1 to $q^{\pi^\alpha(\pi^*_\alpha,\pi_0)}$.
\begin{proof}
The updating equations of Algorithm~\ref{alg:expected alpha} can be written as
\begin{align*}
q_{n+1}(s,a^{env}) =& (1-\eta_t)q_{n}(s,\aEnv)
+ \eta_t (T^{Eq}_\alpha q_n (s,\aEnv) -w_t),
\end{align*}
where
\begin{align*}
w_t = r_t &+ \gamma (1-\alpha)v(s_{t+1})
+\gamma \alpha v^{\pi_0}(s_{t+1})- T^{Eq}_\alpha q_t (s,\aEnv),
\end{align*}
and 
\begin{align*}
&v(s_{t+1}) = \max_{a'} q(s_{t+1},a'), 
&v^{\pi_0}(s_{t+1}) = \sum_{a'} \pi_0(a'\mid s_{t+1})q(s_{t+1},a').
\end{align*}

We let $\mathcal{F}_t = \{\mathcal{H}_{t-1},s_t,\aEnv,X_t,\aChosen,r_t \}$, where $\mathcal{H}_{t-1}$ is the entire history until and including time $t-1$.  i.e, the filtration includes both the chosen action, before deciding whether to act with it or according to $\pi_0$, and the acted action.

We have that,
\begin{align*}
&\mathbb{E}\left[r_t + \gamma (1-\alpha) \max_a q(s_{t+1},\aEnv)(s_{t+1}) \mid \mathcal{F}_t \right]=\\
&\quad=r(s_t,\aEnv)+ \gamma(1-\alpha)\sum_{s'}P(s'\mid s,\aEnv) \max_{a'} q(s',a')
+\gamma \alpha \sum_{s',a'}P(s'\mid s_t,\aEnv)\pi_0(a'\mid s')q(s',a'),
\end{align*}
and $\mathbb{E}\left[ w_t \mid \mathcal{F}_t \right]=0 $. It is also easy to see that ${\mathbb{E}\left[ w_t^2 \mid \mathcal{F}_t  \right]\leq A+B ||Q ||^2_\infty}$. 

Thus, according to \cite{bertsekas1995neuro}[Proposition 4.4] the process converges to the fixed point contraction operator $T^{Eq}_\alpha$, $q^{\pi^\alpha(\pi_\alpha^*,\pi_0)}$ (see Lemma \ref{lemma: T EQ fixed point}).
\end{proof}

\end{theorem}

\subsection{Convergence of Surrogate $\alpha$-Q-Learning}\label{supp: surrogate q alpha learning}

In this section, we prove the convergence of Surrogate $\alpha$-Q-Learning:
\begin{theorem}\label{theorem: surrogate q alpha learning}
Consider the process described in Algorithm~\ref{alg: surrogate Q}. Assume the sequence $\{\eta_t \}_{t=0}^\infty$ satisfies  ${\forall s\in \mathcal{S}}$, ${\forall a\in \mathcal{A}}$,  ${\sum_{t=0}^\infty \eta_t(s_t=s,\aEnvt=a) =\infty}$, and ${\sum_{t=0}^\infty \eta_t^2(s_t=s,\aEnvt=a) <\infty}$. Then, the sequences $\{q_n\}_{n=0}^\infty$ and $\{q_{\alpha,n}\}_{n=0}^\infty$ converges w.p 1 to $q^{\pi^\alpha(\pi^*_\alpha,\pi_0)}$ and $q^*_\alpha$, respectively.
\end{theorem}

We will use the following result \cite{singh2000convergence}[Lemma 1].
\begin{lemma}\label{lemma: singh}
Consider a stochastic process $(\alpha_t,\Delta_t, \Delta_t,f_t)$, $t\geq 0$, where $\alpha_t$, $\Delta_t$, $f_t:X\rightarrow \mathbb{R}$ satisfy the equations
\begin{align}
&\Delta_{t+1}(x) = (1-\alpha_t(x))\Delta_t(x)+\alpha_t(x) f_t(x), \nonumber \\
& x\in X, \ \ t=0,1,2,..\label{eq: lemma by singh}
\end{align}
Let $\mathcal{F}_t$ be a sequence of increasing $\sigma$-fields such that $\alpha_0$ and $\Delta_0$ are $\mathcal{F}_0$-measurable, $t=1,2,...$. Assume that the following hold:
\begin{enumerate}
\item The set $X$ is finite.
\item $0\leq \alpha_t(x)\leq 1$, $\sum_t \alpha_t(x)=\infty$, $\sum_t \alpha^2_t(x)< \infty$ w.p 1.
\item $|| \mathbb{E}\left[  f_t(\cdot)  \mid \mathcal{F}_t \right] || \leq \kappa ||\Delta_t || + c_t$, where $\kappa\in [0,1)$ and $c_t$ converges to zero w.p 1.
\item $Var \left[  F_t(\cdot)  \mid \mathcal{F}_t \right] \leq K(1+|| \Delta_t ||)^2,$ where $K$ is some constant.
\end{enumerate}

Then, $\Delta_t$ converges to zero with probability 1.

\end{lemma}

Observe that $q_{t}$ has updating rule as in Expected $\alpha$-Q-Learning (see Algorithm \ref{alg:expected alpha}), and is independent of $q_\alpha$. Due to the assumptions that $\forall s\in \mathcal{S}, \forall a\in \mathcal{A}$
\begin{align*}
&\sum_{t=0}^\infty \eta_t(s_t=s,\aEnvt=a) =\infty,\\
&\sum_{t=0}^\infty \eta_t(s_t=s,\aEnvt=a) \leq \infty,
\end{align*}
we get that the sequence $\{q_{t} \}_{t=0}^\infty$ converges to $q^{\pi^\alpha(\pi^*_\alpha,\pi_0)}$ w.p 1.

We now manipulate the updating of $q$ in Algorithm \ref{alg: surrogate Q} to have the form of \eqref{eq: lemma by singh}. Define the following difference
\begin{align*}
\Delta_t(s,a) = q_t(s,a) - q^*_\alpha(s,a),
\end{align*}
and consider the filtration $\mathcal{F}_t = \{\mathcal{H}_{t-1},s_t,\aChosen\}$.

By decreasing $q^*_\alpha(s,a)$ from both sides of the updating equations of $q$ in Algorithm \ref{alg: surrogate Q}, we obtain for any $a\in \mathcal{A}$,
\begin{align*}
\Delta_{t+1}(s_t,a) = (1-\eta_t)\Delta_{t}(s_t,a)f_t(s_t,a).
\end{align*}
If $\bar{a}=\aChosen$ then,
\begin{align*}
f_t(s_t,\bar{a}) = r_t + \gamma  v_{\alpha,t}(s_{t+1}) - q^*_\alpha(s,a),
\end{align*}
whereas for $\bar{a}\neq \aChosen$,
\begin{align*}
f_t(s_t,\bar{a}) = &X_t q^{\pi^\alpha(\pi_\alpha^*,\pi_0)}(s_t,\bar{a})
   					+(1-X_t)( r_t + \gamma  v_{\alpha,t}(s_{t+1}))\\
				   &+X_t (q_t(s_t,\bar{a})-q^{\pi^\alpha(\pi_\alpha^*,\pi_0)}(s_t,\bar{a}))- q^*_\alpha(s_t,\bar{a}).
\end{align*}

We now show that for all action entries $\bar{a}\in \mathcal{A}$,  ${\mathbb{E}\left[  f_t(s_t,\bar{a})  \mid \mathcal{F}_t \right] || \leq \kappa ||\Delta_t(s_t,\bar{a}) || + c_t}$, and $c_t$ converges to zero w.p. 1. 

If $\bar{a}=\aChosen$ then,
\begin{align*}
\mathbb{E}\left[  f_t(s_t,\bar{a})  \mid \mathcal{F}_t \right] &=
(1-\alpha)(r(s_t,\bar{a}) + \gamma\sum_{s'}P(s'\mid s_t,\bar{a})\max_{a'}q_{\alpha,t}(s',a'))\\
& \quad+\!\alpha(r^{\pi_0}(s_t)\!\! + \!\!\gamma\sum_{s'}P^{\pi_0}(s'\mid s_t)\max_{a'}q_{\alpha,t}(s',a'))\!\!- \!\!q^*_\alpha(s,a)\\
&= T_\alpha q_{\alpha,t}(s_{t+1},a'))- q^*_\alpha(s,a).
\end{align*}
Thus, for this case,
\begin{align*}
||\mathbb{E}\left[  f_t(s_t,\bar{a})  \mid \mathcal{F}_t \right] || &= || T_\alpha q_{\alpha,t}(s_{t+1},a'))- q^*_\alpha(s,a) || \\
& = || T_\alpha q_{\alpha,t}(s_{t+1},a'))-  q^*_\alpha(s,a) || \\
&\leq \gamma  ||q_{\alpha,t}(s_{t+1},a'))-   q^*_\alpha(s,a)  ||,
\end{align*}
meaning, $c_t=0$ for this entry. We now turn to the case $\bar{a}\neq a^{chosen}$.
\begin{align*}
\mathbb{E}\left[  f_t(s_t,\bar{a})  \mid \mathcal{F}_t \right] &= (1-\alpha) q^{\pi^\alpha(\pi_\alpha^*,\pi_0)}(s_t,\bar{a}) - q^*_\alpha(s,\bar{a})\nonumber \\
&\quad+ \alpha (r^{\pi_0}+\gamma \sum_{s'}P^{\pi_0}(s'\mid s)\max_{a'}q_{\alpha,t}(s',a'))\nonumber\\
&\quad +(1-\alpha)(q_{t}(s_t,\bar{a})-q^{\pi^\alpha(\pi_\alpha^*,\pi_0)}(s_t,\bar{a})). \nonumber
\end{align*}
Define 
\begin{align*}
c_t \triangleq  (1-\alpha)(q_{t}(s_t,\bar{a})-q^{\pi^\alpha(\pi_\alpha^*,\pi_0)}(s_t,\bar{a})).
\end{align*}
See that $c_t$ converges to zero w.p. 1, since $q_t$ converges to $q^{\pi^\alpha(\pi_\alpha^*,\pi_0)}$. Furthermore, using Lemma \ref{lemma: q alpha q pi mixture}, we have that

\begin{align*}
&(1-\alpha) q^{\pi^\alpha(\pi_\alpha^*,\pi_0)}(s_t,\bar{a}) - q^*_\alpha(s,\bar{a})= - \alpha (r^{\pi_0}+\gamma \sum_{s'}P^{\pi_0}(s'\mid s)\max_{a'}q^*_{\alpha}(s',a')).
\end{align*}

Thus,
\begin{align*}
\mathbb{E}\left[  f_t(s_t,\bar{a})  \mid \mathcal{F}_t \right] &=
 -\alpha (r^{\pi_0}+\gamma \sum_{s'}P^{\pi_0}(s'\mid s)\max_{a'}q^*_{\alpha}(s',a'))
 + \alpha (r^{\pi_0}+\gamma \sum_{s'}P^{\pi_0}(s'\mid s)\max_{a'}q_{\alpha,t}(s',a')) +c_t\nonumber\\
&= \alpha \gamma \sum_{s'}P^{\pi_0}(s'\mid s)(\max_{a'}q_{\alpha,t}(s',a')-\max_{a'}q^*_{\alpha}(s',a'))+c_t\nonumber\\
&= \alpha \gamma \sum_{s'}P^{\pi_0}(s'\mid s)|(\max_{a'}q_{\alpha,t}(s',a')-\max_{a'}q^*_{\alpha}(s',a'))|+c_t\nonumber\\
&= \alpha \gamma \sum_{s'}P^{\pi_0}(s'\mid s)\max_{a'}|(q_{\alpha,t}(s',a')-q^*_{\alpha}(s',a'))|+ c_t \nonumber\\
&= \alpha \gamma \max_{s',a'}||q_{\alpha,t}-q^*_{\alpha}||+ c_t \nonumber
\end{align*}
Where in the first relation we applied Lemma \ref{lemma: q alpha q pi mixture}. By showing similar result for $-\mathbb{E}\left[  f_t(s_t,\bar{a})  \mid \mathcal{F}_t \right]$, we conclude that,
\begin{align*}
&\mathbb{E}\left[  f_t(s_t,\bar{a})  \mid \mathcal{F}_t \right] \leq \alpha \gamma \max_{s',a'}||q_{\alpha,t}-q^*_{\alpha}||+ c_t, 
\end{align*}
where $c_t$ converges to zero w.p.1. The $\mathrm{Var}(f_t(\cdot,\cdot))$ can be bounded by $K(1+|| \Delta_t ||)^2$, since the reward is bounded and ${\sum_{t=0}^\infty \eta_t^2(s_t=s,\aEnvt=a) <\infty}$.

We conclude that all conditions of Lemma \ref{lemma: singh} are satisfied for each $\bar{a}\in \mathcal{A}$ and, thus, Lemma \ref{lemma: singh} establishes the convergence of the procedure.

%

\subsection{Proof of the gradients' equivalence in section \ref{sec: alg continuous control}}\label{supp: eq gradient equivalence}
\begin{proof}
\begin{align*}
  {{\nabla }_{u}}{{{q}}_\sigma^{\pi }}\left( s,u \right)&=
  {{\nabla }_{u}}\int\limits_{A}{\mathcal{N}\left( u'|u,\sigma  \right){{q}^{\pi }}\left( s,u' \right)du'} \\ 
 & =\int\limits_{A}{{{q}^{\pi }}\left( s,u' \right){{\nabla }_{u}}\mathcal{N}\left( u'|u,\sigma  \right)du'} \\ 
& =-\int\limits_{A}{{{q}^{\pi }}\left( s,u' \right){{\nabla }_{u'}}\mathcal{N}\left( u'|u,\sigma  \right)du'} \\ 
 & =-\left. {{q}^{\pi }}\left( s,u' \right)\mathcal{N}\left( u'|u,\sigma  \right) \right|_{-\infty }^{\infty }
 +\int\limits_{A}{\mathcal{N}\left( u'|u,\sigma  \right){{\nabla }_{u'}}{{q}^{\pi }}\left( s,u' \right)du'} \\ 
 & =\int\limits_{A}{\mathcal{N}\left( u'|u,\sigma  \right){{\nabla }_{u'}}{{q}^{\pi }}\left( s,u' \right)du'} 
\end{align*}
Where we used integration by parts.

\end{proof}

\end{appendices}

\end{document}